\PassOptionsToPackage{table}{xcolor}
\documentclass[10pt, a4paper, copyright]{krafton-ai}

\usepackage[authoryear, sort&compress, round]{natbib}
\usepackage{nicefrac}
\usepackage{algorithm}
\usepackage{algorithmic}
\usepackage{multirow}
\usepackage{makecell}
\usepackage{wrapfig}
\usepackage{colortbl}
\usepackage{subcaption}
\usepackage{fontawesome}
\usepackage{float}
\usepackage[bottom]{footmisc}

\makeatletter
\@ifundefined{c@theorem}{%
  \newtheorem{theorem}{Theorem}[section]
  \newtheorem{lemma}[theorem]{Lemma}
  \newtheorem{proposition}[theorem]{Proposition}
  \newtheorem{corollary}[theorem]{Corollary}
  \newtheorem{definition}[theorem]{Definition}
  \newtheorem{remark}[theorem]{Remark}
  \newtheorem{assumption}[theorem]{Assumption}
}{}
\makeatother

\DeclareMathOperator*{\argmax}{arg\,max}

\NewTColorBox[auto counter]{example}{ O{!htbp} m O{} }{
  enhanced,
  breakable,
  float*,
  floatplacement={#1},
  width=\textwidth,
  colback=gray!5!white,
  colframe=gray!75!black,
  title={Example~\thetcbcounter: #2},
  enlarge top by=5mm,
  enlarge bottom by=5mm,
  label={#3}
}

\providecolor{palGood}{RGB}{46,125,50}

\uselogo{}

\title{\centering Pruning and Distilling Mixture-of-Experts \\ into Dense Language Models}
\usepackage[useregional=false]{datetime2}
\DTMsetstyle{iso}
\paperdate{\DTMtoday}

\author[1]{Junhyuck Kim}
\author[1]{Jihun Yun}
\author[2]{Haechan Kim\textsuperscript{*}}
\author[1]{Gyeongman Kim}
\author[1]{Joonghyun Bae}
\author[1]{Jaewoong Cho}

\affil[1]{KRAFTON}
\affil[2]{KAIST}

\makeatletter
\renewcommand{\maketitle}{\bgroup\setlength{\parindent}{0pt}
  \begin{adjustwidth}{0pt}{24pt}
    \begin{flushleft}
      {
        {\raggedright \titlefont \@title\par}%
        \vskip11pt
        {\centering\large\@author\par}
        \vskip20pt%
      }%
    \end{flushleft}
  \end{adjustwidth}
  \egroup
  {%
    {\abscontent}
  }%
  \thispagestyle{firststyle}
}
\renewcommand{\abscontent}{%
  \begin{tcolorbox}[
    enhanced, frame hidden, colback=KraftonLightGray, arc=4pt,
    left=12pt, right=12pt, top=12pt, bottom=12pt,
    before skip=0pt, after skip=0pt
  ]
  {\absfont \theabstract}
  \vskip0.8em
  \noindent{\textcolor{black}{\faGithub~GitHub:}\enspace%
    \textcolor{cyan!60!black}{\href{https://github.com/krafton-ai/moe-to-dense}{https://github.com/krafton-ai/moe-to-dense}}}
  \end{tcolorbox}
}
\makeatother

\begin{abstract}
    Mixture-of-Experts (MoE) is now the dominant architecture for frontier language models, yet it requires all expert parameters to be loaded in memory, making it less preferable for memory-constrained deployment.
    Existing compression methods reduce the number of experts but the output remains an MoE model with the same fundamental limitation.
    We present the first systematic framework for converting a trained MoE into a standard \emph{fully dense} architecture: experts are scored, selected, and grouped, then concatenated into a dense FFN and refined by knowledge distillation from the MoE teacher.
    We evaluate 7 scoring, 5 grouping, and 2 magnitude scaling methods across a range of selected expert counts on Qwen3-30B-A3B, yielding 350 configurations.
    We find that the choice of scoring method is the most impactful, with our novel diversity-aware scoring consistently outperforming prior methods on Qwen3-30B-A3B, DeepSeek-V2-Lite, and GPT-OSS-20B.
    Under a controlled comparison at matched parameter count, MoE-to-dense outperforms dense-to-dense pruning by $+$6.3~pp in average downstream accuracy after ${\sim}$4B-token distillation at 1.6$\times$ faster training wall-clock speed.
    
\end{abstract}

\begin{document}

\maketitle

{\renewcommand{\thefootnote}{}
\footnotetext{\textsuperscript{*}Work done during an internship at KRAFTON.}
}
\setcounter{footnote}{0}

\begin{figure}[b]
    \centering
    \includegraphics[width=\linewidth]{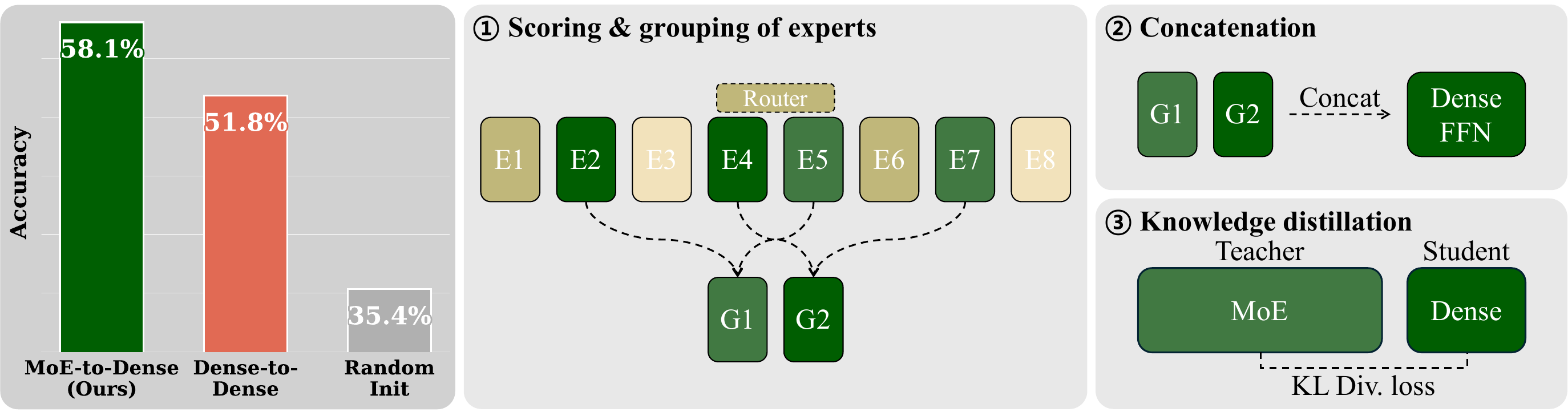}
    \caption{\textbf{Left:}~Average downstream accuracy (Section~\ref{sec:setup}) after ${\sim}$4B-token distillation, comparing three approaches to producing a 3B-parameter dense student.  \emph{MoE-to-dense} uses the MoE teacher (Qwen3-30B-A3B) to initialize and distill the student; \emph{Dense-to-dense} prunes a dense teacher of matched total parameter count (Qwen3-32B); \emph{Random Init} trains the same architecture from scratch via distillation.  \textbf{Right:}~The MoE-to-dense pipeline.  From the original MoE layer, we score and select a subset, group and merge them into $k$ groups, then concatenate into a standard dense FFN with appropriate magnitude scaling.  We evaluate 7 scoring, 5 grouping, and 2 magnitude scaling methods to identify the best configuration.}
    \label{fig:overview}
\end{figure}

\section{Introduction}
\label{sec:introduction}

Mixture-of-Experts (MoE)~\citep{shazeer2017outrageously,fedus2022switch} has become the dominant architecture for the most capable models from LLM providers~\citep{deepseekai2026deepseekv4,qwen3,llama4,zeng2026glm}
as it enables scaling the total number of parameters while remaining practical to train and serve.
However, MoE models require loading all expert parameters into memory despite activating only a fraction per token, making them less preferable for memory-constrained scenarios such as on-device deployment or single-GPU serving.
To address these deployment needs, many providers release families of small dense models alongside their flagship MoE~\citep{qwen3,gemma4,ministral3}.
These dense models are typically either trained from scratch or produced through cascaded pruning and distillation of a large dense teacher, which must itself first be trained~\citep{muralidharan2024compact,sreenivas2024llmpruning,ministral3}.
Both approaches require substantial compute and do not leverage the capability present in the MoE model, which, as we show, provides a stronger starting point for distillation than a dense teacher of matched parameter count.

Recent work has demonstrated that MoE models can be compressed into a smaller MoE by reducing the number of experts (Table~\ref{tab:framework_comparison}).
While highly effective, the resulting models still require loading all remaining experts into memory, preserving the fundamental memory inefficiency.
Although individual methods can be adapted to produce dense outputs, no prior work has systematically studied MoE-to-dense conversion.

To fill this gap, we introduce the first systematic framework for converting a trained MoE language model into a standard \emph{fully dense} architecture.
Our approach proceeds in three steps:
(1)~experts are scored by importance, and the top-scoring experts are selected, grouped, and merged within groups,
(2)~the merged experts are concatenated into a dense FFN, with down-projection matrices scaled to preserve output magnitude, and
(3)~knowledge distillation~\citep{hinton2015distilling} from the MoE teacher recovers quality lost during compression.
Figure~\ref{fig:overview} illustrates the pipeline.

\begin{table}[t]
\caption{MoE compression methods and their design choices.  Prior work produces smaller MoE models, while ours is the first to target a fully dense architecture.  See Section~\ref{sec:related_work} for a detailed comparison.}
\label{tab:framework_comparison}
\centering
\small
\setlength{\tabcolsep}{2.5pt}
\begin{tabular}{llccc}
\toprule
\textbf{Output} & \textbf{Method} & \textbf{Scoring} & \textbf{Grouping} & \textbf{Distill} \\
\midrule
\multirow{6}{*}{MoE}
& MC-SMoE~\citep{li2024mcsmoe}    & Selection freq  & Router-logit anchor-based      & \checkmark \\
& HC-SMoE~\citep{chen2025hcsmoe}    & Selection freq  & Output-sim agglomerative   & -- \\
& Sub-MoE~\citep{li2025submoe}  & Selection freq  & Output-sim K-means   & --\\
& MergeMoE~\citep{miao2025mergemoe} & Selection freq  & Weight-sim anchor-based   & --\\
& PuzzleMoE~\citep{zhao2025puzzlemoe} & Weight $\times$ activation & Entry-wise diff & --\\
& REAP~\citep{lasby2025reap}       & Gate $\times$ act.\ mag.\ & --   & --\\
\midrule
\shortstack[l]{\textbf{Dense}\\[3pt]\phantom{.}} & \shortstack[l]{\textbf{This work}\\[3pt]\phantom{.}}  & \shortstack[c]{\textbf{7 methods}\\\scriptsize(best: gate $\times$ act.\ mag.\ + diversity)}  & \shortstack[c]{\textbf{5 methods}\\[3pt]\phantom{.}}  & \shortstack[c]{\checkmark\\[3pt]\phantom{.}} \\
\bottomrule
\end{tabular}
\end{table}

To identify the best recipe, we evaluate 7 different expert scoring methods, including our novel diversity-aware scoring metric, 5 different grouping methods, 2 different down-projection scaling options, and a wide range of the number of selected experts on Qwen3-30B-A3B, resulting in 350 combinations.
Below are our key findings, which are further validated across DeepSeek-V2-Lite and GPT-OSS-20B:
\begin{enumerate}[nosep]
    \item \textbf{Expert scoring is the dominant factor.}  Best and worst scoring methods differ by 5.7~pp in average downstream accuracy, while grouping strategy contributes only ${\sim}$1~pp.  Our diversity-aware scoring achieves the best accuracy across all 35 scoring-grouping combinations (Section~\ref{sec:design_eval}).

    \item \textbf{Selecting diverse experts without merging is the best strategy.}  On all three models, the best configuration uses diversity-aware scoring with pure pruning (no weight averaging).  In contrast, frequency-based scoring selects redundant experts and benefits from merging them (Sections~\ref{sec:design_eval} and~\ref{sec:cross_model}).

    \item \textbf{MoE-to-dense outperforms dense-to-dense pruning.}  In a controlled comparison at matched total parameter count (${\sim}$30B teacher, 3B student), our best configuration outperforms dense-to-dense pruning~\citep{muralidharan2024compact} by $+$6.3~pp after ${\sim}$4B-token distillation, and is 1.6$\times$ faster in training wall-clock time (Sections~\ref{sec:design_eval} and~\ref{sec:scaling_results}).
\end{enumerate}

\noindent Our contributions are:
\begin{enumerate}[nosep]
    \item \textbf{First MoE-to-dense pruning and distillation framework.}  We present the first end-to-end study of converting a trained MoE language model into a fully dense architecture.

    \item \textbf{Diversity-aware expert selection via Gram log-determinant.}  We introduce a D-optimal selection criterion that maximizes the log-determinant of an importance-weighted Gram matrix, jointly capturing expert importance and mutual diversity.  Combined with activation-weighted scoring (DO-ACP), it achieves the best accuracy across all configurations and all three models.

    \item \textbf{Comprehensive evaluation and cross-model validation.}  We evaluate 350 configurations and identify the best recipe under 0.3B-token distillation.  The advantage persists under extended training (${\sim}$4B tokens) and generalizes to DeepSeek-V2-Lite and GPT-OSS-20B.
\end{enumerate}

\section{Related work}
\label{sec:related_work}

\paragraph{MoE expert pruning.}
A growing line of work reduces the number of active experts while retaining sparse routing.
REAP~\citep{lasby2025reap} scores experts by gate value $\times$ activation norm and removes the lowest-scoring ones. 
We adapt this principle into a factorized variant and combine it with the D-optimal selection criterion (Section~\ref{sec:scoring}).
SlimMoE~\citep{li2025slimmoe} retains all experts but progressively prunes neurons within each expert through multi-stage distillation.
MoE-Pruner~\citep{xie2024moepruner} one-shot prunes weights within experts using a scoring rule that combines weight magnitude, input activation, and router weight.
DiEP~\citep{bai2025diep} learns non-uniform pruning rates per layer via differentiable optimization.
MoE-I\textsuperscript{2}~\citep{yang2024moei2} scores experts by per-expert loss degradation, requiring $E \times L$ forward passes.
Earlier task-specific MoE pruning~\citep{chen2022task} selects experts specialized for a target downstream task.
EMO~\citep{wang2026emopretrainingmixtureexperts} modifies MoE pretraining to encourage emergent expert modularity, enabling subset pruning at deployment time.
All these methods produce \emph{smaller MoE} models, whereas our work produces a fully dense model.

\paragraph{MoE expert merging.}
Rather than removing experts, merging methods combine multiple experts into fewer groups.
REAM~\citep{jha2026ream} shows that merging experts by router-weighted averaging can outperform pruning them entirely.
MC-SMoE~\citep{li2024mcsmoe} groups experts around most frequently activated experts using router-logit similarity, then merges each group via permutation-aligned frequency-weighted averaging.
HC-SMoE~\citep{chen2025hcsmoe} showed that output-similarity clustering outperforms router-logit and weight similarity for grouping.
Sub-MoE~\citep{li2025submoe} uses output-based K-means++ clustering before joint SVD merging.
MergeMoE~\citep{miao2025mergemoe} proves the optimality of frequency-based weighting and clusters by concatenated gate/up weight similarity.
PuzzleMoE~\citep{zhao2025puzzlemoe} constructs new experts via entry-wise selection that combines weight similarity and activation-weight saliency.
NAMEx~\citep{nguyen2025namex} computes merging coefficients via Nash Bargaining optimization.
These methods all produce smaller MoE or specialized sparse models, while our pipeline uses expert selection and merging as an initialization for a dense model refined by knowledge distillation.

\paragraph{Dense model compression and distillation.}
Minitron~\citep{muralidharan2024compact,sreenivas2024llmpruning} showed that activation-based importance scoring followed by width/depth pruning and logit-level knowledge distillation produces compact LLMs competitive with models trained from scratch.
Other approaches include targeted depth removal~\citep{kim2024shortened} and joint structured pruning with continued pre-training~\citep{xia2024sheared}.
Our pipeline follows the general flow of pruning first and then recovering with knowledge distillation, but the pruning focuses on selecting a smaller set of experts and converting to a dense model, and the teacher model stays MoE during distillation.

\paragraph{Subset selection and D-optimal design.}
Selecting a diverse, high-quality subset from a large candidate pool is a classical problem in experimental design~\citep{pukelsheim2006optimal}.
The D-optimality criterion maximizes the log-determinant of an information matrix and is monotone submodular, admitting greedy $(1{-}1/e)$-approximation~\citep{nemhauser1978analysis}.
Our D-Optimal expert selection (Section~\ref{sec:scoring}) instantiates this in the expert output space: the importance-weighted Gram matrix captures both individual quality and pairwise redundancy, so maximizing its log-determinant selects experts that are jointly informative rather than individually top-ranked.
To our knowledge, this is the first application of D-Optimal subset selection to MoE compression.

\section{Method}
\label{sec:method}

We present a method for converting a Mixture-of-Experts (MoE) language model into a dense model with a total number of parameters equivalent to the active number of parameters of the teacher MoE model.  Given $E$ experts per MoE layer with $k$ routed per token, the method selects the top-$K$ experts by importance ($K \geq k$), assigns them to $k$ groups (copying directly when $K{=}k$, or merging via score-weighted averaging when $K{>}k$), and concatenates the resulting weights into a standard dense FFN.  Knowledge distillation from the MoE teacher then recovers quality lost during compression.  Throughout this section, $k$ denotes the MoE router's top-$k$ count (an architecture constant), while $K$ denotes the number of routed experts we select for conversion (a design choice).

\subsection{MoE-to-dense conversion}
\label{sec:conversion}

Consider an MoE transformer with $L$ layers, each containing a router that maps a hidden representation $\mathbf{h} \in \mathbb{R}^d$ to logits over $E$ experts.  The router selects the top-$k$ experts per token and computes a weighted combination of their outputs.  
We consider the case where each expert is a gated MLP (SwiGLU~\citep{shazeer2020glu}) with intermediate dimension $d_{\text{expert}}$, parameterized by $\mathbf{W}^{(e)}_{\text{gate}}, \mathbf{W}^{(e)}_{\text{up}} \in \mathbb{R}^{d_{\text{expert}} \times d}$ and $\mathbf{W}^{(e)}_{\text{down}} \in \mathbb{R}^{d \times d_{\text{expert}}}$.
MoE architectures typically size experts so that $d_{\text{expert}} \times k$ falls in the typical dense FFN range of ${\sim}3$--$5\times d$ (Table~\ref{tab:model_dims}), keeping the active computation per token comparable to a dense model.

\begin{table}[t]
\caption{MoE expert dimensions across models used in this work.  The active FFN width per token ($d_{\text{expert}} \times k$) falls in the typical dense range of ${\sim}3$--$5\times d$.}
\label{tab:model_dims}
\centering
\small
\begin{tabular}{lcccccc}
\toprule
\textbf{Model} & $d$ & $d_{\text{expert}}$ & $E$ & $k$ & $d_{\text{dense}} {=} d_{\text{expert}} {\times} k$ & $d_{\text{dense}}/d$ \\
\midrule
Qwen3-30B-A3B  & 2048 & 768  & 128 & 8 & 6144  & 3.0$\times$ \\
GPT-OSS-20B    & 2880 & 2880 & 32  & 4 & 11520 & 4.0$\times$ \\
DeepSeek-V2-Lite & 2048 & 1408 & 64  & 6 & 8448\rlap{$^*$} & 4.1$\times$ \\
\bottomrule
\end{tabular}

\smallskip
{\scriptsize $^*$DeepSeek-V2-Lite also has 2 shared (always-on) experts; the full dense equivalent is $2 {\times} 1408 + 6 {\times} 1408 = 11264$.}
\end{table}

We convert each MoE layer into a dense FFN by first scoring all $E$ experts by importance and retaining the top-$K$.  How experts are scored is detailed in Section~\ref{sec:scoring}; we evaluate seven methods and find that the choice of scoring is the most impactful decision.  When $K{>}k$, the selected experts are assigned to $k$ groups and merged within each group via score-weighted averaging (grouping strategies are compared in Section~\ref{sec:grouping}).  Let $s^{(e)}$ denote the importance score of expert $e$ and $\mathcal{G}_g$ the set of experts assigned to group $g$:
\begin{equation}
    \label{eq:weighted_avg}
    \mathbf{W}_{\text{proj}}^{(g)} = \sum_{e \in \mathcal{G}_g} \frac{s^{(e)}}{\sum_{e' \in \mathcal{G}_g} s^{(e')}} \, \mathbf{W}_{\text{proj}}^{(e)}, \quad \text{proj} \in \{\text{gate}, \text{up}, \text{down}\}.
\end{equation}
When $K{=}k$, each group contains exactly one expert and no merging is needed.

The $k$ resulting weight matrices are then concatenated into a single dense FFN with intermediate dimension $d_{\text{dense}} = k \times d_{\text{expert}}$:
\begin{align}
    \label{eq:concat}
    \mathbf{W}_{\text{gate}} &= \bigl[\mathbf{W}_{\text{gate}}^{(1)};\; \ldots;\; \mathbf{W}_{\text{gate}}^{(k)}\bigr] \in \mathbb{R}^{d_{\text{dense}} \times d}, \quad
    \mathbf{W}_{\text{up}} = \bigl[\mathbf{W}_{\text{up}}^{(1)};\; \ldots;\; \mathbf{W}_{\text{up}}^{(k)}\bigr] \in \mathbb{R}^{d_{\text{dense}} \times d}, \notag \\
    \mathbf{W}_{\text{down}} &= \bigl[\tilde{\mathbf{W}}_{\text{down}}^{(1)},\; \ldots,\; \tilde{\mathbf{W}}_{\text{down}}^{(k)}\bigr] \in \mathbb{R}^{d \times d_{\text{dense}}},
\end{align}
where $[\,\cdot\,;\,\cdot\,]$ denotes row-concatenation, $[\,\cdot\,,\,\cdot\,]$ column-concatenation, each $\mathbf{W}_{\text{proj}}^{(g)}$ is the weight matrix for group $g$, and $\tilde{\mathbf{W}}_{\text{down}}^{(g)} = \alpha_g \, \mathbf{W}_{\text{down}}^{(g)}$ is the down-projection scaled to approximate the average routing behavior (Section~\ref{sec:scaling}).  Block concatenation preserves the intermediate activations of the constructed group representatives exactly (Appendix~\ref{app:concat_equivalence}); when $K{=}k$, these representatives are copied experts, while for $K{>}k$ they are parameter-averaged proxies.  The final aggregation still differs from the original MoE because static $\alpha_g$ cannot simulate token-dependent router weights.
Attention layers, embeddings, and layer norms are copied unchanged from teacher to student.

The parameter $K$ controls a prune/merge tradeoff: $K{=}k$ copies experts directly (pure pruning), $K{>}k$ averages $K/k$ experts per group, and $K{=}E$ retains all experts (pure merging).  The full algorithmic pseudocode is provided in Algorithm~\ref{alg:merging} in Appendix~\ref{app:algorithm}.

\subsection{Scoring}
\label{sec:scoring}

Expert scores determine which experts to retain (top-$K$ selection) and how to weight them during merging.  We evaluate seven methods spanning three families, all computed from statistics collected in the calibration forward pass.  For each token $t$ and layer $\ell$, the router produces a distribution $p_\ell^{(e)}(t)$ via softmax over all $E$ experts, then selects $\mathcal{S}_\ell(t) = \operatorname{top\text{-}k}\{p_\ell^{(e)}(t)\}_{e=1}^{E}$.

\paragraph{Frequency-based scoring (SF, PP, PS).}
Prior work~\citep{li2024mcsmoe,chen2025hcsmoe,miao2025mergemoe} universally uses \textbf{selection frequency (SF)}, the fraction of tokens for which expert $e$ is among the top-$k$, as the importance metric.  We additionally evaluate \textbf{pre-selection probability (PP)}, the average softmax probability over all tokens regardless of selection, and \textbf{post-selection probability (PS)}, the average over all tokens, taking the softmax probability when $e$ was selected and zero otherwise.  All three favor generalist experts that appear frequently.  Formal definitions are in Appendix~\ref{app:scoring_defs}.

\begin{figure}[t]
    \centering
    \begin{minipage}[t]{0.48\linewidth}
        \centering
        \includegraphics[width=\linewidth]{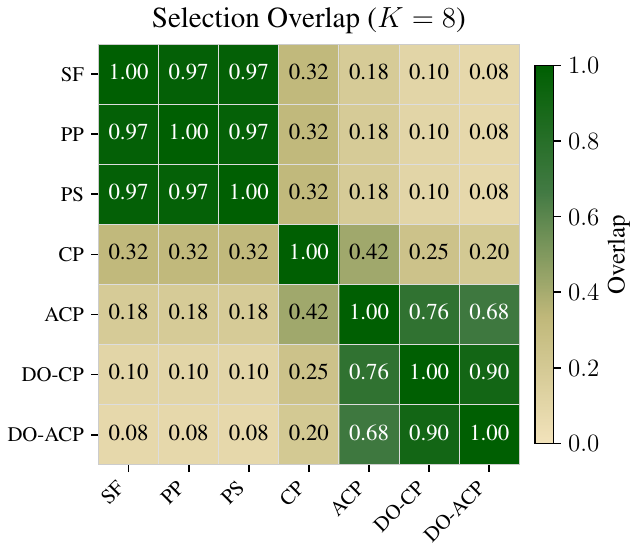}
        \caption{Selection overlap between the top-$K{=}8$ experts chosen by each scoring method, averaged across all layers of Qwen3-30B-A3B.  Frequency-based methods (SF, PP, PS) select nearly identical experts, while DO-ACP shares ${\leq}$0.08 overlap with them.}
        \label{fig:scoring_overlap}
    \end{minipage}\hfill
    \begin{minipage}[t]{0.48\linewidth}
        \centering
        \includegraphics[width=\linewidth]{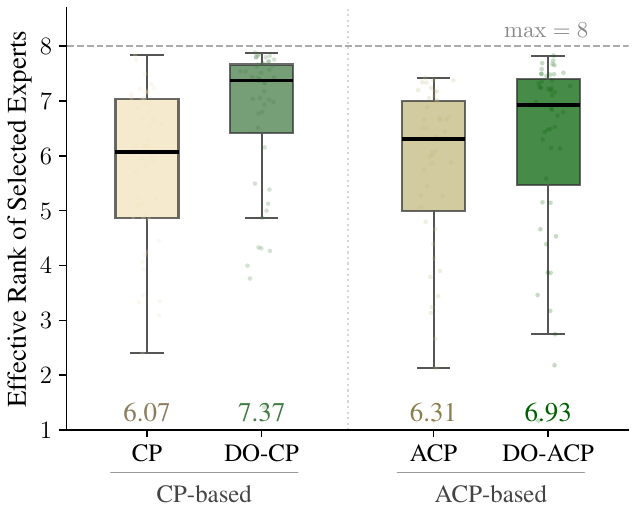}
        \caption{Effective rank~\citep{roy2007effective} of the $K{=}8$ selected expert kernel submatrix (max 8) for Qwen3-30B-A3B.  D-Optimal selection increases effective rank from 6.07 to 7.37 (CP) and 6.31 to 6.93 (ACP).}
        \label{fig:effective_rank}
    \end{minipage}
\end{figure}

\paragraph{Conditional probability (CP).}
  The frequency-based scores above conflate how often an expert is chosen with how confident the router is when it chooses.  CP isolates the latter as the average softmax probability over only the tokens where $e$ was selected:
\begin{equation}
    \label{eq:conditional_prob}
    s_\ell^{(e)} = \frac{\sum_{t : e \in \mathcal{S}_\ell(t)} p_\ell^{(e)}(t)}{\bigl|\{t : e \in \mathcal{S}_\ell(t)\}\bigr|}.
\end{equation}
CP is not diluted by frequency, so specialist experts that are rarely selected but confidently routed score highly. 
As shown in Figure~\ref{fig:scoring_overlap}, CP selects a very different top-$8$ set of experts from the frequency-based methods.

\paragraph{Activation-weighted conditional probability (ACP).}
Routing probabilities ignore the magnitude of expert outputs.  Following REAP~\citep{lasby2025reap}, which scores experts by the product of routing weight and output norm, we define a factorized variant:
$s_\ell^{(e)} = \bar{p}_\ell^{(e)} \cdot \sqrt{\mathbb{E}_{t}[\|f_e(t)\|^2]}$,
where $\bar{p}_\ell^{(e)}$ is the conditional probability from Eq.~\ref{eq:conditional_prob} and $f_e(t)$ is the output of expert $e$ on token $t$.  This factorization decouples routing confidence from output magnitude, producing a per-expert scalar that serves as the base importance for the D-optimal selection criterion we introduce next.

\paragraph{D-Optimal selection (DO).}
The methods above rank experts independently, ignoring redundancy among high-scoring experts.  We introduce a D-optimal criterion that jointly maximizes importance and diversity.  Given base importance $I_e = s^{(e)}$ (instantiated as CP or ACP) and the expert output Gram matrix $\mathbf{G}_{ij} = \mathbb{E}_t[\langle f_i(t), f_j(t) \rangle]$, we form the importance-weighted kernel $\mathcal{K}_{ij} = \sqrt{I_i I_j} \cdot \mathbf{G}_{ij}$ and select:
\begin{equation}
    \label{eq:logdet}
    S^* = \argmax_{|S|=K} \log\det\!\left(\boldsymbol{\mathcal{K}}_S + \lambda_{\text{reg}}\, \mathbf{I}\right),
\end{equation}
where $\lambda_{\text{reg}} = \frac{1}{KE}\sum_{e=1}^{E}\mathcal{K}_{ee}$.  Since exhaustive search over $\binom{E}{K}$ subsets is intractable, we solve this greedily: starting from $S = \emptyset$, we iteratively add the expert that maximally increases $\log\det(\boldsymbol{\mathcal{K}}_S + \lambda_{\text{reg}}\mathbf{I})$, computed efficiently via Schur complement evaluations in $O(K^3 E)$ total time (Algorithm~\ref{alg:do_greedy} in Appendix~\ref{app:algorithm}).  Applying DO to the two best base scores yields \textbf{DO-CP} and \textbf{DO-ACP}; DO-ACP is the overall best method (Section~\ref{sec:design_eval}).  We measure the diversity of a selected expert set by its effective rank~\citep{roy2007effective}: the number of significant independent directions in the expert output space, ranging from 1 (all experts are near-duplicates) to $K$ (maximally diverse).  Figure~\ref{fig:effective_rank} confirms that D-Optimal selection substantially increases the effective rank compared to independent scoring, avoiding redundant experts as predicted by Theorem~\ref{thm:redundancy}.  Excluded scoring methods are discussed in Appendix~\ref{app:excluded_methods}.

\paragraph{Theoretical properties.}
\label{sec:theory}
The log-determinant objective in Eq.~\ref{eq:logdet} has several useful properties that justify its use over independent scoring.  We state three results (proofs in Appendix~\ref{app:theory_proofs}).

\begin{theorem}[Independent scoring can fail under redundancy]
\label{thm:redundancy}
There exists a family of MoE layers, calibration distributions, and a positive regularization choice for which selecting the top-$K$ experts by ACP incurs a constant reconstruction error, while a size-$K$ subset maximizing the log-determinant objective achieves zero error.
\end{theorem}

If several high-scoring experts are near-duplicates, selecting all of them wastes dense capacity.  Appendix~\ref{app:proof_redundancy} constructs such a failure mode explicitly for ACP and shows that the log-determinant objective can avoid it.

\begin{proposition}[Greedy D-Optimal selection is near-optimal]
\label{prop:greedy}
For any positive semidefinite kernel $\boldsymbol{\mathcal{K}}$ and any $\lambda_{\mathrm{reg}}>0$, the normalized objective $\widetilde{F}(S) = \log\det(\mathbf{I} + \lambda_{\mathrm{reg}}^{-1}\boldsymbol{\mathcal{K}}_S)$ is monotone submodular.  The greedy algorithm returns $S_{\mathrm{greedy}}$ with $\widetilde{F}(S_{\mathrm{greedy}}) \geq (1 - 1/e) \max_{|S|=K} \widetilde{F}(S)$.
\end{proposition}

D-Optimal selection is therefore an instance of cardinality-constrained D-optimal design in expert-output space, with a standard greedy guarantee (Appendix~\ref{app:proof_submodularity}).  The result holds for any base importance $I_e$, and we instantiate it with both CP and ACP (yielding DO-CP and DO-ACP).

\begin{theorem}[When does D-Optimal selection help the most?]
\label{thm:incoherence}
Let $F(S) := \log\det(\boldsymbol{\mathcal{K}}_S + \lambda_{\text{reg}}\mathbf{I})$ denote the D-optimal objective from Eq.~\ref{eq:logdet}, and let
\begin{align*}
    G(S) := \sum_{e \in S} \log(\mathcal{K}_{ee} + \lambda_{\mathrm{reg}})
\end{align*}
and let
\begin{align*}
    S_{\mathrm{diag}} \in \argmax_{|S|=K} G(S).
\end{align*}
Let $\mu = \max_{i \neq j} |\mathcal{K}_{ij}| / \sqrt{\mathcal{K}_{ii}\mathcal{K}_{jj}}$ be the mutual coherence of the expert kernel.  If $(K-1)\mu < 1$, then
\begin{align*}
    F(S_{\mathrm{diag}}) \ge \max_{|S|=K} F(S) - K \log \left(\frac{1 + (K-1)\mu}{1-(K-1)\mu}\right).
\end{align*}
\end{theorem}

When experts are nearly incoherent ($\mu \approx 0$), the diagonal proxy is a good approximation to the full log-determinant objective, thus the off-diagonal diversity interactions contribute less (proof in Appendix~\ref{app:proof_incoherence}).  This supports the intuition that diversity corrections matter most when redundant experts create large off-diagonal interactions.  This is consistent with the cross-model trend observed in Section~\ref{sec:cross_model}, where Qwen3 (128 experts) shows a large D-Optimal gain while GPT-OSS (32 experts) shows a compressed scoring gap.  Additional theoretical results (finite-sample calibration, grouping recovery, merging optimality) are in Appendix~\ref{app:theory_extensions}.

\subsection{Grouping}
\label{sec:grouping}

After selecting the top-$K$ experts, we assign them to $k$ groups.  We evaluate five strategies: \textbf{round-robin (RR)}, which assigns score-sorted experts to groups cyclically, producing balanced groups by construction; \textbf{weight clustering (WC)} and \textbf{router clustering (RC)}, which cluster experts by weight or router-vector similarity; \textbf{anchor-based (AB)}~\citep{li2024mcsmoe}, which uses the $k$ highest-scoring experts as anchors and assigns the rest by router-vector similarity; and \textbf{output clustering (OC)}~\citep{chen2025hcsmoe}, which clusters on hidden-state outputs.  Full definitions are in Appendix~\ref{app:grouping_defs}.

\subsection{Down-projection scaling}
\label{sec:scaling}

In the original MoE, the router dynamically weights each expert's output per token.  The concatenated dense FFN has no router, so we apply static scaling factors $\alpha_g$ to each group's down-projection to approximate the average routing behavior. We consider \textbf{uniform} scaling ($\alpha_g = 1/k$, splitting routing mass equally across the $k$ groups) and \textbf{proportional} scaling ($\alpha_g \propto \sum_{e \in \mathcal{G}_g} s^{(e)}$, weighting each group by its share of selected-expert importance).  Full equations are in Appendix~\ref{app:scaling_defs}.

\subsection{Knowledge distillation}
\label{sec:distillation}

The concatenated dense model provides an initialization for knowledge distillation from the MoE teacher.  We minimize the forward KL divergence between the teacher and student output distributions:
\begin{equation}
    \label{eq:kl_loss}
    \mathcal{L}_{\text{KD}} = \frac{1}{|\mathcal{T}|} \sum_{t \in \mathcal{T}} D_{\text{KL}}\!\left(\, p_{\text{teacher}}(\cdot \mid x_t) \;\|\; p_{\text{student}}(\cdot \mid x_t) \,\right),
\end{equation}
where $\mathcal{T}$ is the set of tokens in a batch and $x_t$ is the context preceding token $t$.  Training hyperparameters are given in Section~\ref{sec:setup}.

We also evaluate \textbf{reverse KL} ($D_{\text{KL}}(p_{\text{student}} \| p_{\text{teacher}})$) and an \textbf{intermediate loss} combining logit-level KL with hidden-state MSE across all layers~\citep{muralidharan2024compact}.
 
Furthermore, motivated by~\citet{kim2025expert}, who show that non-activated experts contain valuable knowledge for distillation, we also test overriding the teacher's routing to activate $k' > k$ experts during distillation:
\begin{equation}
    \label{eq:allexpert_teacher}
    F_{\text{teacher}}(t) = \sum_{e \in \mathcal{S}_{k'}(t)} p_\ell^{(e)}(t)\, f_e(t),
\end{equation}
where $\mathcal{S}_{k'}(t)$ is the set of $k'$ highest-probability experts for token $t$.  This exposes the student to knowledge from experts the router would not normally select, at the cost of ${\sim}k'/k$ times more expert activations per MoE layer.
Results on these distillation methods are in Section~\ref{sec:design_eval}.

\section{Experiments}
\label{sec:experiments}

We validate our pipeline on Qwen3-30B-A3B, converting the 30B-parameter MoE into a 3.3B dense model.  Section~\ref{sec:setup} describes the setup and experimental workflow.  Section~\ref{sec:design_eval} evaluates every possible scoring$\times$grouping combination after knowledge distillation.  Section~\ref{sec:distill_exploration} explores distillation method choices.  Section~\ref{sec:scaling_results} reports extended training, and Section~\ref{sec:qualitative} provides a qualitative analysis of these checkpoints.  Section~\ref{sec:cross_model} validates findings on DeepSeek-V2-Lite and GPT-OSS-20B.

\subsection{Experimental setup}
\label{sec:setup}

\paragraph{Workflow.}
Our evaluation proceeds in three stages.  First, we sweep all combinations of 7 scoring methods, 5 grouping strategies, 2 down-projection (DP) scalings, and $K \in \{8, 16, 32, 64, 128\}$, yielding 350 configurations.  Each is evaluated by WikiText-2 perplexity before distillation (we refer to it as pre-distill PPL), measuring initialization quality.  Second, for each of the 35 scoring$\times$grouping pairs, we select the best $(K, \text{DP scaling})$ by pre-distill PPL and distill for 0.3B tokens.  Third, all distilled models are evaluated on five downstream benchmarks.

\paragraph{Models.}
Our primary teacher is \textbf{Qwen3-30B-A3B}~\citep{qwen3}, a 30B-parameter MoE with $E{=}128$ experts and $k{=}8$ active per token.  We use the post-trained variant. 
(We find that the main findings are consistent when using the base model as the teacher model, see Appendix~\ref{sec:base_vs_instruct} for a detailed comparison.)
Regardless of the teacher variant, the student is a base model, distilled on a generic web corpus without instruction tuning.
The dense student has 3.3B parameters with $d_{\text{dense}} = k \times d_{\text{expert}} = 6{,}144$.  For cross-model validation (Section~\ref{sec:cross_model}), we test two additional architectures: \textbf{DeepSeek-V2-Lite}~\citep{dai2024deepseekmoe,deepseekai2024deepseekv2}, a 16B MoE with $E{=}64$ routed experts plus 2 shared experts and $k{=}6$ (2.4B active), and \textbf{GPT-OSS-20B}~\citep{openai2025gptoss}, a 21B MoE with $E{=}32$ experts and $k{=}4$ (3.6B active).  Model dimensions are in Table~\ref{tab:model_dims}.

\paragraph{Training.}
Expert importance is calibrated on 512 WikiText-103~\citep{merity2017wikitext} sequences of up to 2048 tokens.  Expert selection is robust to calibration size and sequence length, and changing the calibration domain moves final accuracy by only 1.0--1.6~pp (Appendix~\ref{app:calibration}).  Unless otherwise noted, distillation uses forward KL (Eq.~\ref{eq:kl_loss}) on FineWeb-Edu~\citep{fineweb} with global batch size 384, sequence length 4096, learning rate $10^{-4}$ with cosine decay to $10^{-5}$, and AdamW ($\beta_1{=}0.9$, $\beta_2{=}0.95$).  Full hyperparameters and infrastructure details are in Appendix~\ref{app:hyperparameters}.

\paragraph{Evaluation.}
Following~\citet{muralidharan2024compact,sreenivas2024llmpruning}, downstream accuracy is measured on Winogrande (5-shot), HellaSwag (10-shot), ARC-Easy (25-shot), ARC-Challenge (25-shot), and MMLU (5-shot).  We report the unweighted average across all five tasks.

\subsection{Scoring and grouping evaluation}
\label{sec:design_eval}

From the 350-configuration pre-distill PPL sweep (see Appendix~\ref{app:full_grid} for the complete 350-configuration pre-distill PPL grid), the best $K$ for each scoring$\times$grouping pair is overwhelmingly $K{=}8$ or $K{=}16$ (32 of 35 configs, only 3 select $K{\geq}32$), indicating that retaining a small number of high-quality experts is more effective than merging many.  We distill all 35 best scoring$\times$grouping combinations for 0.3B tokens each and evaluate on five downstream benchmarks (Figure~\ref{fig:design_axis}).  The full 35-combination accuracy matrix with per-benchmark breakdown is in Appendix~\ref{app:full_distillation} (Table~\ref{tab:distillation}).

\begin{figure}[t]
    \centering
    \includegraphics[width=\linewidth]{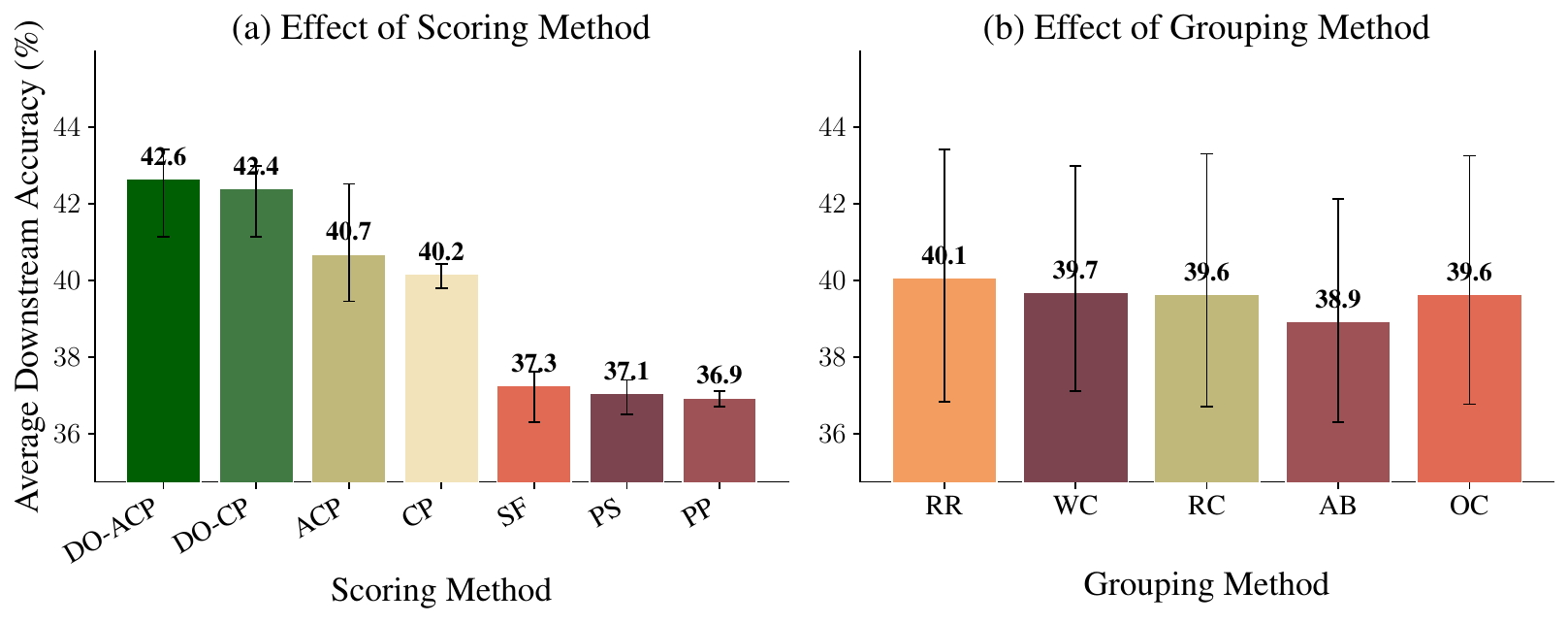}
    \caption{Marginal effect of each design axis on post-distillation downstream accuracy (0.3B tokens).  \textbf{(a)}~Scoring: DO-ACP dominates (42.64\% mean), with a 5.7~pp gap to PP (36.94\%).  \textbf{(b)}~Grouping: small effect (${\sim}$1~pp spread); round-robin (40.08\%) narrowly leads.}
    \label{fig:design_axis}
\end{figure}

\paragraph{Specialist selection and diversity are crucial.}
\label{sec:scoring_comparison}
Three distinct tiers emerge across the 35-combination grid (Figure~\ref{fig:design_axis}), with a 5.7~pp gap between the best and worst scoring families, roughly $5\times$ the grouping spread (1.2~pp).  At the bottom tier (${\sim}$37\%), frequency-based methods (SF, PP, PS) favor generalist experts that serve many tokens but capture broadly shared features; they select nearly identical expert sets (Figure~\ref{fig:scoring_overlap}), explaining their tight clustering.  The middle tier (40--41\%) shows that \emph{focusing on rare-but-confident experts matters}: CP removes the frequency factor, promoting specialists that are rarely selected but confidently routed, yielding $+$3~pp.  ACP further adds output magnitude information ($+$0.5~pp).  The top tier (42--43\%) demonstrates that \emph{diversity among selected experts matters}: DO-ACP and DO-CP apply the D-optimal criterion to explicitly penalize redundancy, gaining another $+$2~pp.  This progression, from generalist to specialist to diverse specialist, demonstrates that routing confidence and inter-expert diversity provide complementary signals for expert selection, with additive gains when combined.

\paragraph{Grouping matters less when few experts are selected.}
The best-performing configurations all use small $K$ (8 or 16), which limits the role of grouping (Appendix~\ref{app:full_grid}).  At $K{=}8$ ($= k$), each group contains exactly one expert, making grouping completely irrelevant; all five strategies produce identical results.  At $K{=}16$, pre-distill PPL varies by up to ${\sim}$3$\times$ across grouping methods within a scoring family (ACP 2{,}002--6{,}334, CP 2{,}148--6{,}033; Table~\ref{tab:distillation}), but distillation largely compensates and the post-distill PPL spread collapses to within 3 perplexity points.  Across the 35 best-per-row configurations, the grouping-induced accuracy spread remains small: RR achieves the highest column mean at 40.08\%, the three clustering methods trail by only ${\sim}$0.4~pp (WC 39.70\%, RC 39.64\%, OC 39.63\%), and AB underperforms at 38.92\% (Figure~\ref{fig:design_axis}).

\paragraph{Pure pruning outperforms merging for diversity-aware scoring.}
\label{sec:pruning_vs_merging}
The top-7 configurations all use $K{=}8$ (pure pruning: one expert per group, no weight averaging; see Table~\ref{tab:distillation}).  Although $K{>}8$ yields lower pre-distill PPL, distillation reverses the ranking: ACP $K{=}8$ reaches 42.52\% post-distill vs.\ 40.50\% for ACP$\times$OC $K{=}16$, a $+$2.0~pp gap.  The same holds with the base teacher (Appendix~\ref{sec:base_vs_instruct}); cross-model results reveal a more nuanced interaction between scoring and $K$ (Section~\ref{sec:cross_model}).  DO-ACP at $K{=}8$ achieves the highest average accuracy across all 35 configurations (43.41\%; Table~\ref{tab:distillation}).  Since $K{=}8$ dominates and grouping is irrelevant at $K{=}k$, we refer to the best configuration simply as DO-ACP hereafter.

\begin{table}[t]
\caption{Comparison to baselines on Qwen3-30B-A3B (0.3B-token distillation).  DO-ACP at $K{=}8$ uses our diversity-aware scoring; SF$\times$AB and SF$\times$OC use the metrics of MC-SMoE~\citep{li2024mcsmoe} and HC-SMoE~\citep{chen2025hcsmoe} repurposed for a dense student.  D2D pruning is Dense-to-dense pruning of a matched-parameter dense teacher (Qwen3-32B).  Random FFN + teacher attn copies the teacher's attention with a random FFN.  Random initialization trains the same architecture from scratch.  Iterative D2D reports the strongest of three prune-and-distill schedules at the same token budget (Appendix~\ref{app:d2d}).}
\label{tab:baselines}
\centering
\small
\begin{tabular}{l ccccc c}
\toprule
\textbf{Configuration} & \textbf{Wino} & \textbf{Hella} & \textbf{ARC-E} & \textbf{ARC-C} & \textbf{MMLU} & \textbf{Avg (\%)} \\
\midrule
DO-ACP, $K{=}8$ (best) & \textbf{57.0} & \textbf{41.1} & \textbf{57.4} & \textbf{29.9} & \textbf{31.7} & \textbf{43.41} \\
SF$\times$AB (MC-SMoE metrics) & 52.2 & 31.6 & 47.1 & 23.1 & 27.5 & 36.31 \\
SF$\times$OC (HC-SMoE metrics) & 53.1 & 33.1 & 49.3 & 24.7 & 27.3 & 37.52 \\
\midrule
One-shot D2D pruning (Qwen3-32B $\to$ 3.4B) & 51.7 & 27.7 & 38.9 & 22.4 & 25.7 & 33.28 \\
Iterative D2D pruning (best schedule) & 52.9 & 36.8 & 55.6 & 28.7 & 27.0 & 40.20 \\
Random FFN + teacher attn & 53.4 & 28.0 & 34.7 & 22.2 & 25.2 & 32.70 \\
Random initialization & 51.2 & 25.3 & 28.9 & 22.4 & 23.0 & 30.15 \\
\midrule
\rowcolor{gray!10}
Teacher (Qwen3-30B-A3B) & 82.6 & 82.0 & 85.3 & 65.3 & 83.2 & 79.65 \\
\rowcolor{gray!10}
D2D Teacher (Qwen3-32B) & 76.8 & 84.0 & 89.9 & 73.1 & 81.9 & 81.14 \\
\bottomrule
\end{tabular}
\end{table}

\paragraph{Comparison to baselines.}
\label{sec:baselines}
We find that MC-SMoE's~\citep{li2024mcsmoe} (SF$\times$AB, 36.31\%) and HC-SMoE's~\citep{chen2025hcsmoe} (SF$\times$OC, 37.52\%) scoring$\times$grouping metrics, evaluated within our framework, lag DO-ACP by 7.1~pp and 5.9~pp (Table~\ref{tab:baselines}).  Beyond our framework, the strongest comparison is dense-to-dense (D2D) pruning: following the Minitron approach~\citep{muralidharan2024compact}, we search over five student architectures at matched parameter count (${\sim}$3.4B) by varying pruning strategy (width-only vs.\ width+depth), number of layers, and hidden dimensions (Appendix~\ref{app:d2d}, Table~\ref{tab:d2d_configs}).  We select the best candidate by pre-distill PPL (3.44B, width-only, all 64 layers preserved) and distill it with its dense teacher (Qwen3-32B) using the same hyperparameters and token budget as all our experiments.  With this setup, D2D reaches only 33.28\%, 10.1~pp below our best configuration (DO-ACP, 43.41\%).  D2D is also barely above the \emph{random FFN} baseline (32.70\%, $+$0.6~pp), which copies the teacher's attention layers but initializes FFN weights randomly; this gap suggests that one-shot dense pruning provides little structural advantage at this compression ratio.  A full random initialization baseline (30.15\%) confirms that expert structure provides a much stronger initialization than random weights.  Iterative prune-and-distill narrows the gap, with the strongest of three schedules reaching 40.20\%, but DO-ACP still leads it by $+$3.2~pp and on every benchmark (Appendix~\ref{app:d2d}).

\subsection{Distillation method exploration}
\label{sec:distill_exploration}

Starting from the best initialization (DO-ACP, $K{=}8$), we explore whether distillation method choices can further improve quality.

\paragraph{Expanded teacher routing.}
\citet{kim2025expert} show that non-activated MoE experts contain useful knowledge for distillation.  We test overriding the teacher's routing to activate $k' > k$ experts (Eq.~\ref{eq:allexpert_teacher}), sweeping $k' \in \{8, 16, 32, 64, 96, 128\}$.
Consistent with~\citet{kim2025expert}, an intermediate routing breadth ($k'{=}16$, i.e., $2k$) improves quality by $+$0.70~pp (Table~\ref{tab:all_expert}), likely because the next-highest-scoring experts provide complementary knowledge.  Beyond $k'{=}32$, performance degrades monotonically, likely because low-ranked experts add noise.  However, the gain is modest relative to the ${\sim}2\times$ teacher FLOPs per MoE layer, so we keep standard routing by default.

\begin{table}[t]
\caption{Effect of expanded teacher routing during distillation (0.3B tokens each).  All configurations use the same pruned student (DO-ACP).  An intermediate $k'{=}16$ yields the best accuracy ($+$0.70~pp), not the default 8 or all 128 experts.}
\label{tab:all_expert}
\centering
\small
\begin{tabular}{r ccccc c c}
\toprule
\textbf{$k'$} & \textbf{Wino} & \textbf{Hella} & \textbf{ARC-E} & \textbf{ARC-C} & \textbf{MMLU} & \textbf{Avg (\%)} & $\Delta$ \\
\midrule
8 (default) & 57.0 & \textbf{41.1} & 57.4 & 29.9 & 31.7 & 43.41 & -- \\
\textbf{16} & \textbf{57.9} & 40.9 & \textbf{57.8} & \textbf{31.1} & \textbf{32.8} & \textbf{44.11} & $+$0.70 \\
32          & 57.1 & 40.6 & 57.0 & 31.0 & 32.5 & 43.63 & $+$0.22 \\
64          & 55.0 & 39.5 & 56.1 & 30.8 & 32.1 & 42.71 & $-$0.70 \\
96          & 56.0 & 38.8 & 55.8 & 29.8 & 31.6 & 42.39 & $-$1.02 \\
128 (all)   & 55.6 & 38.3 & 54.5 & 29.0 & 31.6 & 41.78 & $-$1.63 \\
\bottomrule
\end{tabular}
\end{table}

\begin{table}[t]
\caption{Loss function ablation (0.3B tokens each).  All use the same pruned student (DO-ACP).  Forward KL performs best, while reverse KL and intermediate loss both degrade quality.}
\label{tab:loss_ablation}
\centering
\small
\begin{tabular}{l ccccc c c}
\toprule
\textbf{Loss} & \textbf{Wino} & \textbf{Hella} & \textbf{ARC-E} & \textbf{ARC-C} & \textbf{MMLU} & \textbf{Avg (\%)} & $\Delta$ \\
\midrule
\textbf{Forward KL} & \textbf{57.0} & \textbf{41.1} & \textbf{57.4} & \textbf{29.9} & \textbf{31.7} & \textbf{43.41} & -- \\
Intermediate (logit + hidden MSE) & 55.5 & 39.1 & 53.3 & 28.7 & 30.9 & 41.50 & $-$1.91 \\
Reverse KL & 53.0 & 33.8 & 42.4 & 25.6 & 31.0 & 37.17 & $-$6.24 \\
\bottomrule
\end{tabular}
\end{table}

\paragraph{Loss function ablation.}
We compare forward KL (baseline), reverse KL, and forward KL augmented with intermediate hidden-state MSE loss~\citep{muralidharan2024compact}.  Forward KL substantially outperforms both alternatives (Table~\ref{tab:loss_ablation}).  Reverse KL loses $-$6.24~pp, likely because the student benefits more by recovering the full teacher distribution rather than concentrating on high-probability modes.  Adding intermediate hidden-state MSE also hurts ($-$1.91~pp).  This is consistent with Minitron's recommendation of logit-only distillation when depth is not significantly reduced, as is the case for our student~\citep{muralidharan2024compact}.

For the extended training in the next section, we use forward KL with standard teacher routing ($k'{=}k$) to maximize training throughput.

\subsection{Extended training}
\label{sec:scaling_results}

\begin{wrapfigure}{r}{0.47\textwidth}
    \vspace{-12pt}
    \centering
    \includegraphics[width=0.46\textwidth]{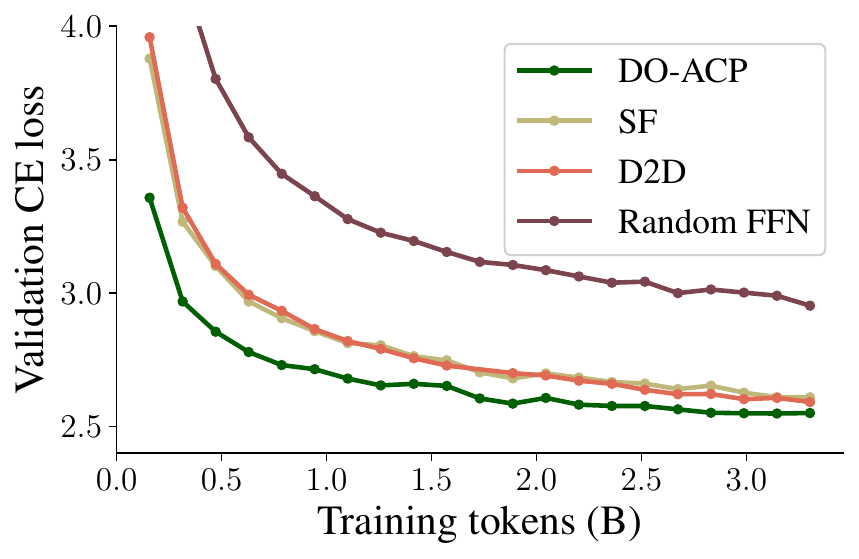}
    \caption{Validation CE loss during extended training.  DO-ACP maintains the lowest loss throughout.}
    \label{fig:scaleup_val_loss}
    \vspace{-10pt}
\end{wrapfigure}

All results above use 0.3B-token distillation.  To investigate whether our findings persist at scale, we train four configurations for ${\sim}$4B tokens.  The four configurations are: \textbf{DO-ACP} $K{=}8$ (our best), \textbf{SF} $K{=}16$ with weight clustering (best selection-frequency configuration from Section~\ref{sec:scoring_comparison}), \textbf{D2D} (one-shot dense-to-dense pruning baseline), and \textbf{Random FFN} (random FFN + teacher attention).

\begin{table}[t]
\caption{Extended training results after ${\sim}$4B tokens.  The ranking established at 0.3B tokens holds at scale: DO-ACP maintains its advantage, reaching 58.10\% average accuracy, $+$6.3~pp over D2D and $+$12.7~pp over random FFN.}
\label{tab:scaling}
\centering
\small
\begin{tabular}{l ccccc c}
\toprule
\textbf{Configuration} & \textbf{Wino} & \textbf{Hella} & \textbf{ARC-E} & \textbf{ARC-C} & \textbf{MMLU} & \textbf{Avg (\%)} \\
\midrule
DO-ACP, $K{=}8$ & \textbf{63.1} & \textbf{60.3} & \textbf{75.6} & \textbf{45.4} & \textbf{46.1} & \textbf{58.10} \\
SF, $K{=}16$ & 61.2 & 56.3 & 74.0 & 43.1 & 32.7 & 53.46 \\
D2D pruning (Qwen3-32B $\to$ 3.4B) & 60.5 & 57.5 & 73.1 & 41.5 & 26.6 & 51.84 \\
Random FFN + teacher attn & 54.4 & 45.4 & 66.0 & 34.2 & 27.1 & 45.44 \\
\midrule
\rowcolor{gray!10}
Qwen3-1.7B (pretrained reference) & 66.1 & 67.1 & 81.9 & 55.5 & 62.6 & 66.63 \\
\rowcolor{gray!10}
Qwen3-4B (pretrained reference) & 72.0 & 75.8 & 86.2 & 64.6 & 73.1 & 74.34 \\
\bottomrule
\end{tabular}
\end{table}

The ranking established at 0.3B tokens holds at scale.  DO-ACP reaches \textbf{58.10\%} average accuracy, outperforming D2D pruning (51.84\%) by $+$6.3~pp, SF (53.46\%) by $+$4.6~pp, and random FFN (45.44\%) by $+$12.7~pp (Table~\ref{tab:scaling}).
The gap is especially pronounced on MMLU: DO-ACP reaches 46.1\% versus 32.7\% (SF), 26.6\% (D2D), and 27.1\% (random), suggesting that diversity-aware expert selection may particularly benefit knowledge-intensive tasks during extended training.
Figure~\ref{fig:scaleup_val_loss} confirms that DO-ACP maintains the lowest validation CE loss throughout training.
On code, math, instruction following, long context, multilingual, and agentic tasks, DO-ACP also leads on seven of nine measures (Appendix~\ref{app:extended_eval}).

MoE-to-dense distillation is also faster than dense-to-dense (D2D) pruning: 73\,s/step vs.\ 116\,s/step on identical hardware (2$\times$B200 GPUs, GBS$=$384), a \textbf{1.6$\times$ speedup}.  This is due to the more efficient inference of the MoE teacher (Qwen3-30B-A3B), which activates only 3B parameters per token, while the dense teacher (Qwen3-32B) runs all 32B.  Combined with the accuracy advantage over one-shot D2D, MoE-to-dense produces a better student in less wall-clock time.  At inference, the student has 9.14$\times$ fewer parameters and uses 3.0$\times$ less peak VRAM than its MoE parent, with 1.46$\times$ higher throughput (Appendix~\ref{app:deployment}).

\subsection{Qualitative analysis}
\label{sec:qualitative}

To complement the benchmark evaluation, we analyze the generation quality of the four model checkpoints from Section~\ref{sec:scaling_results}.  Each model is prompted on 567 MMLU samples (10 per subject, 57 subjects) with a chain-of-thought zero-shot format and generates free-form responses (temperature 0.7, up to 2048 tokens).  We classify each response into six categories using rule-based heuristics for surface-level failures and LLM-as-a-judge (Claude Opus 4.6) for semantic errors: \emph{incoherent} (nonsensical output), \emph{repetitive loop} (same phrases cycling without progress), \emph{knowledge error} (coherent but factually wrong), \emph{reasoning error} (flawed logic in STEM), and \emph{other} (topic drift, truncation, out-of-range).  Category definitions and representative examples are in Appendix~\ref{app:error_taxonomy}.

\begin{figure}[t]
    \centering
    \includegraphics[width=\linewidth]{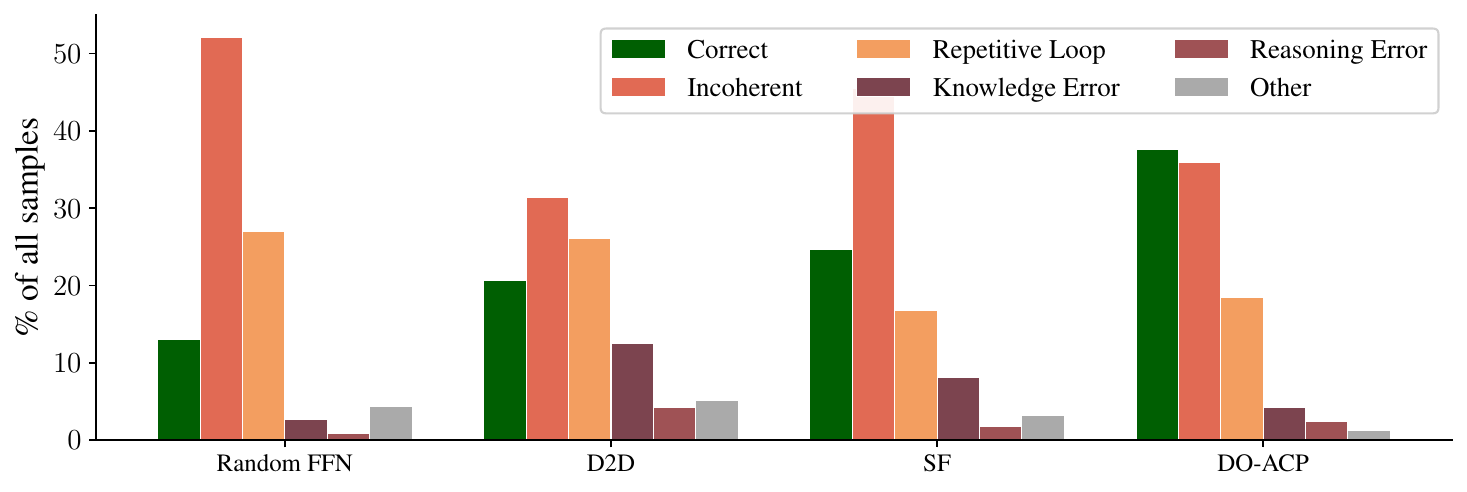}
    \caption{Error distribution on MMLU chain-of-thought zero-shot (567 samples, 4B-token models).  As model quality improves (left to right), catastrophic failures (incoherent, repetitive) decrease and the correct rate rises.  DO-ACP achieves the lowest total failure rate while maintaining the lowest knowledge error rate.}
    \label{fig:error_analysis}
\end{figure}

\paragraph{Catastrophic failures.}
Figure~\ref{fig:error_analysis} shows the error distribution.  We define catastrophic failures as responses that never reach meaningful reasoning: incoherent outputs and repetitive loops.  DO-ACP has the lowest total catastrophic failure rate (54.5\%), followed by D2D (57.5\%), SF (62.3\%), and Random FFN (79.0\%).  The pattern differs by failure mode: D2D has the lowest incoherent rate (31.4\%), while SF has the lowest repetitive loop rate (16.8\%).  DO-ACP achieves the best overall rate by reducing both failure modes simultaneously.

\paragraph{Knowledge errors.}
Random FFN produces almost no knowledge errors (2.6\%) because the vast majority of its outputs are catastrophic failures that never reach the stage of factual reasoning.  Among the three models that do attempt reasoning, knowledge error rates differ substantially: D2D shows 12.5\%, SF shows 8.1\%, and DO-ACP shows only 4.2\%.  DO-ACP achieves the highest accuracy (37.6\%) with the lowest knowledge error rate, indicating that diversity-aware expert selection preserves factual knowledge more effectively than both dense-to-dense pruning and frequency-based expert selection.

\paragraph{Subject-level breakdown.}
Breaking down accuracy by MMLU subject category confirms this pattern.  DO-ACP reaches 49.2\% on humanities, compared to 25.4\% (SF), 19.2\% (D2D), and 18.5\% (Random FFN).  The gap is largest on knowledge-intensive subjects (humanities $+$24~pp over SF, social sciences $+$11~pp) and smallest on STEM ($+$6~pp), where all models struggle with mathematical reasoning.  This suggests that the experts selected by DO-ACP carry richer factual knowledge, consistent with the low knowledge error rate observed above.

\subsection{Cross-model validation}
\label{sec:cross_model}

To test whether our findings generalize beyond Qwen3, we apply the pipeline to two additional MoE architectures: \textbf{DeepSeek-V2-Lite}~\citep{deepseekai2024deepseekv2} (16B, 64 routed + 2 shared experts, top-6) and \textbf{GPT-OSS-20B}~\citep{openai2025gptoss} (21B, 32 experts, top-4).  These span different expert counts (32--128), routing widths (4--8), and training stages (DeepSeek-V2-Lite is a base model; GPT-OSS-20B is post-trained).  Figure~\ref{fig:cross_model} summarizes results; full per-benchmark tables are in Appendix~\ref{app:cross_model_tables}.

\paragraph{Architectural adjustments.}
For each model, the dense FFN width is set to match the total active FFN computation per token.  Full details are in Appendix~\ref{app:architecture}.
\begin{itemize}[nosep]
    \item \textbf{DeepSeek-V2-Lite} has two shared (always-on) experts per MoE layer alongside $k{=}6$ routed experts.  Since $2 + 6 = 8$ FFN components are always active, $d_{\text{dense}} = 8 \times 1408 = 11{,}264$.  The shared experts are copied directly with no DP scaling (scale $= 1.0$), since router weighting does not apply to them, and $K$ applies only to routed experts.  The model does not renormalize routing probabilities after top-$k$, so the routed experts' weights do not sum to one.  We therefore use each group's average conditional probability as the DP scaling ratio, matching the average scaling the MoE applies to each expert's output.  The model's very first layer is a standard dense FFN whose intermediate dimension is 10{,}944 and is thus zero-padded to 11{,}264 to match subsequent layers.
    \item \textbf{GPT-OSS-20B} has $E{=}32$ experts with $k{=}4$, giving $d_{\text{dense}} = 4 \times 2880 = 11{,}520$.  Both Qwen3 and GPT-OSS renormalize after top-$k$, so we sweep uniform and proportional scaling (Section~\ref{sec:scaling}).  GPT-OSS is a post-trained reasoning model that generates chain-of-thought traces before answers; we evaluate all models (teacher and students) in standard completion mode for consistency with the other two models, which underestimates the teacher's native-format capability (e.g., MMLU 49\% in completion mode vs.\ 72\% with chat template).
\end{itemize}

\begin{figure}[t]
    \centering
    \includegraphics[width=\linewidth]{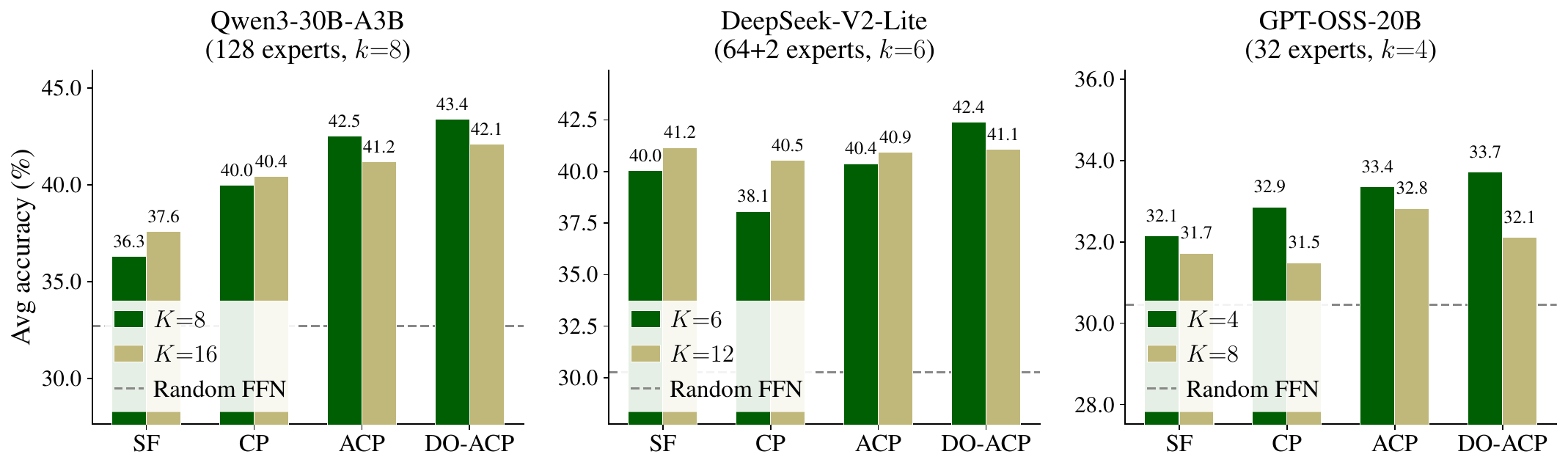}
    \caption{Cross-model validation: post-distillation accuracy (0.3B tokens) by scoring method and $K$ on three MoE architectures.  Green bars: pure pruning ($K{=}k$, one expert per slot), coral bars: merging ($K{>}k$, multiple experts averaged per slot).  Dashed line: random FFN baseline.  The scoring hierarchy and the benefit of pure pruning both persist across architectures, though the gap compresses with smaller expert pools.}
    \label{fig:cross_model}
\end{figure}

\paragraph{The best configuration uses pure pruning on all architectures.}
On all three models, the single best configuration uses pure pruning ($K{=}k$): DO-ACP at $K{=}8$ on Qwen3 (Section~\ref{sec:design_eval}), DO-ACP at $K{=}6$ on DeepSeek (42.39\% vs.\ 41.07\% at $K{=}12$), and DO-ACP at $K{=}4$ on GPT-OSS (33.71\% vs.\ 32.11\% at $K{=}8$).  However, the interaction between scoring and $K$ is more nuanced: on DeepSeek, merging ($K{=}12$) outperforms pure pruning ($K{=}6$) for SF ($+$1.1~pp), CP ($+$2.5~pp), and ACP ($+$0.6~pp), while only DO-ACP benefits from pure pruning ($+$1.3~pp).  This suggests that diversity-aware scoring selects experts that are individually strong enough to stand alone, whereas frequency-based methods select redundant generalists that benefit from averaging.

\paragraph{Selection benefit is smaller on models with smaller expert pools.}
DO-ACP and ACP remain among the top scoring methods on both additional models, but the advantage of diversity-aware scoring diminishes with smaller expert pools.  The scoring gap (best vs.\ worst method at $K{=}k$) shrinks from 7.1~pp on Qwen3 (128 experts) to 4.3~pp on DeepSeek (64 experts) to 1.6~pp on GPT-OSS (32 experts).  The overall benefit of MoE-to-dense over random baselines follows the same trend: $+$13.3~pp, $+$12.1~pp, and $+$3.7~pp respectively.  We report this as an observed cross-model trend rather than a causal conclusion, since GPT-OSS also differs in its post-training and native reasoning format.  One plausible explanation is redundancy in the expert pool: with 128 experts and $K{=}8$, many candidates are interchangeable, giving diversity-aware methods room to outperform naive selection, while with only 32 experts and $K{=}4$, each expert handles a larger share of tokens, leaving less room for scoring to differentiate.  On Qwen3, post-training alone shifts the initialization advantage by only 0.8~pp (Appendix~\ref{sec:base_vs_instruct}), which is consistent with, but does not isolate, an effect of pool size.

\subsection{Compatibility with modularity-aware pretraining}
\label{sec:emo}

\begin{wrapfigure}{r}{0.47\textwidth}
    \vspace{-12pt}
    \centering
    \includegraphics[width=0.46\textwidth]{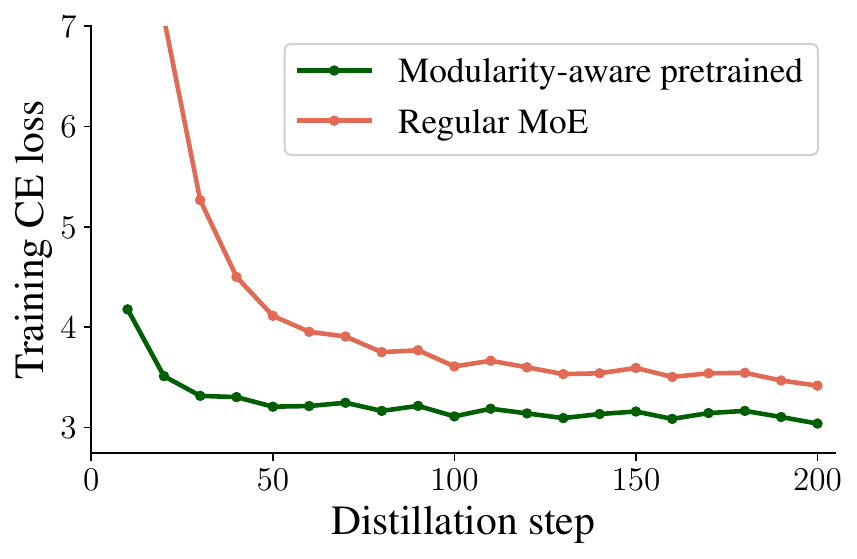}
    \caption{Cross-entropy on the training data during MoE-to-dense distillation (DO-ACP, $K{=}k$, forward-KL objective) from matched modularity-aware pretrained and regular MoE teachers.  The modularity-aware student remains lower throughout.}
    \label{fig:emo_ce_curves}
    \vspace{-10pt}
\end{wrapfigure}

We have so far evaluated our pipeline on standard MoE teachers.
An approach complementary to ours, introduced by \citet{wang2026emopretrainingmixtureexperts}, enhances the compressibility of MoE models during pretraining by constraining each training document to route through a sampled per-document expert pool.  The resulting model retains downstream performance under expert subset pruning.
While the original focus of the method is on expert-pruned models that retain the router structure, we investigate its compatibility with our pipeline, which produces a fully dense student.

We experiment with the open-source weights of \citet{wang2026emopretrainingmixtureexperts}: a 14B-parameter MoE with $E{=}127$ routed experts plus 1 shared expert and $k{=}7$ (1B active).
As a control, we use the matched regularly pretrained MoE released alongside, sharing architecture and training data.
We apply our strongest configuration (DO-ACP scoring with pure pruning at $K{=}k{=}7$ routed experts plus the model's single shared expert) to both teachers, yielding two 1.5B dense students.
Distillation follows the same protocol as Section~\ref{sec:setup}.

Table~\ref{tab:emo_pair} reports downstream accuracy and WikiText-2 perplexity before and after distillation.
The modularity-aware pretrained teacher yields a $+$3.61~pp downstream lift and an ${\sim}87{\times}$ lower pre-distill PPL relative to the regular MoE control.
Figure~\ref{fig:emo_ce_curves} shows that this initialization advantage persists throughout distillation: the student distilled from the modularity-aware teacher begins at substantially lower training CE and remains below its regular MoE counterpart.
We view this as preliminary evidence that compression-aware pretraining could pair well with our MoE-to-dense pipeline, and leave a fuller co-design of the two stages to future work.

\begin{table}[t]
\caption{Compatibility with modularity-aware pretraining (DO-ACP, $K{=}k{=}7$).  The modularity-aware pretrained teacher yields both lower pre-distill PPL and higher downstream accuracy than the regular MoE control.  \emph{Pre}/\emph{Post} = WikiText-2 PPL before/after distillation.}
\label{tab:emo_pair}
\centering
\small
\begin{tabular}{l cc cccccc}
\toprule
 & \multicolumn{2}{c}{\textbf{PPL} ($\downarrow$)} & \multicolumn{6}{c}{\textbf{Downstream Accuracy (\%)}} \\
\cmidrule(lr){2-3} \cmidrule(lr){4-9}
\textbf{Teacher} & \textbf{Pre} & \textbf{Post} & \textbf{Wino} & \textbf{Hella} & \textbf{ARC-E} & \textbf{ARC-C} & \textbf{MMLU} & \textbf{Avg} \\
\midrule
Regular MoE                & 13{,}549 & 31.13 & 52.6 & 33.8 & 44.5 & 23.0 & 24.3 & 35.67 \\
Modularity-aware pretrained & \textbf{156} & \textbf{19.76} & \textbf{56.9} & \textbf{40.1} & \textbf{49.1} & \textbf{25.1} & \textbf{25.2} & \textbf{39.28} \\
\bottomrule
\end{tabular}
\end{table}

\section{Conclusion}
\label{sec:conclusion}

We presented the first systematic framework for converting a Mixture-of-Experts language model into a fully dense architecture via expert scoring, grouping, and concatenating, followed by knowledge distillation.  We introduced a diversity-aware expert selection criterion (DO-ACP) based on the Gram log-determinant that jointly maximizes expert importance and mutual diversity.  We evaluated 350 scoring$\times$grouping$\times$scaling$\times$$K$ combinations on Qwen3-30B-A3B (128~experts) and validated findings on DeepSeek-V2-Lite (64+2~experts) and GPT-OSS-20B (32~experts).

Our evaluation reveals that expert scoring is the dominant design axis (5.7~pp spread vs.\ ${\sim}$1~pp for grouping), with DO-ACP achieving the best accuracy across all configurations and all three models.  Diversity-aware scoring enables effective pure pruning: on every architecture, the best configuration retains exactly $K{=}k$ experts with no weight averaging.  MoE-to-dense outperforms dense-to-dense pruning by $+$6.3~pp after ${\sim}$4B-token distillation at 1.6$\times$ faster training wall-clock speed.  Together, these findings point to a simple recipe: scoring with DO-ACP, retaining exactly the top-$k$ experts per layer, and distilling with forward KL.

\section{Limitations}
\label{sec:limitations}

When merging is used ($K > k$), our framework ties merge weights to selection scores, and decoupling these may improve such configurations.  Our extended training reaches ${\sim}$4B tokens, and scaling to tens of billions is needed to establish the quality ceiling.  At this budget, converted models are not yet competitive with off-the-shelf dense models, which are reference points rather than controlled baselines, since our goal is to derive a dense model from an already-trained MoE more efficiently than training from scratch.  The $+$6.3~pp comparison uses one-shot D2D, as iterative D2D was evaluated only at 0.3B tokens.  Fixed experts may lose domain-specific capabilities, and the method does not aim to preserve a post-trained teacher's native behavior.  The smaller benefit on GPT-OSS than on Qwen3 does not isolate expert pool size, since GPT-OSS is also a post-trained reasoning model evaluated in completion mode.

\bibliography{references}

\begin{thebibliography}{48}
\providecommand{\natexlab}[1]{#1}
\providecommand{\url}[1]{\texttt{#1}}
\expandafter\ifx\csname urlstyle\endcsname\relax
  \providecommand{\doi}[1]{doi: #1}\else
  \providecommand{\doi}{doi: \begingroup \urlstyle{rm}\Url}\fi

\bibitem[Agarwal et~al.(2025)Agarwal, Ahmad, Ai, Altman, Applebaum, Arbus,
  Arora, Bai, Baker, Bao, et~al.]{openai2025gptoss}
S.~Agarwal, L.~Ahmad, J.~Ai, S.~Altman, A.~Applebaum, E.~Arbus, R.~K. Arora,
  Y.~Bai, B.~Baker, H.~Bao, et~al.
\newblock gpt-oss-120b \& gpt-oss-20b model card.
\newblock \emph{arXiv preprint arXiv:2508.10925}, 2025.

\bibitem[Austin et~al.(2021)Austin, Odena, Nye, Bosma, Michalewski, Dohan,
  Jiang, Cai, Terry, Le, et~al.]{austin2021program}
J.~Austin, A.~Odena, M.~Nye, M.~Bosma, H.~Michalewski, D.~Dohan, E.~Jiang,
  C.~Cai, M.~Terry, Q.~Le, et~al.
\newblock Program synthesis with large language models.
\newblock \emph{arXiv preprint arXiv:2108.07732}, 2021.

\bibitem[Bai et~al.(2025)Bai, Li, Zhang, Hong, and Guo]{bai2025diep}
S.~Bai, H.~Li, J.~Zhang, Z.~Hong, and S.~Guo.
\newblock {DiEP}: Adaptive mixture-of-experts compression through
  differentiable expert pruning.
\newblock \emph{arXiv preprint arXiv:2509.16105}, 2025.

\bibitem[Bandarkar et~al.(2024)Bandarkar, Liang, Muller, Artetxe, Shukla, Husa,
  Goyal, Krishnan, Zettlemoyer, and Khabsa]{bandarkar2024belebele}
L.~Bandarkar, D.~Liang, B.~Muller, M.~Artetxe, S.~N. Shukla, D.~Husa, N.~Goyal,
  A.~Krishnan, L.~Zettlemoyer, and M.~Khabsa.
\newblock The {Belebele} benchmark: a parallel reading comprehension dataset in
  122 language variants.
\newblock In \emph{Proceedings of the 62nd Annual Meeting of the Association
  for Computational Linguistics (Volume 1: Long Papers)}, pages 749--775, 2024.

\bibitem[Chen et~al.(2025)Chen, Liu, Sun, Chao, Hsu, and Lee]{chen2025hcsmoe}
I.-C. Chen, H.-S. Liu, W.-F. Sun, C.-H. Chao, Y.-C. Hsu, and C.-Y. Lee.
\newblock Retraining-free merging of sparse {MoE} via hierarchical clustering.
\newblock In \emph{International Conference on Machine Learning}, 2025.

\bibitem[Chen et~al.(2021)Chen, Tworek, Jun, Yuan, {Ponde de Oliveira Pinto},
  Kaplan, Edwards, Burda, Joseph, Brockman, et~al.]{chen2021evaluating}
M.~Chen, J.~Tworek, H.~Jun, Q.~Yuan, H.~{Ponde de Oliveira Pinto}, J.~Kaplan,
  H.~Edwards, Y.~Burda, N.~Joseph, G.~Brockman, et~al.
\newblock Evaluating large language models trained on code.
\newblock \emph{arXiv preprint arXiv:2107.03374}, 2021.

\bibitem[Chen et~al.(2022)Chen, Huang, Xie, Jiao, Jiang, Zhou, Li, and
  Wei]{chen2022task}
T.~Chen, S.~Huang, Y.~Xie, B.~Jiao, D.~Jiang, H.~Zhou, J.~Li, and F.~Wei.
\newblock Task-specific expert pruning for sparse mixture-of-experts.
\newblock \emph{arXiv preprint arXiv:2206.00277}, 2022.

\bibitem[Cobbe et~al.(2021)Cobbe, Kosaraju, Bavarian, Chen, Jun, Kaiser,
  Plappert, Tworek, Hilton, Nakano, et~al.]{cobbe2021training}
K.~Cobbe, V.~Kosaraju, M.~Bavarian, M.~Chen, H.~Jun, L.~Kaiser, M.~Plappert,
  J.~Tworek, J.~Hilton, R.~Nakano, et~al.
\newblock Training verifiers to solve math word problems.
\newblock \emph{arXiv preprint arXiv:2110.14168}, 2021.

\bibitem[Dai et~al.(2024)Dai, Deng, Zhao, Xu, Gao, Chen, Li, Zeng, Yu, Wu, Xie,
  Li, Huang, Luo, Ruan, Sui, and Liang]{dai2024deepseekmoe}
D.~Dai, C.~Deng, C.~Zhao, R.~Xu, H.~Gao, D.~Chen, J.~Li, W.~Zeng, X.~Yu, Y.~Wu,
  Z.~Xie, Y.~Li, P.~Huang, F.~Luo, C.~Ruan, Z.~Sui, and W.~Liang.
\newblock {DeepSeekMoE}: Towards ultimate expert specialization in
  mixture-of-experts language models.
\newblock \emph{arXiv preprint arXiv:2401.06066}, 2024.

\bibitem[{DeepSeek-AI}(2024)]{deepseekai2024deepseekv2}
{DeepSeek-AI}.
\newblock {DeepSeek-V2}: A strong, economical, and efficient mixture-of-experts
  language model.
\newblock \emph{arXiv preprint arXiv:2405.04434}, 2024.

\bibitem[{DeepSeek-AI}(2026)]{deepseekai2026deepseekv4}
{DeepSeek-AI}.
\newblock Deepseek-v4: Towards highly efficient million-token context
  intelligence.
\newblock Technical report, 2026.
\newblock URL
  \url{https://huggingface.co/deepseek-ai/DeepSeek-V4-Pro/blob/main/DeepSeek_V4.pdf}.

\bibitem[Fedus et~al.(2022)Fedus, Zoph, and Shazeer]{fedus2022switch}
W.~Fedus, B.~Zoph, and N.~Shazeer.
\newblock Switch transformers: Scaling to trillion parameter models with simple
  and efficient sparsity.
\newblock \emph{Journal of Machine Learning Research}, 23:\penalty0 1--39,
  2022.

\bibitem[{Google DeepMind}(2026)]{gemma4}
{Google DeepMind}.
\newblock Gemma 4.
\newblock \url{https://deepmind.google/models/gemma/gemma-4/}, 2026.

\bibitem[Goyal et~al.(2022)Goyal, Gao, Chaudhary, Chen, Wenzek, Ju, Krishnan,
  Ranzato, Guzm{\'a}n, and Fan]{goyal2022flores}
N.~Goyal, C.~Gao, V.~Chaudhary, P.-J. Chen, G.~Wenzek, D.~Ju, S.~Krishnan,
  M.~Ranzato, F.~Guzm{\'a}n, and A.~Fan.
\newblock The {FLORES-101} evaluation benchmark for low-resource and
  multilingual machine translation.
\newblock \emph{Transactions of the Association for Computational Linguistics},
  10:\penalty0 522--538, 2022.

\bibitem[Hendrycks et~al.(2021)Hendrycks, Burns, Kadavath, Arora, Basart, Tang,
  Song, and Steinhardt]{hendrycks2021measuring}
D.~Hendrycks, C.~Burns, S.~Kadavath, A.~Arora, S.~Basart, E.~Tang, D.~Song, and
  J.~Steinhardt.
\newblock Measuring mathematical problem solving with the {MATH} dataset.
\newblock \emph{arXiv preprint arXiv:2103.03874}, 2021.

\bibitem[Hinton et~al.(2015)Hinton, Vinyals, and Dean]{hinton2015distilling}
G.~Hinton, O.~Vinyals, and J.~Dean.
\newblock Distilling the knowledge in a neural network.
\newblock \emph{arXiv preprint arXiv:1503.02531}, 2015.

\bibitem[Hsieh et~al.(2024)Hsieh, Sun, Kriman, Acharya, Rekesh, Jia, Zhang, and
  Ginsburg]{hsieh2024ruler}
C.-P. Hsieh, S.~Sun, S.~Kriman, S.~Acharya, D.~Rekesh, F.~Jia, Y.~Zhang, and
  B.~Ginsburg.
\newblock {RULER}: What's the real context size of your long-context language
  models?
\newblock \emph{arXiv preprint arXiv:2404.06654}, 2024.

\bibitem[Jha et~al.(2026)Jha, Hashemzadeh, Saheb~Pasand, Parviz, Lee, and
  Knyazev]{jha2026ream}
S.~Jha, M.~Hashemzadeh, A.~Saheb~Pasand, A.~Parviz, M.-J. Lee, and B.~Knyazev.
\newblock {REAM}: Merging improves pruning of experts in {LLMs}.
\newblock \emph{arXiv preprint arXiv:2604.04356}, 2026.

\bibitem[Kim et~al.(2024)Kim, Kim, Kim, Castells, Choi, Shin, and
  Song]{kim2024shortened}
B.-K. Kim, G.~Kim, T.-H. Kim, T.~Castells, S.~Choi, J.~Shin, and H.-K. Song.
\newblock Shortened llama: Depth pruning for large language models with
  comparison of retraining methods.
\newblock \emph{arXiv preprint arXiv:2402.02834}, 2024.

\bibitem[Kim et~al.(2025)Kim, Chu, and Yang]{kim2025expert}
G.~Kim, G.~Chu, and E.~Yang.
\newblock Every expert matters: Towards effective knowledge distillation for
  mixture-of-experts language models.
\newblock \emph{arXiv preprint arXiv:2502.12947}, 2025.

\bibitem[Kwon et~al.(2023)Kwon, Li, Zhuang, Sheng, Zheng, Yu, Gonzalez, Zhang,
  and Stoica]{kwon2023efficient}
W.~Kwon, Z.~Li, S.~Zhuang, Y.~Sheng, L.~Zheng, C.~H. Yu, J.~Gonzalez, H.~Zhang,
  and I.~Stoica.
\newblock Efficient memory management for large language model serving with
  {PagedAttention}.
\newblock In \emph{Proceedings of the 29th symposium on operating systems
  principles}, pages 611--626, 2023.

\bibitem[Lasby et~al.(2025)Lasby, Lazarevich, Sinnadurai, Lie, Ioannou, and
  Thangarasa]{lasby2025reap}
M.~Lasby, I.~Lazarevich, N.~Sinnadurai, S.~Lie, Y.~Ioannou, and V.~Thangarasa.
\newblock Reap the experts: Why pruning prevails for one-shot moe compression.
\newblock \emph{arXiv preprint arXiv:2510.13999}, 2025.

\bibitem[Li et~al.(2025{\natexlab{a}})Li, Qiyuan, Wang, Li, Gu, Han, and
  Guo]{li2025submoe}
L.~Li, Z.~Qiyuan, J.~Wang, W.~Li, H.~Gu, S.~Han, and Y.~Guo.
\newblock {Sub-MoE}: Efficient mixture-of-expert {LLMs} compression via
  subspace expert merging.
\newblock \emph{arXiv preprint arXiv:2506.23266}, 2025{\natexlab{a}}.

\bibitem[Li et~al.(2024)Li, Zhang, Yadav, Sung, Cheng, Bansal, and
  Chen]{li2024mcsmoe}
P.~Li, Z.~Zhang, P.~Yadav, Y.-L. Sung, Y.~Cheng, M.~Bansal, and T.~Chen.
\newblock Merge, then compress: Demystify efficient {SMoE} with hints from its
  routing policy.
\newblock In \emph{International Conference on Learning Representations}, 2024.

\bibitem[Li et~al.(2025{\natexlab{b}})Li, Liang, Zhang, Hong, Kim, Chen, and
  Zhao]{li2025slimmoe}
Z.~Li, C.~Liang, Z.~Zhang, I.~Hong, Y.~J. Kim, W.~Chen, and T.~Zhao.
\newblock {SlimMoE}: Structured compression of large {MoE} models via expert
  slimming and distillation.
\newblock \emph{arXiv preprint arXiv:2506.18349}, 2025{\natexlab{b}}.

\bibitem[Liu et~al.(2026)Liu, Khandelwal, Subramanian, Jouault,
  et~al.]{ministral3}
A.~H. Liu, K.~Khandelwal, S.~Subramanian, V.~Jouault, et~al.
\newblock Ministral 3.
\newblock \emph{arXiv preprint arXiv:2601.08584}, 2026.

\bibitem[Merity et~al.(2017)Merity, Xiong, Bradbury, and
  Socher]{merity2017wikitext}
S.~Merity, C.~Xiong, J.~Bradbury, and R.~Socher.
\newblock Pointer sentinel mixture models.
\newblock \emph{arXiv preprint arXiv:1609.07843}, 2017.

\bibitem[{Meta}(2025)]{llama4}
{Meta}.
\newblock The llama 4 herd: The beginning of a new era of natively multimodal
  ai innovation.
\newblock \url{https://ai.meta.com/blog/llama-4-multimodal-intelligence/},
  2025.

\bibitem[Miao et~al.(2025)Miao, Yao, Wang, Wang, Yi, Liu, Zhao, and
  Yang]{miao2025mergemoe}
R.~Miao, Y.~Yao, Z.~Wang, Z.~Wang, B.~Yi, L.~Liu, Y.~Zhao, and T.~Yang.
\newblock {MergeMoE}: Efficient compression of {MoE} models via expert output
  merging.
\newblock \emph{arXiv preprint arXiv:2510.14436}, 2025.

\bibitem[Mihaylov et~al.(2018)Mihaylov, Clark, Khot, and
  Sabharwal]{mihaylov2018can}
T.~Mihaylov, P.~Clark, T.~Khot, and A.~Sabharwal.
\newblock Can a suit of armor conduct electricity? a new dataset for open book
  question answering.
\newblock In \emph{Proceedings of the 2018 conference on empirical methods in
  natural language processing}, pages 2381--2391, 2018.

\bibitem[Muralidharan et~al.(2024)Muralidharan, Sreenivas, Joshi, Chochowski,
  Patwary, Shoeybi, Catanzaro, Kautz, and Molchanov]{muralidharan2024compact}
S.~Muralidharan, S.~T. Sreenivas, R.~Joshi, M.~Chochowski, M.~Patwary,
  M.~Shoeybi, B.~Catanzaro, J.~Kautz, and P.~Molchanov.
\newblock Compact language models via pruning and knowledge distillation.
\newblock \emph{arXiv preprint arXiv:2407.14679}, 2024.

\bibitem[Nemhauser et~al.(1978)Nemhauser, Wolsey, and
  Fisher]{nemhauser1978analysis}
G.~L. Nemhauser, L.~A. Wolsey, and M.~L. Fisher.
\newblock An analysis of approximations for maximizing submodular set
  functions---{I}.
\newblock \emph{Mathematical Programming}, 14\penalty0 (1):\penalty0 265--294,
  1978.

\bibitem[Nguyen et~al.(2025)Nguyen, Nguyen, Nguyen, Nguyen, Jiang, Fetaya,
  Tran, Chechik, and Nguyen]{nguyen2025namex}
D.~V. Nguyen, A.~T. Nguyen, M.~H. Nguyen, L.~Q. Nguyen, S.~Jiang, E.~Fetaya,
  L.~D. Tran, G.~Chechik, and T.~M. Nguyen.
\newblock Expert merging in sparse mixture of experts with nash bargaining.
\newblock \emph{arXiv preprint arXiv:2510.16138}, 2025.

\bibitem[Penedo et~al.(2024)Penedo, Kydl{\'\i}{\v{c}}ek, Lozhkov, Mitchell,
  Raffel, Von~Werra, Wolf, et~al.]{fineweb}
G.~Penedo, H.~Kydl{\'\i}{\v{c}}ek, A.~Lozhkov, M.~Mitchell, C.~Raffel,
  L.~Von~Werra, T.~Wolf, et~al.
\newblock The fineweb datasets: Decanting the web for the finest text data at
  scale.
\newblock \emph{Advances in Neural Information Processing Systems},
  37:\penalty0 30811--30849, 2024.

\bibitem[Pukelsheim(2006)]{pukelsheim2006optimal}
F.~Pukelsheim.
\newblock \emph{Optimal Design of Experiments}.
\newblock SIAM, 2006.

\bibitem[Roy and Vetterli(2007)]{roy2007effective}
O.~Roy and M.~Vetterli.
\newblock The effective rank: A measure of effective dimensionality.
\newblock \emph{15th European Signal Processing Conference}, pages 606--610,
  2007.

\bibitem[Shazeer(2020)]{shazeer2020glu}
N.~Shazeer.
\newblock {GLU} variants improve transformer.
\newblock \emph{arXiv preprint arXiv:2002.05202}, 2020.

\bibitem[Shazeer et~al.(2017)Shazeer, Mirhoseini, Maziarz, Davis, Le, Hinton,
  and Dean]{shazeer2017outrageously}
N.~Shazeer, A.~Mirhoseini, K.~Maziarz, A.~Davis, Q.~Le, G.~Hinton, and J.~Dean.
\newblock Outrageously large neural networks: The sparsely-gated
  mixture-of-experts layer.
\newblock \emph{arXiv preprint arXiv:1701.06538}, 2017.

\bibitem[Sreenivas et~al.(2024)Sreenivas, Muralidharan, Joshi, Chochowski,
  Patwary, Shoeybi, Catanzaro, Kautz, and Molchanov]{sreenivas2024llmpruning}
S.~T. Sreenivas, S.~Muralidharan, R.~Joshi, M.~Chochowski, M.~Patwary,
  M.~Shoeybi, B.~Catanzaro, J.~Kautz, and P.~Molchanov.
\newblock {LLM} pruning and distillation in practice: The minitron approach.
\newblock \emph{arXiv preprint arXiv:2408.11796}, 2024.

\bibitem[Sun et~al.(2023)Sun, Liu, Bair, and Kolter]{sun2023wanda}
M.~Sun, Z.~Liu, A.~Bair, and J.~Z. Kolter.
\newblock A simple and effective pruning approach for large language models.
\newblock \emph{arXiv preprint arXiv:2306.11695}, 2023.

\bibitem[Wang et~al.(2026)Wang, Bhagia, and
  Min]{wang2026emopretrainingmixtureexperts}
R.~Wang, A.~Bhagia, and S.~Min.
\newblock {EMO}: Pretraining mixture of experts for emergent modularity.
\newblock \emph{arXiv preprint arXiv:2605.06663}, 2026.

\bibitem[Welbl et~al.(2017)Welbl, Liu, and Gardner]{welbl2017crowdsourcing}
J.~Welbl, N.~F. Liu, and M.~Gardner.
\newblock Crowdsourcing multiple choice science questions.
\newblock In \emph{Proceedings of the 3rd Workshop on Noisy User-generated
  Text}, pages 94--106, 2017.

\bibitem[Xia et~al.(2024)Xia, Gao, Zeng, and Chen]{xia2024sheared}
M.~Xia, T.~Gao, Z.~Zeng, and D.~Chen.
\newblock Sheared {LLaMA}: Accelerating language model pre-training via
  structured pruning.
\newblock \emph{arXiv preprint arXiv:2310.06694}, 2024.

\bibitem[Xie et~al.(2024)Xie, Zhang, Zhou, Xie, Song, Liu, Wang, Lin, and
  Xu]{xie2024moepruner}
Y.~Xie, Z.~Zhang, D.~Zhou, C.~Xie, Z.~Song, X.~Liu, Y.~Wang, X.~Lin, and A.~Xu.
\newblock {MoE-Pruner}: Pruning mixture-of-experts large language model using
  the hints from its router.
\newblock \emph{arXiv preprint arXiv:2410.12013}, 2024.

\bibitem[Yang et~al.(2025)Yang, Li, Yang, Zhang, Hui, Zheng, Yu, Gao, Huang,
  Lv, et~al.]{qwen3}
A.~Yang, A.~Li, B.~Yang, B.~Zhang, B.~Hui, B.~Zheng, B.~Yu, C.~Gao, C.~Huang,
  C.~Lv, et~al.
\newblock Qwen3 technical report.
\newblock \emph{arXiv preprint arXiv:2505.09388}, 2025.

\bibitem[Yang et~al.(2024)Yang, Sui, Xiao, Huang, Gong, Duan, Jia, Yin, Cheng,
  and Yuan]{yang2024moei2}
C.~Yang, Y.~Sui, J.~Xiao, L.~Huang, Y.~Gong, Y.~Duan, W.~Jia, M.~Yin, Y.~Cheng,
  and B.~Yuan.
\newblock {MoE-I$^2$}: Compressing mixture of experts models through
  inter-expert pruning and intra-expert low-rank decomposition.
\newblock In \emph{Findings of EMNLP}, 2024.

\bibitem[Zeng et~al.(2026)Zeng, Lv, Hou, Du, Zheng, Chen, Yin, Ge, Huang, Xie,
  et~al.]{zeng2026glm}
A.~Zeng, X.~Lv, Z.~Hou, Z.~Du, Q.~Zheng, B.~Chen, D.~Yin, C.~Ge, C.~Huang,
  C.~Xie, et~al.
\newblock Glm-5: from vibe coding to agentic engineering.
\newblock \emph{arXiv preprint arXiv:2602.15763}, 2026.

\bibitem[Zhao et~al.(2025)Zhao, Wang, and Zhang]{zhao2025puzzlemoe}
Y.~Zhao, Z.~Wang, and M.~Zhang.
\newblock {PuzzleMoE}: Efficient compression of large mixture-of-experts models
  via sparse expert merging and bit-packed inference.
\newblock \emph{arXiv preprint arXiv:2511.04805}, 2025.

\end{thebibliography}

\clearpage
\appendix

\makeatletter
\@ifundefined{c@theorem}{%
  \newtheorem{theorem}{Theorem}[section]
  \newtheorem{lemma}[theorem]{Lemma}
  \newtheorem{proposition}[theorem]{Proposition}
  \newtheorem{corollary}[theorem]{Corollary}
  \newtheorem{definition}[theorem]{Definition}
  \newtheorem{remark}[theorem]{Remark}
  \newtheorem{assumption}[theorem]{Assumption}
}{}
\makeatother

\section{Block concatenation preserves representative activations}
\label{app:concat_equivalence}

We show that, once a layer has been converted into $k$ representative FFNs, block concatenation produces the same intermediate activations as those representatives and differs only in the final static aggregation weights.
When $K{=}k$, the representatives are copied experts, so this is an exact statement about the selected original experts.
When $K{>}k$, the representatives are first formed by parameter-space averaging; the result below applies to those constructed representatives, not to every original expert before averaging.

Consider a single representative FFN with SwiGLU activation.  Representative $g$ computes:
\begin{equation}
    f_g(\mathbf{h}) = \mathbf{W}_{\text{down}}^{(g)} \bigl(\sigma(\mathbf{W}_{\text{gate}}^{(g)} \mathbf{h}) \odot \mathbf{W}_{\text{up}}^{(g)} \mathbf{h}\bigr),
\end{equation}
where $\sigma$ is the SiLU activation and $\odot$ is element-wise multiplication.  The original MoE output uses token-dependent routing weights:
\begin{equation}
    f_{\text{MoE}}(\mathbf{h}) = \sum_{e \in \mathcal{S}(\mathbf{h})} w_e(\mathbf{h}) \cdot f_e(\mathbf{h}),
\end{equation}
where $\mathcal{S}(\mathbf{h})$ is the top-$k$ routed expert set and $w_e(\mathbf{h})$ are the router weights.

Now consider the concatenated dense FFN with $\mathbf{W}_{\text{gate}} = [\mathbf{W}_{\text{gate}}^{(1)}; \ldots; \mathbf{W}_{\text{gate}}^{(k)}]$ and similarly for $\mathbf{W}_{\text{up}}$.  Since row-concatenation distributes over matrix--vector multiplication:
\begin{equation}
    \mathbf{W}_{\text{gate}} \mathbf{h} = \begin{bmatrix} \mathbf{W}_{\text{gate}}^{(1)} \mathbf{h} \\ \vdots \\ \mathbf{W}_{\text{gate}}^{(k)} \mathbf{h} \end{bmatrix}, \qquad
    \mathbf{W}_{\text{up}} \mathbf{h} = \begin{bmatrix} \mathbf{W}_{\text{up}}^{(1)} \mathbf{h} \\ \vdots \\ \mathbf{W}_{\text{up}}^{(k)} \mathbf{h} \end{bmatrix}.
\end{equation}
The SwiGLU activation applies element-wise, so the intermediate activation decomposes into $k$ independent blocks, each identical to the corresponding representative's intermediate activation.  The column-concatenated down-projection then sums over these blocks:
\begin{equation}
    f_{\text{dense}}(\mathbf{h}) = \mathbf{W}_{\text{down}} \bigl(\sigma(\mathbf{W}_{\text{gate}} \mathbf{h}) \odot \mathbf{W}_{\text{up}} \mathbf{h}\bigr) = \sum_{g=1}^{k} \alpha_g \cdot f_g(\mathbf{h}).
\end{equation}
The representative intermediate activations (gate projections, up projections, and SwiGLU outputs) are \emph{exactly} preserved by block concatenation.  Within this block-concatenation step, the only approximation is that the dense model uses static $\alpha_g$ from down-projection scaling (Section~\ref{sec:scaling}) instead of token-dependent router weights; the full MoE-to-dense conversion also fixes the selected experts and, when $K{>}k$, replaces groups of original experts by parameter-averaged representatives.  Knowledge distillation compensates for these approximations.

\clearpage
\section{Algorithms}
\label{app:algorithm}

This section presents the two core procedures of our pipeline: the per-layer MoE-to-dense conversion (Algorithm~\ref{alg:merging}) and the greedy D-Optimal subset selection used by the DO scoring family (Algorithm~\ref{alg:do_greedy}).

\vspace{-0.5em}
\paragraph{MoE-to-dense conversion.}
Algorithm~\ref{alg:merging} gives the per-layer conversion, where $E$ is the number of experts, $k$ the MoE top-$k$ count (also the number of groups), $K \geq k$ the experts selected for conversion, $\mathcal{G}_g \subseteq [E]$ the experts assigned to group $g \in [k]$, and $\alpha_g$ the down-projection scaling factor (uniform $1/k$ or score-proportional, Section~\ref{sec:scaling}).

\begin{algorithm}[h]
\caption{MoE-to-dense conversion for layer $\ell$}
\label{alg:merging}
\begin{algorithmic}[1]
\REQUIRE Experts $\{\mathbf{W}^{(e)}\}_{e=1}^{E}$, scores $\{s^{(e)}\}_{e=1}^{E}$, top-$K$, number of groups $k$; scoring, grouping, and scaling method choices.
\ENSURE Dense projections $\mathbf{W}_{\text{gate}}, \mathbf{W}_{\text{up}} \in \mathbb{R}^{d_{\text{dense}} \times d}$, $\mathbf{W}_{\text{down}} \in \mathbb{R}^{d \times d_{\text{dense}}}$
\STATE \textbf{Score:} Compute scores $s^{(e)}$ for all $E$ experts; select top-$K$.
\STATE \textbf{Group:} Partition the $K$ selected experts into $k$ groups using the chosen grouping method.
\FOR{each group $g$}
    \STATE Compute within-group weights: $w^{(e)} = s^{(e)} \big/ \sum_{e' \in \mathcal{G}_g} s^{(e')}$ for each $e \in \mathcal{G}_g$.
    \FOR{$\text{proj} \in \{\text{gate}, \text{up}, \text{down}\}$}
        \STATE \textbf{Merge:} $\mathbf{W}_{\text{proj}}^{(g)} = \sum_{e \in \mathcal{G}_g} w^{(e)} \, \mathbf{W}_{\text{proj}}^{(e)}$ \hfill $\triangleright$ Weighted average (identity for $|\mathcal{G}_g|=1$)
    \ENDFOR
    \STATE \textbf{Scale:} $\tilde{\mathbf{W}}_{\text{down}}^{(g)} = \alpha_g \cdot \mathbf{W}_{\text{down}}^{(g)}$ \hfill $\triangleright$ Uniform or proportional
\ENDFOR
\STATE \textbf{Concatenate:} $\mathbf{W}_{\text{gate}} = [\mathbf{W}_{\text{gate}}^{(1)}; \ldots; \mathbf{W}_{\text{gate}}^{(k)}]$, \; $\mathbf{W}_{\text{up}} = [\mathbf{W}_{\text{up}}^{(1)}; \ldots; \mathbf{W}_{\text{up}}^{(k)}]$ \hfill $\triangleright$ Row-concat
\STATE \hphantom{\textbf{Concatenate:}} $\mathbf{W}_{\text{down}} = [\tilde{\mathbf{W}}_{\text{down}}^{(1)}, \ldots, \tilde{\mathbf{W}}_{\text{down}}^{(k)}]$ \hfill $\triangleright$ Column-concat
\end{algorithmic}
\end{algorithm}

The algorithm applies independently to each of the $L$ MoE layers.  All non-MoE parameters (attention, embeddings, layer norms) are copied unchanged from teacher to student.

\vspace{-0.5em}
\paragraph{Greedy D-Optimal selection.}
Algorithm~\ref{alg:do_greedy} gives the greedy size-$K$ subset selection used by the DO scoring family, instantiating the (1${-}$1$/e$)-approximation of Proposition~\ref{prop:greedy}.  The importance-weighted Gram matrix $\boldsymbol{\mathcal{K}} \in \mathbb{R}^{E\times E}$ has entries $\boldsymbol{\mathcal{K}}_{ij} = \sqrt{I_i I_j}\,\mathbf{G}_{ij}$, with base importance $I_e = s^{(e)} \geq 0$ (CP or ACP) and Gram entries $\mathbf{G}_{ij} = \mathbb{E}_t[\langle f_i(t), f_j(t) \rangle]$ over calibration tokens.  For $S \subseteq [E]$, $\boldsymbol{\mathcal{K}}_S$ is the principal submatrix indexed by $S$, and $\boldsymbol{\mathcal{K}}_{eS}$, $\boldsymbol{\mathcal{K}}_{Se}$ the corresponding row and column slices for $e \notin S$.  The objective is $F(S) := \log\det(\boldsymbol{\mathcal{K}}_S + \lambda_{\mathrm{reg}}\mathbf{I})$ with $\lambda_{\mathrm{reg}} = \frac{1}{KE}\sum_{e=1}^{E}\boldsymbol{\mathcal{K}}_{ee}$.

\begin{algorithm}[h]
\caption{Greedy D-Optimal expert selection (per layer)}
\label{alg:do_greedy}
\begin{algorithmic}[1]
\REQUIRE Importance-weighted Gram $\boldsymbol{\mathcal{K}}$ (positive semidefinite), target subset size $K$, regularizer $\lambda_{\mathrm{reg}}$.
\ENSURE Ordered subset $S \subseteq [E]$ with $|S| = K$
\STATE $S \gets \emptyset$
\FOR{$\text{step} = 1, \dots, K$}
    \IF{$S = \emptyset$}
        \STATE $\mathrm{gain}(e) \gets \log(\boldsymbol{\mathcal{K}}_{ee} + \lambda_{\mathrm{reg}})$ for each $e \in [E]$
    \ELSE
        \STATE $A_S \gets \boldsymbol{\mathcal{K}}_S + \lambda_{\mathrm{reg}}\mathbf{I}$;\; precompute $A_S^{-1}$ \hfill $\triangleright$ once per step
        \FOR{each $e \notin S$}
            \STATE $\sigma_e \gets \boldsymbol{\mathcal{K}}_{ee} + \lambda_{\mathrm{reg}} - \boldsymbol{\mathcal{K}}_{eS}\,A_S^{-1}\,\boldsymbol{\mathcal{K}}_{Se}$ \hfill $\triangleright$ Schur complement
            \STATE $\mathrm{gain}(e) \gets \log\sigma_e$
        \ENDFOR
    \ENDIF
    \STATE $e^\star \gets \arg\max_{e \notin S} \mathrm{gain}(e)$
    \STATE $S \gets S \cup \{e^\star\}$
\ENDFOR
\STATE \textbf{return} $S$
\end{algorithmic}
\end{algorithm}

$A_S^{-1}$ is recomputed each step at cost $O(K^3)$, and the Schur complement is evaluated for each of the $E$ candidates at $O(K^2)$ per evaluation, yielding overall $O(K^3 E)$ time.

\section{Scoring baseline definitions}
\label{app:scoring_defs}

For completeness, we provide the formal definitions of the three routing-based baseline scoring methods (Section~\ref{sec:scoring}).  Let $p_\ell^{(e)}(t)$ be the softmax routing probability and $\mathcal{S}_\ell(t)$ the set of top-$k$ selected experts for token $t$ in layer $\ell$.

\paragraph{Selection frequency (SF).}
$s_\ell^{(e)} = \frac{1}{N}\sum_{t=1}^{N} \mathbb{1}[e \in \mathcal{S}_\ell(t)].$

\paragraph{Pre-selection probability (PP).}
$s_\ell^{(e)} = \frac{1}{N}\sum_{t=1}^{N} p_\ell^{(e)}(t).$

\paragraph{Post-selection probability (PS).}
$s_\ell^{(e)} = \frac{1}{N}\sum_{t : e \in \mathcal{S}_\ell(t)} p_\ell^{(e)}(t).$

PS decomposes as $\text{SF} \times \text{CP}$: $s_\ell^{(e)} = f_\ell^{(e)} \cdot \mathbb{E}[p_\ell^{(e)}(t) \mid e \in \mathcal{S}_\ell(t)]$.

\section{Excluded scoring and grouping methods}
\label{app:excluded_methods}

\paragraph{Scoring methods excluded.}
We do not evaluate Fisher-weighted merging, which MC-SMoE~\citep{li2024mcsmoe} found to underperform frequency weighting (their Table~8).  We omit loss-degradation scoring from MoE-I$^2$~\citep{yang2024moei2}, which requires $E \times L = 6{,}144$ forward passes and is impractical at our scale.  We do not evaluate Nash Bargaining merging coefficients (NAMEx~\citep{nguyen2025namex}), since our framework ties merge weights to selection scores (Section~\ref{sec:limitations}).  Per-weight saliency metrics (Wanda~\citep{sun2023wanda}) operate at intra-expert granularity, orthogonal to our per-expert framework.

\paragraph{Grouping methods excluded.}
We do not evaluate K-means clustering, which HC-SMoE~\citep{chen2025hcsmoe} found to underperform hierarchical clustering for expert grouping.  We omit the element-wise similarity and saliency masks of PuzzleMoE~\citep{zhao2025puzzlemoe}, which merge experts entry by entry and rely on a bit-packed format with a custom decode kernel for efficient inference, rather than selecting whole experts.

\section{Proofs for D-Optimal expert selection}
\label{app:theory_proofs}

This section provides proofs for the theoretical results stated in Section~\ref{sec:theory}. We use the notation established there: $f_e(t) \in \mathbb{R}^d$ is the output of expert $e$ on token $t$, $I_e = s^{(e)} \geq 0$ is a base importance score (CP or ACP), and $\mathcal{K}_{ij} = \sqrt{I_i I_j} \cdot \mathbb{E}_{t}[\langle f_i(t), f_j(t) \rangle]$ is the importance-weighted kernel. For any subset $S \subseteq [E]$, we define:
\begin{align}
    F(S) &:= \log\det \left(\boldsymbol{\mathcal{K}}_S + \lambda_{\mathrm{reg}} \mathbf{I}\right), &
    \widetilde{F}(S) &:= \log\det\!\left(\mathbf{I} + \lambda_{\mathrm{reg}}^{-1} \boldsymbol{\mathcal{K}}_S\right).
\end{align}
For fixed cardinality $|S| = K$, $F$ and $\widetilde{F}$ differ by the additive constant $K\log\lambda_{\mathrm{reg}}$.  We use $\widetilde{F}$ for the submodularity argument (it is monotone for any $\lambda_{\mathrm{reg}}>0$, while $F$ is not) and $F$ in the incoherence and stability proofs (it factors cleanly with the diagonal proxy $G$).

\subsection{Redundancy counterexample (Proof of Theorem~\ref{thm:redundancy})}
\label{app:proof_redundancy}

\begin{proof}
Fix $K \geq 2$ and set $E = 2K - 1$.  Let the calibration space be
\begin{align}
    \mathcal{X} := \{a_1, \ldots, a_K, b_2, \ldots, b_K\},
\end{align}
and let $\mathcal{D}$ be the uniform distribution on $\mathcal{X}$.  We construct a top-1 MoE layer with scalar expert outputs.  Define
\begin{align}
    f_e(x) &= \mathbf{1}\{x \in \{a_1, \ldots, a_K\}\}, \quad e = 1, \ldots, K, \\
    f_{K+j-1}(x) &= \mathbf{1}\{x = b_j\}, \quad j = 2, \ldots, K.
\end{align}
Thus the first $K$ experts are identical, while the remaining $K-1$ experts are pairwise orthogonal specialists.

Define a deterministic router $r : \mathcal{X} \to [E]$ by
\begin{align}
    r(a_j) &= j, \quad j = 1, \ldots, K, \\
    r(b_j) &= K+j-1, \quad j = 2, \ldots, K,
\end{align}
and let the routing probabilities be $p^{(e)}(x) := \mathbf{1}\{e = r(x)\}$.  The resulting MoE teacher output is
\begin{align}
    F_{\mathrm{MoE}}(x)
    := \sum_{e=1}^{E} p^{(e)}(x) f_e(x)
    = 1
    \qquad \text{for every } x \in \mathcal{X}.
\end{align}

We take ACP as the base importance score.  Because the router is deterministic, every selected expert has conditional probability $1$, so the ACP score of expert $e$ reduces to $\sqrt{\mathbb{E}[f_e(x)^2]}$.  Therefore
\begin{align}
    I_1 = \cdots = I_K
    =: I_A
    = \sqrt{\frac{K}{2K-1}},
    \qquad
    I_{K+1} = \cdots = I_{2K-1}
    =: I_B
    = \sqrt{\frac{1}{2K-1}},
\end{align}
and hence $I_A > I_B > 0$.  Independent top-$K$ ranking by the base importances therefore selects
\begin{align}
    S_{\mathrm{ind}} = \{1, \ldots, K\}.
\end{align}

For any subset $S \subseteq [E]$, define its best linear reconstruction error against the teacher output by
\begin{align}
    \mathcal{L}(S)
    := \inf_{a \in \mathbb{R}^{|S|}}
    \mathbb{E}_{x \sim \mathcal{D}}
    \left[
        \left(
            \sum_{e \in S} a_e f_e(x) - F_{\mathrm{MoE}}(x)
        \right)^2
    \right].
\end{align}

\paragraph{Independent ranking incurs constant error.}
For $S_{\mathrm{ind}}$, every selected feature equals $\mathbf{1}\{x \in \{a_1, \ldots, a_K\}\}$.  Hence every linear combination over $S_{\mathrm{ind}}$ has the form
\begin{align}
    x \mapsto c \, \mathbf{1}\{x \in \{a_1, \ldots, a_K\}\}
\end{align}
for some scalar $c$.  The choice $c = 1$ is optimal, because it matches $F_{\mathrm{MoE}}(x) = 1$ on the $K$ points $a_1, \ldots, a_K$.  On the remaining $K-1$ points $b_2, \ldots, b_K$, the reconstruction is zero, so the squared error is one.  Therefore
\begin{align}
    \mathcal{L}(S_{\mathrm{ind}})
    = \frac{K-1}{2K-1}.
\end{align}

\paragraph{Log-det selects a zero-error subset.}
Now define
\begin{align}
    S_{\mathrm{good}} := \{1\} \cup \{K+1, \ldots, 2K-1\}.
\end{align}
For this subset, choosing coefficient vector $a \equiv 1$ gives
\begin{align}
    \sum_{e \in S_{\mathrm{good}}} a_e f_e(x)
    = \mathbf{1}\{x \in \{a_1, \ldots, a_K\}\}
      + \sum_{j=2}^{K} \mathbf{1}\{x = b_j\}
    = 1
    = F_{\mathrm{MoE}}(x)
\end{align}
for every $x \in \mathcal{X}$, so $\mathcal{L}(S_{\mathrm{good}}) = 0$.

To identify the log-determinant maximizer, let
\begin{align}
    \alpha := \mathcal{K}_{11}
    = I_A \, \mathbb{E}[f_1(x)^2]
    = \left(\frac{K}{2K-1}\right)^{3/2},
    \qquad
    \beta := \mathcal{K}_{K+1,K+1}
    = I_B \, \mathbb{E}[f_{K+1}(x)^2]
    = \left(\frac{1}{2K-1}\right)^{3/2}.
\end{align}
Because the first block of experts is identical and the specialist block is orthogonal to it and to itself off-diagonal, every size-$K$ subset $S$ is characterized by
\begin{align}
    t := |S \cap \{1, \ldots, K\}| \in \{1, \ldots, K\},
\end{align}
and the nonzero eigenvalues of $\boldsymbol{\mathcal{K}}_S$ are
\begin{align}
    t\alpha
    \qquad \text{and} \qquad
    \underbrace{\beta, \ldots, \beta}_{K-t \text{ times}}.
\end{align}
The remaining $t-1$ eigenvalues are zero.  Setting $\lambda_{\mathrm{reg}} := \beta$, we obtain
\begin{align}
    \det(\boldsymbol{\mathcal{K}}_S + \lambda_{\mathrm{reg}} \mathbf{I})
    = (\beta + t\alpha)\, \beta^{\,t-1} (2\beta)^{K-t}
    = \beta^{K-1} 2^{K-t} (\beta + t\alpha).
\end{align}
Let $D(t) := \beta^{K-1} 2^{K-t} (\beta + t\alpha)$.  For every $t \geq 1$,
\begin{align}
    \frac{D(t+1)}{D(t)}
    = \frac{\beta + (t+1)\alpha}{2(\beta + t\alpha)}
    < 1
\end{align}
because
\begin{align}
    \beta + (t+1)\alpha < 2\beta + 2t\alpha
    \iff (1-t)\alpha < \beta,
\end{align}
and the right-hand side is true for all $t \geq 1$ since $\beta > 0$.  Thus $D(t)$ is strictly decreasing in $t$, so every maximizer of the log-determinant objective has $t=1$.  Every such maximizer has the form
\begin{align}
    \{i\} \cup \{K+1, \ldots, 2K-1\}
    \qquad \text{for some } i \in \{1, \ldots, K\}.
\end{align}
Because all first-block experts are identical, each of these maximizing subsets has the same reconstruction error as $S_{\mathrm{good}}$, namely zero.  Consequently, any size-$K$ subset maximizing the log-determinant objective achieves zero reconstruction error.
\end{proof}

\subsection{Submodularity and greedy guarantee (Proof of Proposition~\ref{prop:greedy})}
\label{app:proof_submodularity}

\begin{proof}
Fix $S \subseteq [E]$ and $e \notin S$.  Let $A_S := \boldsymbol{\mathcal{K}}_S + \lambda_{\mathrm{reg}} \mathbf{I}$ and define the Schur complement $\sigma_e(S) := \mathcal{K}_{ee} + \lambda_{\mathrm{reg}} - \boldsymbol{\mathcal{K}}_{eS} A_S^{-1} \boldsymbol{\mathcal{K}}_{Se}$.  By the block-determinant formula:
\begin{align}
    \det\!\left(\boldsymbol{\mathcal{K}}_{S \cup \{e\}} + \lambda_{\mathrm{reg}} \mathbf{I}\right)
    = \det(A_S) \cdot \sigma_e(S).
\end{align}
Now
\begin{align}
    \begin{bmatrix}
        A_S & \boldsymbol{\mathcal{K}}_{Se} \\
        \boldsymbol{\mathcal{K}}_{eS} & \mathcal{K}_{ee}
    \end{bmatrix}
    =
    \begin{bmatrix}
        \boldsymbol{\mathcal{K}}_S & \boldsymbol{\mathcal{K}}_{Se} \\
        \boldsymbol{\mathcal{K}}_{eS} & \mathcal{K}_{ee}
    \end{bmatrix}
    +
    \begin{bmatrix}
        \lambda_{\mathrm{reg}} \mathbf{I} & 0 \\
        0 & 0
    \end{bmatrix}
    \succeq 0,
\end{align}
because both summands are positive semidefinite and $A_S \succ 0$.  Taking the Schur complement with respect to the positive-definite block $A_S$ gives
\begin{align}
    \mathcal{K}_{ee} - \boldsymbol{\mathcal{K}}_{eS} A_S^{-1} \boldsymbol{\mathcal{K}}_{Se} \geq 0.
\end{align}
Therefore $\sigma_e(S) \geq \lambda_{\mathrm{reg}}$, and hence
\begin{align}
    \widetilde{F}(S \cup \{e\}) - \widetilde{F}(S)
    = \log\!\left(\frac{\sigma_e(S)}{\lambda_{\mathrm{reg}}}\right)
    \geq 0,
\end{align}
which proves monotonicity.

For diminishing returns, fix $S \subseteq T \subseteq [E]$ and $e \notin T$.  We show $\sigma_e(S) \geq \sigma_e(T)$, i.e., $\boldsymbol{\mathcal{K}}_{eT} A_T^{-1} \boldsymbol{\mathcal{K}}_{Te} \geq \boldsymbol{\mathcal{K}}_{eS} A_S^{-1} \boldsymbol{\mathcal{K}}_{Se}$.  By the variational characterization of quadratic forms with positive-definite matrices:
\begin{align}
    \boldsymbol{\mathcal{K}}_{eT} A_T^{-1} \boldsymbol{\mathcal{K}}_{Te}
    = \max_{z \in \mathbb{R}^{|T|}} \bigl\{2z^\top \boldsymbol{\mathcal{K}}_{Te} - z^\top A_T z\bigr\}.
\end{align}
Restricting $z$ to have support only on $S$ (i.e., $z = (z_S, \mathbf{0}_{T \setminus S})$), the $S \times S$ principal subblock of $A_T$ is $A_S$, so the restricted maximum equals $\boldsymbol{\mathcal{K}}_{eS} A_S^{-1} \boldsymbol{\mathcal{K}}_{Se}$.  Since the unrestricted maximum is at least the restricted one:
\begin{align}
    \boldsymbol{\mathcal{K}}_{eT} A_T^{-1} \boldsymbol{\mathcal{K}}_{Te} \geq \boldsymbol{\mathcal{K}}_{eS} A_S^{-1} \boldsymbol{\mathcal{K}}_{Se}.
\end{align}
Hence $\widetilde{F}(S \cup \{e\}) - \widetilde{F}(S) = \log(\sigma_e(S)/\lambda_{\mathrm{reg}}) \geq \log(\sigma_e(T)/\lambda_{\mathrm{reg}}) = \widetilde{F}(T \cup \{e\}) - \widetilde{F}(T)$, which is submodularity.  The $(1{-}1/e)$ greedy bound follows from~\citet{nemhauser1978analysis}.
\end{proof}

\subsection{Incoherence bound (Proof of Theorem~\ref{thm:incoherence})}
\label{app:proof_incoherence}

\begin{definition}[Mutual coherence]
For $i \neq j$ with $\mathcal{K}_{ii}\mathcal{K}_{jj} > 0$, define $\rho_{ij} := \mathcal{K}_{ij} / \sqrt{\mathcal{K}_{ii}\mathcal{K}_{jj}}$ and $\mu := \max_{i \neq j} |\rho_{ij}|$.
\end{definition}

\begin{proof}
Fix $S$ with $|S| = K$.  Define the diagonal proxy $G(S) := \sum_{e \in S} \log(\mathcal{K}_{ee} + \lambda_{\mathrm{reg}})$ and the matrices:
\begin{align}
    B_S &:= \mathrm{diag}(\mathcal{K}_{ee} + \lambda_{\mathrm{reg}})_{e \in S}, &
    R_S &:= B_S^{-1/2}(\boldsymbol{\mathcal{K}}_S - \mathrm{diag}(\mathcal{K}_{ee})_{e \in S}) B_S^{-1/2}.
\end{align}
Then $\boldsymbol{\mathcal{K}}_S + \lambda_{\mathrm{reg}}\mathbf{I} = B_S^{1/2}(\mathbf{I} + R_S)B_S^{1/2}$, so $F(S) = G(S) + \log\det(\mathbf{I} + R_S)$.

The matrix $R_S$ is symmetric with zero diagonal, and every off-diagonal entry has magnitude at most $\mu$.  By Gershgorin's theorem, every eigenvalue of $R_S$ lies in $[-(K{-}1)\mu,\, (K{-}1)\mu]$.  Since $(K{-}1)\mu < 1$ by assumption, $\mathbf{I} + R_S$ is positive definite with eigenvalues in $[1{-}(K{-}1)\mu,\, 1{+}(K{-}1)\mu]$.  Taking logarithms and summing over the $K$ eigenvalues:
\begin{align}
    K\log\!\bigl(1 - (K{-}1)\mu\bigr) \leq \log\det(\mathbf{I} + R_S) \leq K\log\!\bigl(1 + (K{-}1)\mu\bigr).
\end{align}
Applying the upper bound to $S^\star = \argmax_{|S|=K} F(S)$ and the lower bound to $S_{\mathrm{diag}} = \argmax_{|S|=K} G(S)$, and using $G(S_{\mathrm{diag}}) \geq G(S^\star)$:
\begin{align}
    F(S_{\mathrm{diag}}) \geq F(S^\star) - K\log\!\left(\frac{1 + (K{-}1)\mu}{1 - (K{-}1)\mu}\right).
\end{align}
\end{proof}

\section{Additional theoretical results}
\label{app:theory_extensions}

\subsection{Finite-sample calibration guarantee}
\label{app:finite_sample}

In practice, the DO kernel is estimated from a finite calibration set using the CP or ACP scores defined in Section~\ref{sec:scoring}.  The
following theorem analyzes those estimators directly.

\begin{theorem}[Uniform stability of empirical DO-CP and DO-ACP]
\label{thm:finite_sample}
For each token $t$, let $\mathcal{S}(t) = \operatorname{top\text{-}k}\{p^{(e)}(t)\}_{e=1}^{E}$, with any fixed tie-breaking
rule, and define
\begin{align}
    z_e(t) &:= \mathbf{1}\{e \in \mathcal{S}(t)\}, \\
    q_e &:= \mathbb{E}[z_e(t)], &
    a_e &:= \mathbb{E}[z_e(t)\,p^{(e)}(t)], &
    v_e &:= \mathbb{E}[\|f_e(t)\|_2^2], \\
    \mathrm{CP}_e &:= \frac{a_e}{q_e}, &
    \mathrm{ACP}_e &:= \frac{a_e}{q_e}\sqrt{v_e}, &
    G_{ij} &:= \mathbb{E}[\langle f_i(t), f_j(t)\rangle].
\end{align}
Given i.i.d.\ calibration tokens $t_1,\dots,t_n \sim \mathcal{D}$, define empirical quantities
\begin{align}
    \widehat q_e &:= \frac{1}{n}\sum_{m=1}^{n} z_e(t_m), &
    \widehat a_e &:= \frac{1}{n}\sum_{m=1}^{n} z_e(t_m)\,p^{(e)}(t_m), \\
    \widehat v_e &:= \frac{1}{n}\sum_{m=1}^{n} \|f_e(t_m)\|_2^2, &
    \widehat G_{ij} &:= \frac{1}{n}\sum_{m=1}^{n} \langle f_i(t_m), f_j(t_m)\rangle.
\end{align}
Set
\begin{align}
    \widehat{\mathrm{CP}}_e
    &:=
    \begin{cases}
        \widehat a_e / \widehat q_e, & \widehat q_e > 0, \\
        0, & \widehat q_e = 0,
    \end{cases} \\
    \widehat{\mathrm{ACP}}_e
    &:=
    \widehat{\mathrm{CP}}_e \sqrt{\widehat v_e}.
\end{align}
Choose the base score either as DO-CP, with
\begin{align}
    I_e := \mathrm{CP}_e,
    \qquad
    \widehat I_e := \widehat{\mathrm{CP}}_e,
\end{align}
or as DO-ACP, with
\begin{align}
    I_e := \mathrm{ACP}_e,
    \qquad
    \widehat I_e := \widehat{\mathrm{ACP}}_e.
\end{align}
Let
\begin{align}
    \mathcal{K}_{ij} := \sqrt{I_i I_j}\,G_{ij},
    \qquad
    \widehat{\mathcal{K}}_{ij} := \sqrt{\widehat I_i \widehat I_j}\,\widehat G_{ij}.
\end{align}
For a target subset size $K$, set the population and empirical regularizers by the same diagonal rule as Section~\ref{sec:scoring}:
\begin{align}
    \lambda_{\mathrm{reg}}
    :=
    \frac{1}{K E}\sum_{e=1}^{E}\mathcal{K}_{ee},
    \qquad
    \widehat{\lambda}_{\mathrm{reg}}
    :=
    \frac{1}{K E}\sum_{e=1}^{E}\widehat{\mathcal{K}}_{ee},
\end{align}
Assume $\lambda_{\mathrm{reg}}>0$, and define
\begin{align}
    F(S) &:= \log\det(\boldsymbol{\mathcal{K}}_S + \lambda_{\mathrm{reg}}\mathbf{I}), &
    \widehat F(S) &:= \log\det(\widehat{\boldsymbol{\mathcal{K}}}_S + \widehat{\lambda}_{\mathrm{reg}}\mathbf{I}).
\end{align}

Assume $\|f_e(t)\|_2 \leq B_f$ almost surely for all $e \in [E]$, $q_e \geq q_{\min} > 0$ for all $e$, and
$I_e \in [I_{\min}, I_{\max}]$ for all $e$.  In the DO-ACP case, also assume $v_e \geq v_{\min} > 0$ for all $e$.
Define
\begin{align}
    r_n &:= \sqrt{\frac{2\log(8E^2/\delta)}{n}}, \\
    C_{\mathrm{CP}} &:= \frac{4}{q_{\min}}, \\
    C_{\mathrm{ACP}} &:= \frac{4B_f}{q_{\min}} + \frac{B_f^2}{\sqrt{2v_{\min}}}, \\
    C_I &:=
    \begin{cases}
        C_{\mathrm{CP}}, & \text{for DO-CP}, \\
        C_{\mathrm{ACP}}, & \text{for DO-ACP},
    \end{cases} \\
    C_h &:= B_f^2\sqrt{\frac{2I_{\max}}{I_{\min}} + 1} + I_{\max} + \frac{I_{\min}}{2}, \\
    \varepsilon_n &:= C_h \max\{C_I, B_f^2\}\,r_n, \\
    \Delta_n &:= \left(K + \frac{1}{K}\right)\varepsilon_n.
\end{align}
If
\begin{align}
    r_n \leq \frac{q_{\min}}{2},
    \qquad
    C_I r_n \leq \frac{I_{\min}}{2},
\end{align}
and, in the DO-ACP case,
\begin{align}
    B_f^2 r_n \leq \frac{v_{\min}}{2},
\end{align}
then with probability at least $1-\delta$:
\begin{align}
    \max_{i,j \in [E]} |\widehat{\mathcal{K}}_{ij} - \mathcal{K}_{ij}| \leq \varepsilon_n.
\end{align}
Consequently, if $\Delta_n < \lambda_{\mathrm{reg}}$, then
\begin{align}
    \bigl|\widehat{F}(S) - F(S)\bigr|
    \leq
    -K\log\!\left(1 - \frac{\Delta_n}{\lambda_{\mathrm{reg}}}\right)
    \qquad \text{for all } |S| = K,
\end{align}
and if $\widehat{S}$ maximizes $\widehat{F}$ while $S^\star$ maximizes $F$, then
\begin{align}
    F(\widehat{S})
    \geq
    F(S^\star)
    - K\log\!\left(
        \frac{1 + \Delta_n/\lambda_{\mathrm{reg}}}
             {1 - \Delta_n/\lambda_{\mathrm{reg}}}
    \right).
\end{align}
\end{theorem}

\begin{proof}
We proceed in four steps.  Write $\lambda := \lambda_{\mathrm{reg}}$ for brevity.

\paragraph{Step 1: Uniform concentration of the empirical moments.}
For each $e \in [E]$, the random variables $z_e(t)$ and $z_e(t)p^{(e)}(t)$ lie in $[0,1]$, while
$\|f_e(t)\|_2^2$ lies in $[0,B_f^2]$.  Also, for every $i,j \in [E]$,
\begin{align}
    |\langle f_i(t), f_j(t)\rangle|
    \leq \|f_i(t)\|_2 \|f_j(t)\|_2
    \leq B_f^2
\end{align}
almost surely.  Therefore Hoeffding's inequality gives
\begin{align}
    \Pr\!\left(
        \max_{e \in [E]} |\widehat q_e - q_e| > r_n
    \right)
    &\leq 2E \exp(-2nr_n^2), \\
    \Pr\!\left(
        \max_{e \in [E]} |\widehat a_e - a_e| > r_n
    \right)
    &\leq 2E \exp(-2nr_n^2), \\
    \Pr\!\left(
        \max_{e \in [E]} |\widehat v_e - v_e| > B_f^2 r_n
    \right)
    &\leq 2E \exp(-2nr_n^2), \\
    \Pr\!\left(
        \max_{i,j \in [E]} |\widehat G_{ij} - G_{ij}| > B_f^2 r_n
    \right)
    &\leq 2E^2 \exp\!\left(-\frac{nr_n^2}{2}\right).
\end{align}
By the definition of $r_n$ and a union bound over these four events, with probability at least $1-\delta$ we are on an
event $\Omega$ such that
\begin{align}
    \max_{e \in [E]} |\widehat q_e - q_e| &\leq r_n, &
    \max_{e \in [E]} |\widehat a_e - a_e| &\leq r_n, \\
    \max_{e \in [E]} |\widehat v_e - v_e| &\leq B_f^2 r_n, &
    \max_{i,j \in [E]} |\widehat G_{ij} - G_{ij}| &\leq B_f^2 r_n.
\end{align}

\paragraph{Step 2: Concentration of the empirical CP and ACP scores.}
Fix the event $\Omega$.  Since $r_n \leq q_{\min}/2$ and $q_e \geq q_{\min}$, we have $\widehat q_e \geq q_{\min}/2 > 0$
for every $e$, so $\widehat{\mathrm{CP}}_e$ is well-defined on $\Omega$.
Moreover, $0 \leq a_e \leq q_e$ because $0 \leq p^{(e)}(t) \leq 1$, so
\begin{align}
    |\widehat{\mathrm{CP}}_e - \mathrm{CP}_e|
    &= \left| \frac{\widehat a_e}{\widehat q_e} - \frac{a_e}{q_e} \right| \\
    &\leq \frac{|\widehat a_e - a_e|}{\widehat q_e}
        + a_e \left| \frac{1}{\widehat q_e} - \frac{1}{q_e} \right| \\
    &\leq \frac{2r_n}{q_{\min}}
        + q_e \frac{|\widehat q_e - q_e|}{q_e \widehat q_e} \\
    &\leq \frac{2r_n}{q_{\min}} + \frac{2r_n}{q_{\min}}
    = C_{\mathrm{CP}} r_n.
\end{align}
This proves the score concentration bound for DO-CP.

In the DO-ACP case, the condition $B_f^2 r_n \leq v_{\min}/2$ implies $\widehat v_e \geq v_{\min}/2$ for every $e$.
Since the derivative of $x \mapsto \sqrt{x}$ is bounded by $1/\sqrt{2v_{\min}}$ on $[v_{\min}/2,\infty)$, the mean-value
theorem gives
\begin{align}
    |\sqrt{\widehat v_e} - \sqrt{v_e}|
    \leq \frac{|\widehat v_e - v_e|}{\sqrt{2v_{\min}}}
    \leq \frac{B_f^2}{\sqrt{2v_{\min}}} r_n.
\end{align}
Also $\widehat v_e \leq B_f^2$, hence $\sqrt{\widehat v_e} \leq B_f$, and $\mathrm{CP}_e \leq 1$.  Therefore
\begin{align}
    |\widehat{\mathrm{ACP}}_e - \mathrm{ACP}_e|
    &= \left| \widehat{\mathrm{CP}}_e \sqrt{\widehat v_e} - \mathrm{CP}_e \sqrt{v_e} \right| \\
    &\leq \sqrt{\widehat v_e}\,|\widehat{\mathrm{CP}}_e - \mathrm{CP}_e|
        + \mathrm{CP}_e\,|\sqrt{\widehat v_e} - \sqrt{v_e}| \\
    &\leq B_f C_{\mathrm{CP}} r_n + \frac{B_f^2}{\sqrt{2v_{\min}}} r_n
    = C_{\mathrm{ACP}} r_n.
\end{align}
Thus, in either score choice,
\begin{align}
    \max_{e \in [E]} |\widehat I_e - I_e| \leq C_I r_n.
\end{align}

\paragraph{Step 3: Entrywise kernel concentration.}
Because $C_I r_n \leq I_{\min}/2$ and $I_e \geq I_{\min}$, we have
\begin{align}
    \widehat I_e \in \left[\frac{I_{\min}}{2},\, I_{\max} + \frac{I_{\min}}{2}\right]
    \qquad \text{for every } e \in [E].
\end{align}
Consider the map $h(a,b,c) := \sqrt{ab}\,c$ on the compact domain
\begin{align}
    D := \left[\frac{I_{\min}}{2},\, I_{\max} + \frac{I_{\min}}{2}\right]^2 \times [-B_f^2, B_f^2].
\end{align}
On $D$, its partial derivatives satisfy
\begin{align}
    \left|\frac{\partial h}{\partial a}\right|,
    \left|\frac{\partial h}{\partial b}\right|
    &\leq \frac{B_f^2}{2}\sqrt{\frac{2I_{\max}}{I_{\min}} + 1}, \\
    \left|\frac{\partial h}{\partial c}\right|
    &\leq I_{\max} + \frac{I_{\min}}{2}.
\end{align}
By the mean-value theorem,
\begin{align}
    |\widehat{\mathcal{K}}_{ij} - \mathcal{K}_{ij}|
    &=
    |h(\widehat I_i,\widehat I_j,\widehat G_{ij}) - h(I_i,I_j,G_{ij})| \\
    &\leq
    C_h \max\!\left\{
        |\widehat I_i - I_i|,
        |\widehat I_j - I_j|,
        |\widehat G_{ij} - G_{ij}|
    \right\} \\
    &\leq C_h \max\{C_I, B_f^2\}\,r_n
    = \varepsilon_n
\end{align}
for all $i,j \in [E]$.

\paragraph{Step 4: Transfer to the log-det objective.}
The diagonal regularizer is stable under the same entrywise bound:
\begin{align}
    |\widehat{\lambda}_{\mathrm{reg}} - \lambda_{\mathrm{reg}}|
    \leq
    \frac{\varepsilon_n}{K}.
\end{align}
For any $S$ with $|S| = K$, the matrix $\widehat{\boldsymbol{\mathcal{K}}}_S - \boldsymbol{\mathcal{K}}_S$ is $K \times K$
and has entrywise bound $\varepsilon_n$, so
\begin{align}
    \left\|
    \widehat{\boldsymbol{\mathcal{K}}}_S + \widehat{\lambda}_{\mathrm{reg}}\mathbf{I}
    - \boldsymbol{\mathcal{K}}_S - \lambda_{\mathrm{reg}}\mathbf{I}
    \right\|_{\mathrm{op}}
    \leq K\varepsilon_n + \frac{\varepsilon_n}{K}
    = \Delta_n.
\end{align}
Let $\sigma_1,\dots,\sigma_K$ be the eigenvalues of $\boldsymbol{\mathcal{K}}_S + \lambda \mathbf{I}$ and
$\widehat{\sigma}_1,\dots,\widehat{\sigma}_K$ those of $\widehat{\boldsymbol{\mathcal{K}}}_S + \widehat{\lambda}_{\mathrm{reg}} \mathbf{I}$.
Weyl's inequality gives
\begin{align}
    |\widehat{\sigma}_i - \sigma_i| \leq \Delta_n
    \qquad \text{for each } i \in [K].
\end{align}
Since $\Delta_n < \lambda$, every $\widehat{\sigma}_i$ is positive.  Therefore
\begin{align}
    \widehat{F}(S) - F(S)
    &= \sum_{i=1}^{K}\log\!\left(\frac{\widehat{\sigma}_i}{\sigma_i}\right)
    = \sum_{i=1}^{K}\log\!\left(1 + \frac{\widehat{\sigma}_i - \sigma_i}{\sigma_i}\right).
\end{align}
Because $\sigma_i \geq \lambda$ and $|\widehat{\sigma}_i - \sigma_i| \leq \Delta_n$, we have
\begin{align}
    K\log\!\left(1 - \frac{\Delta_n}{\lambda}\right)
    \leq \widehat{F}(S) - F(S)
    \leq
    K\log\!\left(1 + \frac{\Delta_n}{\lambda}\right).
\end{align}
Taking absolute values yields
\begin{align}
    \bigl|\widehat{F}(S) - F(S)\bigr|
    \leq
    -K\log\!\left(1 - \frac{\Delta_n}{\lambda}\right).
\end{align}
Applying the one-sided bounds to $\widehat{S}$ and $S^\star$ gives
\begin{align}
    F(\widehat{S})
    &\geq \widehat{F}(\widehat{S}) - K\log\!\left(1 + \frac{\Delta_n}{\lambda}\right) \\
    &\geq \widehat{F}(S^\star) - K\log\!\left(1 + \frac{\Delta_n}{\lambda}\right) \\
    &\geq F(S^\star) + K\log\!\left(1 - \frac{\Delta_n}{\lambda}\right)
      - K\log\!\left(1 + \frac{\Delta_n}{\lambda}\right),
\end{align}
which is exactly the claimed transfer bound.
\end{proof}

\subsection{Grouping recovery}
\label{app:theory_grouping}

This subsection analyzes the specific OC variant defined in Appendix~\ref{app:grouping_defs}: average-linkage agglomerative
clustering on empirical cosine dissimilarities.

Fix a selected expert set $S \subseteq [E]$ with $|S| = K$ and a target group count $G$.
Define the population second moments
\begin{align}
    M_{ij} := \mathbb{E}[\langle f_i(t), f_j(t)\rangle],
    \qquad i,j \in S,
\end{align}
and assume throughout this subsection that
\begin{align}
    M_{ee} > 0 \qquad \text{for every } e \in S,
\end{align}
so that the population cosine similarities below are well defined.  Define
\begin{align}
    \rho_{ij} := \frac{M_{ij}}{\sqrt{M_{ii}M_{jj}}},
    \qquad i,j \in S,
\end{align}
the population cosine similarities, and
\begin{align}
    d(i,j) := 1 - \rho_{ij}.
\end{align}
For two nonempty disjoint clusters $C,C' \subseteq S$, define the average-linkage distance
\begin{align}
    d_{\mathrm{avg}}(C,C')
    :=
    \frac{1}{|C|\,|C'|}
    \sum_{i \in C}\sum_{j \in C'} d(i,j).
\end{align}

\begin{definition}[Average-linkage output clustering]
Starting from the singleton partition of $S$, repeatedly merge the pair of clusters with minimum
$d_{\mathrm{avg}}(C,C')$.  Stop when exactly $G$ clusters remain.
\end{definition}

\begin{assumption}[Output-space separation]
\label{assump:oc-sep}
There exists a partition $\mathcal{P}^\star = \{\mathcal{G}_1^\star,\dots,\mathcal{G}_G^\star\}$ of $S$ and constants
$\Delta_{\mathrm{in}}, \Delta_{\mathrm{out}}$ with $\Delta_{\mathrm{in}} < \Delta_{\mathrm{out}}$ such that
\begin{align}
    \max_{g \in [G]} \max_{i,j \in \mathcal{G}_g^\star} d(i,j)
    \leq \Delta_{\mathrm{in}}
    <
    \Delta_{\mathrm{out}}
    \leq
    \min_{g \neq h} \min_{\substack{i \in \mathcal{G}_g^\star \\ j \in \mathcal{G}_h^\star}} d(i,j).
\end{align}
\end{assumption}

\begin{theorem}[Exact recovery of output clustering]
\label{thm:oc-exact}
Assume Assumption~\ref{assump:oc-sep}.  Then average-linkage output clustering returns $\mathcal{P}^\star$
(up to permutation of the group labels).
\end{theorem}

\begin{proof}
\emph{Step 1: Cluster-level separation.}
Let $C \subseteq \mathcal{G}_g^\star$ and $C' \subseteq \mathcal{G}_g^\star$ be disjoint clusters contained in the same true
group.  Then every pair $(i,j) \in C \times C'$ satisfies $d(i,j) \leq \Delta_{\mathrm{in}}$, hence
\begin{align}
    d_{\mathrm{avg}}(C,C') \leq \Delta_{\mathrm{in}}.
\end{align}
Similarly, if $C \subseteq \mathcal{G}_g^\star$ and $C' \subseteq \mathcal{G}_h^\star$ with $g \neq h$, then every pair
$(i,j) \in C \times C'$ satisfies $d(i,j) \geq \Delta_{\mathrm{out}}$, so
\begin{align}
    d_{\mathrm{avg}}(C,C') \geq \Delta_{\mathrm{out}}.
\end{align}

\emph{Step 2: Induction on the merge steps.}
Initially every cluster is a singleton and is therefore contained in some true group.  Suppose inductively that, at a given
iteration, every current cluster is contained in a true group.  If a true group $\mathcal{G}_g^\star$ is represented by at
least two current clusters $C_1,C_2 \subseteq \mathcal{G}_g^\star$, then Step~1 gives
\begin{align}
    d_{\mathrm{avg}}(C_1,C_2) \leq \Delta_{\mathrm{in}}.
\end{align}
By contrast, any pair of clusters drawn from different true groups has average-linkage distance at least
$\Delta_{\mathrm{out}}$.  Since $\Delta_{\mathrm{in}} < \Delta_{\mathrm{out}}$, the globally closest pair must lie within the
same true group.  Therefore the algorithm performs an intra-group merge, and the merged cluster remains contained in a true
group.  This preserves the induction hypothesis.

\emph{Step 3: Termination.}
By Step~2, the algorithm performs only intra-group merges until each true group has been merged into a single cluster.
At that moment exactly $G$ clusters remain, namely $\mathcal{G}_1^\star,\dots,\mathcal{G}_G^\star$.
\end{proof}

\begin{assumption}[Bounded outputs and nondegenerate norms for output clustering]
\label{assump:oc-bounded}
There exist constants $B_f,\sigma_{\min} > 0$ such that
\begin{align}
    \|f_e(t)\|_2 \leq B_f \quad \text{almost surely}
    \qquad \text{and} \qquad
    M_{ee} \geq \sigma_{\min}^2
\end{align}
for every $e \in S$.
\end{assumption}

\begin{lemma}[Uniform concentration of empirical cosine similarities]
\label{lem:oc-conc}
Assume Assumption~\ref{assump:oc-bounded}.  Given i.i.d.\ calibration tokens $t_1,\dots,t_n \sim \mathcal{D}$, define
\begin{align}
    \widehat M_{ij}
    &:= \frac{1}{n}\sum_{m=1}^{n} \langle f_i(t_m), f_j(t_m)\rangle, \\
    \widehat\rho_{ij}
    &:=
    \begin{cases}
        \widehat M_{ij} / \sqrt{\widehat M_{ii}\widehat M_{jj}}, & \widehat M_{ii}\widehat M_{jj} > 0, \\
        0, & \widehat M_{ii}\widehat M_{jj} = 0,
    \end{cases} \\
    \widehat d(i,j) &:= 1 - \widehat\rho_{ij}.
\end{align}
Let
\begin{align}
    C_{\mathrm{oc}} := \frac{2}{\sigma_{\min}^2} + \frac{4B_f^2}{\sigma_{\min}^4}.
\end{align}
Then for any $\tau \in (0,\sigma_{\min}^2/2]$,
\begin{align}
    \Pr\!\left(
        \max_{i,j \in S} |\widehat M_{ij} - M_{ij}| > \tau
        \;\text{or}\;
        \max_{i,j \in S} |\widehat\rho_{ij} - \rho_{ij}| > C_{\mathrm{oc}}\tau
    \right)
    \leq
    2|S|^2 \exp\!\left(-\frac{n\tau^2}{2B_f^4}\right).
\end{align}
\end{lemma}

\begin{proof}
Fix $i,j \in S$ and define $X_m^{(ij)} := \langle f_i(t_m), f_j(t_m)\rangle$.  By Cauchy--Schwarz and
Assumption~\ref{assump:oc-bounded},
\begin{align}
    |X_m^{(ij)}|
    \leq \|f_i(t_m)\|_2 \|f_j(t_m)\|_2
    \leq B_f^2
\end{align}
almost surely, and $\mathbb{E}[X_m^{(ij)}] = M_{ij}$.  Hoeffding's inequality therefore gives
\begin{align}
    \Pr\!\left(
        |\widehat M_{ij} - M_{ij}| > \tau
    \right)
    \leq
    2\exp\!\left(-\frac{n\tau^2}{2B_f^4}\right).
\end{align}
Applying a union bound over $(i,j) \in S \times S$, we obtain an event $\Omega_\tau$ of probability at least
\begin{align}
    1 - 2|S|^2\exp\!\left(-\frac{n\tau^2}{2B_f^4}\right)
\end{align}
on which
\begin{align}
    \max_{i,j \in S} |\widehat M_{ij} - M_{ij}| \leq \tau.
\end{align}
Fix this event.  Since $M_{ii} \geq \sigma_{\min}^2$ and $\tau \leq \sigma_{\min}^2/2$, we have
\begin{align}
    \widehat M_{ii} \geq M_{ii} - \tau \geq \sigma_{\min}^2/2 > 0
\end{align}
for every $i \in S$.  Hence $\widehat\rho_{ij} = \widehat M_{ij}/\sqrt{\widehat M_{ii}\widehat M_{jj}}$ on $\Omega_\tau$.
Now fix $i,j \in S$.  Using $\sqrt{\widehat M_{ii}\widehat M_{jj}} \geq \sigma_{\min}^2/2$,
\begin{align}
    |\widehat\rho_{ij} - \rho_{ij}|
    &\leq
    \frac{|\widehat M_{ij} - M_{ij}|}{\sqrt{\widehat M_{ii}\widehat M_{jj}}}
    +
    |M_{ij}|
    \left|
        \frac{1}{\sqrt{\widehat M_{ii}\widehat M_{jj}}}
        -
        \frac{1}{\sqrt{M_{ii}M_{jj}}}
    \right| \\
    &\leq
    \frac{2\tau}{\sigma_{\min}^2}
    +
    B_f^2
    \left|
        \frac{1}{\sqrt{\widehat M_{ii}\widehat M_{jj}}}
        -
        \frac{1}{\sqrt{M_{ii}M_{jj}}}
    \right|.
\end{align}
For the reciprocal term,
\begin{align}
    \left|
        \frac{1}{\sqrt{\widehat M_{ii}\widehat M_{jj}}}
        -
        \frac{1}{\sqrt{M_{ii}M_{jj}}}
    \right|
    &\leq
    \frac{1}{\sqrt{\widehat M_{jj}}}
    \left|
        \frac{1}{\sqrt{\widehat M_{ii}}}
        -
        \frac{1}{\sqrt{M_{ii}}}
    \right|
    +
    \frac{1}{\sqrt{M_{ii}}}
    \left|
        \frac{1}{\sqrt{\widehat M_{jj}}}
        -
        \frac{1}{\sqrt{M_{jj}}}
    \right|.
\end{align}
Since $x \mapsto x^{-1/2}$ has derivative $-\frac{1}{2}x^{-3/2}$, it is $\sqrt{2}/\sigma_{\min}^3$-Lipschitz on
$[\sigma_{\min}^2/2,\infty)$.  Therefore
\begin{align}
    \left|
        \frac{1}{\sqrt{\widehat M_{ii}}}
        -
        \frac{1}{\sqrt{M_{ii}}}
    \right|,
    \left|
        \frac{1}{\sqrt{\widehat M_{jj}}}
        -
        \frac{1}{\sqrt{M_{jj}}}
    \right|
    \leq
    \frac{\sqrt{2}}{\sigma_{\min}^3}\tau.
\end{align}
Using $1/\sqrt{\widehat M_{jj}} \leq \sqrt{2}/\sigma_{\min}$ and $1/\sqrt{M_{ii}} \leq 1/\sigma_{\min}$, we obtain
\begin{align}
    \left|
        \frac{1}{\sqrt{\widehat M_{ii}\widehat M_{jj}}}
        -
        \frac{1}{\sqrt{M_{ii}M_{jj}}}
    \right|
    \leq
    \frac{2+\sqrt{2}}{\sigma_{\min}^4}\tau
    \leq
    \frac{4}{\sigma_{\min}^4}\tau.
\end{align}
Combining the preceding bounds yields
\begin{align}
    |\widehat\rho_{ij} - \rho_{ij}|
    \leq
    \left(
        \frac{2}{\sigma_{\min}^2} + \frac{4B_f^2}{\sigma_{\min}^4}
    \right)\tau
    = C_{\mathrm{oc}}\tau.
\end{align}
This holds uniformly over $i,j \in S$ on $\Omega_\tau$.
\end{proof}

\begin{theorem}[Finite-sample recovery of output clustering]
\label{thm:oc-finite}
Assume Assumptions~\ref{assump:oc-sep} and~\ref{assump:oc-bounded}.  Let
\begin{align}
    \gamma := \Delta_{\mathrm{out}} - \Delta_{\mathrm{in}} > 0,
    \qquad
    \tau_\star := \min\!\left\{
        \frac{\sigma_{\min}^2}{2},
        \frac{\gamma}{4C_{\mathrm{oc}}}
    \right\}.
\end{align}
If
\begin{align}
    n \geq \frac{2B_f^4}{\tau_\star^2}\log\!\left(\frac{2|S|^2}{\delta}\right),
\end{align}
then average-linkage output clustering on the empirical dissimilarities $\widehat d(i,j)$ returns $\mathcal{P}^\star$
with probability at least $1-\delta$.
\end{theorem}

\begin{proof}
By Lemma~\ref{lem:oc-conc} with $\tau = \tau_\star$, the stated lower bound on $n$ guarantees that with probability
at least $1-\delta$,
\begin{align}
    \max_{i,j \in S} |\widehat\rho_{ij} - \rho_{ij}|
    \leq C_{\mathrm{oc}}\tau_\star
    \leq \frac{\gamma}{4}.
\end{align}
Fix this event and set $\varepsilon := \gamma/4$.  Then
\begin{align}
    |\widehat d(i,j) - d(i,j)| = |\widehat\rho_{ij} - \rho_{ij}| \leq \varepsilon
    \qquad \text{for all } i,j \in S.
\end{align}
Hence, if $i,j$ lie in the same true group, Assumption~\ref{assump:oc-sep} gives
\begin{align}
    \widehat d(i,j) \leq d(i,j) + \varepsilon \leq \Delta_{\mathrm{in}} + \varepsilon.
\end{align}
If $i,j$ lie in different true groups, then
\begin{align}
    \widehat d(i,j) \geq d(i,j) - \varepsilon \geq \Delta_{\mathrm{out}} - \varepsilon.
\end{align}
Since $\varepsilon = \gamma/4$, we have
\begin{align}
    \Delta_{\mathrm{in}} + \varepsilon
    <
    \Delta_{\mathrm{out}} - \varepsilon.
\end{align}
Thus the empirical dissimilarity $\widehat d$ satisfies the same strict separation property as in
Assumption~\ref{assump:oc-sep}, with thresholds $\Delta_{\mathrm{in}} + \varepsilon$ and
$\Delta_{\mathrm{out}} - \varepsilon$.  Theorem~\ref{thm:oc-exact} therefore applies to $\widehat d$ and yields exact
recovery of $\mathcal{P}^\star$.
\end{proof}

\subsection{Merging optimality}
\label{app:theory_merging}

Fix a selected subset $S \subseteq [E]$ and a partition $\{\mathcal{G}_g\}_{g=1}^{G}$ of $S$.
View each expert output $f_e$ as an element of the Hilbert space
\begin{align}
    \mathcal{H} := L^2(\mathcal{D}; \mathbb{R}^d)
\end{align}
with inner product
\begin{align}
    \langle u, v \rangle_{\mathcal{H}}
    :=
    \mathbb{E}_{t \sim \mathcal{D}}[\langle u(t), v(t)\rangle].
\end{align}

\begin{theorem}[Score-weighted output averaging is the oracle proxy representative]
\label{thm:oracle_merge}
Let $s^{(e)} \geq 0$ be the expert scores used for within-group weighting, and assume that every group has strictly positive
score mass
\begin{align}
    S_g := \sum_{e \in \mathcal{G}_g} s^{(e)} > 0.
\end{align}
For group representatives $\mu_1,\dots,\mu_G \in \mathcal{H}$, define the quadratic merge distortion
\begin{align}
    \mathcal{L}(\mu_1,\dots,\mu_G)
    :=
    \sum_{g=1}^{G}\sum_{e \in \mathcal{G}_g}
    s^{(e)} \|f_e - \mu_g\|_{\mathcal{H}}^2.
\end{align}
Then the unique minimizer is
\begin{align}
    \mu_g^\star
    :=
    \frac{1}{S_g}\sum_{e \in \mathcal{G}_g} s^{(e)} f_e,
    \qquad g \in [G].
\end{align}
Thus, if merging were performed directly in output-function space, the oracle representative would use the same score weights that Eq.~\eqref{eq:weighted_avg} uses in parameter space.
\end{theorem}

\begin{proof}
Fix a group $g$ and write
\begin{align}
    L_g(\mu)
    := \sum_{e \in \mathcal{G}_g} s^{(e)} \|f_e - \mu\|_{\mathcal{H}}^2.
\end{align}
Set $\mu_g^\star := S_g^{-1}\sum_{e \in \mathcal{G}_g} s^{(e)} f_e$.
For any $\mu \in \mathcal{H}$, expand
\begin{align}
    f_e - \mu = (f_e - \mu_g^\star) + (\mu_g^\star - \mu),
\end{align}
so
\begin{align}
    \|f_e - \mu\|_{\mathcal{H}}^2
    &=
    \|f_e - \mu_g^\star\|_{\mathcal{H}}^2
    +
    \|\mu_g^\star - \mu\|_{\mathcal{H}}^2
    +
    2\langle f_e - \mu_g^\star, \mu_g^\star - \mu\rangle_{\mathcal{H}}.
\end{align}
Multiply by $s^{(e)}$ and sum over $e \in \mathcal{G}_g$:
\begin{align}
    L_g(\mu)
    &=
    \sum_{e \in \mathcal{G}_g} s^{(e)} \|f_e - \mu_g^\star\|_{\mathcal{H}}^2
    +
    S_g \|\mu_g^\star - \mu\|_{\mathcal{H}}^2 \notag \\
    &\quad +
    2\left\langle
        \sum_{e \in \mathcal{G}_g} s^{(e)}(f_e - \mu_g^\star),
        \mu_g^\star - \mu
    \right\rangle_{\mathcal{H}}.
\end{align}
The last inner-product term vanishes by the definition of $\mu_g^\star$, giving
\begin{align}
    L_g(\mu)
    =
    L_g(\mu_g^\star)
    +
    S_g \|\mu - \mu_g^\star\|_{\mathcal{H}}^2.
\end{align}
Since $S_g > 0$, the unique minimizer of $L_g$ is $\mu_g^\star$.  Summing this identity over $g = 1,\dots,G$ proves that
$(\mu_g^\star)_{g=1}^{G}$ uniquely minimizes $\mathcal{L}$.
\end{proof}

\begin{remark}[Parameter-space averaging is a proxy]
Theorem~\ref{thm:oracle_merge} is a statement about output functions in $\mathcal{H}$.
It does not imply that averaging the parameters of nonlinear SwiGLU experts exactly realizes $\mu_g^\star$.
The implemented merge in Eq.~\eqref{eq:weighted_avg} applies the oracle function-space weights to parameters as a practical proxy; it is exact only in special cases such as linear experts or singleton groups.
\end{remark}

\begin{proposition}[Scaling induces explicit global weights in the function-space proxy]
\label{prop:scaling_weights}
Let
\begin{align}
    w^{(e)} := \frac{s^{(e)}}{S_g},
    \qquad e \in \mathcal{G}_g,
\end{align}
be the score weights from Eq.~\eqref{eq:weighted_avg}, define the function-space proxy representative
\begin{align}
    \mu_g^{\mathrm{proxy}} := \sum_{e \in \mathcal{G}_g} w^{(e)} f_e,
\end{align}
and let $\alpha_1,\dots,\alpha_G \geq 0$ be arbitrary group scaling factors.
Then the scaled proxy output can be written exactly as
\begin{align}
    \sum_{g=1}^{G} \alpha_g \mu_g^{\mathrm{proxy}}
    =
    \sum_{e \in S} \eta_e f_e,
    \qquad
    \eta_e := \alpha_{g(e)} \frac{s^{(e)}}{S_{g(e)}},
\end{align}
where $g(e)$ is the unique group index such that $e \in \mathcal{G}_{g(e)}$.
\end{proposition}

\begin{proof}
By definition of $w^{(e)}$,
\begin{align}
    \sum_{g=1}^{G} \alpha_g \mu_g^{\mathrm{proxy}}
    &=
    \sum_{g=1}^{G} \alpha_g \left( \sum_{e \in \mathcal{G}_g} \frac{s^{(e)}}{S_g} f_e \right) \\
    &=
    \sum_{g=1}^{G}\sum_{e \in \mathcal{G}_g} \alpha_g \frac{s^{(e)}}{S_g} f_e \\
    &=
    \sum_{e \in S} \alpha_{g(e)} \frac{s^{(e)}}{S_{g(e)}} f_e,
\end{align}
because $\{\mathcal{G}_g\}_{g=1}^{G}$ partitions $S$.
\end{proof}

\begin{corollary}[Proportional scaling recovers global score weighting in the proxy]
\label{cor:prop_scaling}
Under proportional scaling,
\begin{align}
    \alpha_g = \frac{S_g}{\sum_{e' \in S} s^{(e')}},
\end{align}
the induced global weights are
\begin{align}
    \eta_e = \frac{s^{(e)}}{\sum_{e' \in S} s^{(e')}}.
\end{align}
Thus proportional scaling makes the scaled proxy output exactly equal to the global score-weighted average over the selected
experts.
\end{corollary}

\begin{proof}
Substitute the proportional-scaling formula into Proposition~\ref{prop:scaling_weights}:
\begin{align}
    \eta_e
    =
    \frac{S_{g(e)}}{\sum_{e' \in S} s^{(e')}} \cdot \frac{s^{(e)}}{S_{g(e)}}
    =
    \frac{s^{(e)}}{\sum_{e' \in S} s^{(e')}}.
\end{align}
\end{proof}

\begin{corollary}[Uniform scaling equalizes group mass in the proxy]
\label{cor:uniform_scaling}
Under uniform scaling,
\begin{align}
    \alpha_g = \frac{1}{G},
\end{align}
the induced global weights are
\begin{align}
    \eta_e
    =
    \frac{s^{(e)}}{G\,S_{g(e)}},
\end{align}
each group's total weight is exactly
\begin{align}
    \sum_{e \in \mathcal{G}_g} \eta_e = \frac{1}{G},
\end{align}
and the total weight across all selected experts sums to $1$, matching the simplex constraint of the original MoE router.
\end{corollary}

\begin{proof}
The first identity is Proposition~\ref{prop:scaling_weights} with $\alpha_g = 1/G$.
Summing over one group gives
\begin{align}
    \sum_{e \in \mathcal{G}_g} \eta_e
    =
    \frac{1}{G\,S_g} \sum_{e \in \mathcal{G}_g} s^{(e)}
    =
    \frac{1}{G},
\end{align}
and summing over all $G$ groups yields $\sum_{e \in S} \eta_e = 1$.
\end{proof}

\section{Grouping strategy definitions}
\label{app:grouping_defs}

Given $K$ selected experts assigned to $k$ groups:

\paragraph{Round-robin (RR).} Experts are sorted by descending score and assigned cyclically: expert with rank $r$ goes to group $r \bmod G$.

\paragraph{Weight clustering (WC).} Agglomerative clustering on the concatenated flattened gate, up, and down projection matrices (${\sim}$4.7M dimensions) with cosine similarity.

\paragraph{Router clustering (RC).} Agglomerative clustering on router gate weight vectors~\citep{li2024mcsmoe}, using cosine similarity in the $d$-dimensional router space.

\paragraph{Anchor-based (AB).} The $k$ highest-scoring experts serve as anchors; remaining experts are assigned to the anchor with highest router-vector cosine similarity~\citep{li2024mcsmoe}.

\paragraph{Output clustering (OC).} Average-linkage agglomerative clustering on the empirical cosine dissimilarities.  Let $\bar{G}_{ij} = \frac{1}{n}\sum_{m=1}^{n}\langle f_i(t_m), f_j(t_m)\rangle$ and $\bar{V}_i = \frac{1}{n}\sum_{m=1}^{n}\|f_i(t_m)\|_2^2$.  Then:
\begin{align}
    \widehat\rho_{ij}
    &:=
    \begin{cases}
        \bar{G}_{ij} \big/ \sqrt{\bar{V}_i \, \bar{V}_j},
        & \bar{V}_i \, \bar{V}_j > 0, \\[4pt]
        0, & \text{otherwise},
    \end{cases} \\
    \widehat d(i,j) &:= 1 - \widehat\rho_{ij},
\end{align}
computed from the calibration tokens $t_1,\dots,t_n$.  HC-SMoE~\citep{chen2025hcsmoe} showed output-based clustering
substantially outperforms weight and router similarity for grouping.

\section{Down-projection scaling equations}
\label{app:scaling_defs}

After merging, each group's down-projection is scaled by $\alpha_g$.

\paragraph{Uniform scaling.}  Each group is scaled equally so that the total contribution sums to $1$, matching the simplex constraint of the MoE router:
\begin{equation}
    \label{eq:uniform_scaling}
    \tilde{\mathbf{W}}_{\text{down}}^{(g)} = \frac{1}{k} \, \mathbf{W}_{\text{down}}^{(g)}.
\end{equation}

\paragraph{Proportional scaling.}
\begin{equation}
    \label{eq:proportional_scaling}
    \tilde{\mathbf{W}}_{\text{down}}^{(g)} = \frac{\sum_{e \in \mathcal{G}_g} s_\ell^{(e)}}{\sum_{e' \in S} s_\ell^{(e')}} \; \mathbf{W}_{\text{down}}^{(g)},
\end{equation}
where $S$ is the set of $K$ selected experts.  Proportional scaling preserves relative contribution magnitude.

\section{Base vs.\ post-trained teacher comparison}
\label{sec:base_vs_instruct}

All primary experiments use the post-trained (instruct/reasoning hybrid) variant of Qwen3-30B-A3B as the teacher model.  To verify that this choice does not bias our findings, we repeat the full 350-configuration pre-distill PPL sweep and distill the top-1 configuration per scoring method at both $K{=}8$ and $K{=}16$ using the base variant (\texttt{Qwen3-30B-A3B-Base}) with identically collected importance scores and Gram matrices.

\paragraph{Scoring rankings are preserved across model variants.}
Table~\ref{tab:base_vs_instruct} compares the best pre-distill PPL per scoring method between the base and post-trained teacher variants.  The ranking by best PPL is nearly identical:

\begin{table}[h]
\caption{Best pre-distill PPL per scoring method: post-trained (instruct) vs.\ base teacher.  Rankings match in the top-4 positions (ACP, DO-CP, DO-ACP, CP), confirming that scoring method conclusions are robust to teacher variant choice.}
\label{tab:base_vs_instruct}
\centering
\small
\begin{tabular}{l cc cc}
\toprule
\textbf{Scoring} & \multicolumn{2}{c}{\textbf{Post-trained (Instruct)}} & \multicolumn{2}{c}{\textbf{Base}} \\
\cmidrule(lr){2-3} \cmidrule(lr){4-5}
 & \textbf{Best PPL} & \textbf{Rank} & \textbf{Best PPL} & \textbf{Rank} \\
\midrule
ACP    & 2{,}002 & 1 & 2{,}424 & 1 \\
DO-CP  & 2{,}840 & 2 & 4{,}608 & 2 \\
DO-ACP & 3{,}939 & 3 & 5{,}697 & 3 \\
CP     & 4{,}280 & 4 & 8{,}175 & 4 \\
PP     & 9{,}463 & 5 & 9{,}326 & 5 \\
PS     & 8{,}509 & 6 & 11{,}308 & 6 \\
SF     & 8{,}638 & 7 & 13{,}370 & 7 \\
\bottomrule
\end{tabular}
\end{table}

The top-4 scoring methods (ACP, DO-CP, DO-ACP, CP) maintain their exact ranking across both variants.  The bottom three (PP, PS, SF) swap positions 5--7 but remain tightly clustered.  Base PPL values are systematically higher than post-trained values (1.2--1.6$\times$).

\paragraph{Downstream benchmarks confirm PPL rankings.}
We distill the top-1 base configuration per scoring method at both $K{=}8$ (pure pruning) and $K{=}16$ (merging) for 0.3B tokens and evaluate on all five benchmarks.  Table~\ref{tab:base_distill} shows results.

\begin{table}[h]
\caption{Base teacher distillation results ($K{=}8$ and $K{=}16$, 0.3B tokens each).  \textbf{Bold}: best per scoring.  ACP and DO-ACP at $K{=}8$ outperform their $K{=}16$ counterparts, confirming the ``pure pruning wins'' finding from the main experiments.}
\label{tab:base_distill}
\centering
\small
\begin{tabular}{ll c ccccc c}
\toprule
\textbf{Scoring} & $K$ & \textbf{Pre-distill PPL} & \textbf{Wino} & \textbf{Hella} & \textbf{ARC-E} & \textbf{ARC-C} & \textbf{MMLU} & \textbf{Avg (\%)} \\
\midrule
\textbf{ACP} & \textbf{8} & 9{,}049 & 54.8 & \textbf{42.2} & \textbf{57.5} & \textbf{31.4} & 28.1 & \textbf{42.81} \\
ACP & 16 & 2{,}424 & 54.5 & 39.2 & 53.5 & 28.0 & 25.4 & 40.13 \\
\midrule
\textbf{DO-ACP} & \textbf{8} & 10{,}220 & 55.6 & 41.2 & 55.5 & 29.6 & \textbf{31.0} & \textbf{42.59} \\
DO-ACP & 16 & 5{,}697 & 54.6 & 39.6 & 53.7 & 28.6 & 28.1 & 40.93 \\
\midrule
CP & 8 & 9{,}223 & \textbf{55.6} & 39.8 & 54.6 & 28.5 & 26.9 & 41.08 \\
CP & 16 & 8{,}175 & 54.9 & 37.7 & 56.9 & 29.5 & 26.3 & 41.07 \\
\midrule
SF & 8 & 29{,}059 & 51.7 & 29.5 & 39.7 & 22.1 & 27.0 & 34.01 \\
SF & 16 & 13{,}370 & 52.6 & 33.4 & 48.4 & 24.1 & 27.2 & 37.17 \\
\midrule
\multicolumn{2}{l}{Random FFN} & -- & 51.9 & 28.0 & 33.0 & 22.3 & 26.1 & 32.26 \\
\multicolumn{2}{l}{Random init} & -- & 50.4 & 25.2 & 28.0 & 23.8 & 22.9 & 30.08 \\
\bottomrule
\end{tabular}
\end{table}

The three-tier scoring hierarchy from the main experiments is preserved: ACP and DO-ACP lead at ${\sim}$42.7\%, followed by CP at ${\sim}$41.1\%, then SF at 34--37\%, and random baselines at 30--32\%.  Pure pruning ($K{=}8$) outperforms merging ($K{=}16$) for ACP ($+$2.68~pp) and DO-ACP ($+$1.66~pp), confirming the finding from Section~\ref{sec:pruning_vs_merging}.  CP shows a near-tie across $K$ values, and only SF benefits from merging.  Absolute accuracy is comparable between teacher variants: the best base configuration (ACP at $K{=}8$, 42.81\%) slightly exceeds its instruct counterpart (42.52\%, $+$0.29~pp), while DO-ACP at $K{=}8$ on the base model (42.59\%) is 0.82~pp below instruct (43.41\%).

Overall, scoring method rankings from the main experiments (Section~\ref{sec:scoring_comparison}) are robust to teacher variant, and the pure pruning advantage (Section~\ref{sec:pruning_vs_merging}) holds for both variants.  This is consistent with~\citet{ministral3}, who found that post-trained teachers produce comparable or stronger students for knowledge distillation.

\section{Calibration sensitivity}
\label{app:calibration}

We measure how the top-8 experts selected per layer by DO-ACP respond to calibration size, sequence length, and domain.  Overlap is the mean fraction of experts shared with a fixed reference configuration (WikiText, $n{=}1024$, maximum sequence length 2048).  This reference is larger than the $n{=}512$ setting used for our main results, and the two overlap by 98.4\%.  All domain corpora are packed to the same token count and length distribution as the reference, so only content differs (Tables~\ref{tab:calibration_overlap} and~\ref{tab:calibration_pairwise}).  Across domains, 4.33 of the 8 experts per layer are selected by every domain, and 2.83 are selected by exactly one domain.

DO-ACP selection is robust to calibration size and sequence length (at least 91.7\% overlap for sets 4$\times$ smaller to 2$\times$ larger and for halved or doubled sequence length), but moves by 22--30\% across calibration domains.

\begin{table}[H]
\caption{Selection overlap with the reference calibration configuration.}
\label{tab:calibration_overlap}
\centering
\small
\begin{tabular}{l l cc r cc}
\toprule
\textbf{Axis} & \textbf{Configuration} & $n$ & \textbf{Seq.} & \textbf{Tokens} & \textbf{Overlap} & \textbf{Per-layer min} \\
\midrule
Size & 4$\times$ smaller & 256 & 2048 & 43{,}414 & 91.7\% & 62\% \\
Size & 2$\times$ smaller (main results) & 512 & 2048 & 85{,}352 & 98.4\% & 88\% \\
Size & 2$\times$ larger & 2048 & 2048 & 320{,}590 & 95.3\% & 88\% \\
\midrule
Seq.\ length & Half & 1024 & 1024 & 174{,}956 & 96.4\% & 88\% \\
Seq.\ length & Double & 1024 & 4096 & 174{,}956 & 97.1\% & 88\% \\
\midrule
Domain & Code & 1024 & 2048 & 174{,}998 & 70.3\% & 38\% \\
Domain & Benchmark-mix & 1024 & 2048 & 174{,}298 & 73.4\% & 50\% \\
Domain & Math & 1024 & 2048 & 174{,}990 & 75.0\% & 50\% \\
Domain & Instruction & 1024 & 2048 & 174{,}978 & 78.4\% & 38\% \\
\bottomrule
\end{tabular}
\end{table}

\begin{table}[H]
\caption{Pairwise selection overlap between calibration domains (\%).}
\label{tab:calibration_pairwise}
\centering
\small
\begin{tabular}{l ccccc}
\toprule
 & \textbf{WikiText} & \textbf{Code} & \textbf{Math} & \textbf{Instruction} & \textbf{Benchmark-mix} \\
\midrule
WikiText & 100.0 & 70.3 & 75.0 & 78.4 & 73.4 \\
Code & 70.3 & 100.0 & 75.5 & 75.5 & 69.0 \\
Math & 75.0 & 75.5 & 100.0 & 82.3 & 77.1 \\
Instruction & 78.4 & 75.5 & 82.3 & 100.0 & 78.6 \\
Benchmark-mix & 73.4 & 69.0 & 77.1 & 78.6 & 100.0 \\
\bottomrule
\end{tabular}
\end{table}

To test whether these shifts matter downstream, we re-run the pipeline with the shifted selections (Table~\ref{tab:calibration_downstream}).  A 25--27\% change in the selected experts changes final accuracy by only 1.0--1.6~pp.

\begin{table}[H]
\caption{Downstream accuracy after re-running the pipeline with shifted selections (0.3B tokens).}
\label{tab:calibration_downstream}
\centering
\small
\begin{tabular}{l c ccccc c}
\toprule
\textbf{Calibration} & \textbf{Overlap} & \textbf{Wino} & \textbf{Hella} & \textbf{ARC-E} & \textbf{ARC-C} & \textbf{MMLU} & \textbf{Avg (\%)} \\
\midrule
WikiText (main results) & -- & 57.0 & 41.1 & 57.4 & 29.9 & 31.7 & 43.41 \\
Math & 75.0\% & 56.3 & 42.1 & 59.1 & 31.2 & 33.4 & 44.42 \\
Benchmark-mix & 73.4\% & 56.4 & 43.1 & 59.3 & 31.7 & 34.7 & 45.01 \\
\bottomrule
\end{tabular}
\end{table}

\section{Full distillation hyperparameters}
\label{app:hyperparameters}

Table~\ref{tab:hyperparams} lists all hyperparameters used for knowledge distillation.  Each configuration processes approximately 0.3B tokens ($200 \times 384 \times 4096$) and takes ${\sim}$5.5 hours on 4$\times$~H200 GPUs.

\begin{table}[h]
\caption{Distillation hyperparameters for 200-step (0.3B-token) runs.}
\label{tab:hyperparams}
\centering
\begin{tabular}{ll}
\toprule
\textbf{Parameter} & \textbf{Value} \\
\midrule
Optimizer & AdamW ($\beta_1 = 0.9$, $\beta_2 = 0.95$) \\
Weight decay & 0.01 \\
Peak learning rate & $10^{-4}$ \\
Min learning rate & $10^{-5}$ \\
LR schedule & Cosine decay (warmup 20 steps) \\
Temperature ($\tau$) & 1.0 \\
Loss function & Forward KL divergence on logits \\
Sequence length & 4096 \\
Micro-batch size & 4 per GPU \\
Gradient accumulation & 24 steps \\
Global batch size & 384 sequences \\
Total steps & 200 (0.3B tokens) \\
\midrule
\multicolumn{2}{l}{\textit{Infrastructure}} \\
\midrule
GPUs & 4$\times$ NVIDIA H200 (140~GB HBM3e each) \\
Parallelism & DeepSpeed ZeRO Stage~2 \\
Precision & BF16 mixed precision \\
Attention & Flash Attention 2 \\
Training data & FineWeb-Edu (streaming, no repeat) \\
Calibration data & 512 samples from WikiText-103 \\
Evaluation data & WikiText-2 test set (full) \\
\bottomrule
\end{tabular}
\end{table}

\clearpage
\section{Full pre-distill PPL grid}
\label{app:full_grid}

Tables~\ref{tab:full_ppl_grid_routing} and~\ref{tab:full_ppl_grid_diverse} present the comprehensive WikiText-2 perplexity for all 350 pre-distill configurations: 7 scoring $\times$ 5 grouping $\times$ 2 scaling (Uniform/Proportional) $\times$ 5 values of $K$.  Both down-projection scaling options are shown separately for each scoring$\times$grouping pair.  At $K{=}8$ ($= k$), each group contains exactly one expert; for DO-CP, ACP, and DO-ACP the same 8 experts are selected regardless of grouping, so all five grouping rows share the same value.

\begin{table*}[h]
\caption{WikiText-2 pre-distill perplexity ($\downarrow$) for frequency-based and conditional probability scoring methods.  Each row shows one scoring $\times$ grouping $\times$ DP scaling combination across all $K$ values.  \textbf{Bold}: best $K$ per row.  ``--'': numerical instability.  Values $\geq$100k in compact notation.}
\label{tab:full_ppl_grid_routing}
\centering
\footnotesize
\setlength{\tabcolsep}{4pt}
\renewcommand{\arraystretch}{1.15}
\begin{tabular}{ll l rrrrr}
\toprule
\textbf{Scoring} & \textbf{Grouping} & \textbf{DP Scaling} & $K{=}8$ & $K{=}16$ & $K{=}32$ & $K{=}64$ & $K{=}128$ \\
\midrule
\multirow{10}{*}{\shortstack[l]{Selection\\Frequency}} & \multirow{2}{*}{Round-Robin} & Uniform & \textbf{20{,}633} & 21{,}112 & 22{,}343 & 22{,}853 & 22{,}799 \\
 &  & Proportional & 21{,}074 & \textbf{19{,}440} & 21{,}036 & 21{,}950 & 22{,}022 \\
 & \multirow{2}{*}{Weight Cluster} & Uniform & 20{,}633 & \textbf{12{,}896} & 22{,}333 & 31{,}088 & -- \\
 &  & Proportional & 21{,}074 & \textbf{17{,}661} & 20{,}733 & 64{,}433 & -- \\
 & \multirow{2}{*}{Router Cluster} & Uniform & 20{,}633 & \textbf{11{,}342} & 14{,}833 & 13{,}473 & 14{,}607 \\
 &  & Proportional & 21{,}074 & \textbf{11{,}594} & 18{,}269 & 30{,}705 & 80{,}436 \\
 & \multirow{2}{*}{Anchor-Based} & Uniform & \textbf{11{,}988} & 15{,}494 & 16{,}796 & 17{,}916 & 17{,}963 \\
 &  & Proportional & \textbf{11{,}668} & 13{,}863 & 15{,}362 & 16{,}385 & 16{,}661 \\
 & \multirow{2}{*}{Output Cluster} & Uniform & 11{,}988 & 12{,}499 & 17{,}518 & \textbf{8{,}638} & 26{,}775 \\
 &  & Proportional & \textbf{11{,}668} & 17{,}565 & 56{,}400 & 34{,}566 & 120k \\
\cmidrule(lr){1-8}
\multirow{10}{*}{\shortstack[l]{Pre-Selection\\Probability}} & \multirow{2}{*}{Round-Robin} & Uniform & \textbf{14{,}983} & 15{,}523 & 36{,}409 & 162k & 315k \\
 &  & Proportional & \textbf{13{,}340} & 13{,}555 & 23{,}351 & 79{,}675 & 152k \\
 & \multirow{2}{*}{Weight Cluster} & Uniform & 14{,}983 & \textbf{11{,}085} & 12{,}528 & 19{,}148 & 20{,}430 \\
 &  & Proportional & 13{,}340 & \textbf{11{,}794} & 22{,}710 & 74{,}225 & 430k \\
 & \multirow{2}{*}{Router Cluster} & Uniform & 14{,}983 & 14{,}665 & \textbf{14{,}585} & 15{,}038 & 15{,}176 \\
 &  & Proportional & 13{,}340 & \textbf{13{,}299} & 30{,}706 & 82{,}825 & 122k \\
 & \multirow{2}{*}{Anchor-Based} & Uniform & \textbf{9{,}463} & 13{,}981 & 22{,}389 & 26{,}906 & 35{,}660 \\
 &  & Proportional & \textbf{10{,}697} & 16{,}254 & 33{,}221 & 46{,}464 & 62{,}482 \\
 & \multirow{2}{*}{Output Cluster} & Uniform & \textbf{9{,}463} & 10{,}223 & 36{,}539 & 61{,}045 & 44{,}465 \\
 &  & Proportional & \textbf{10{,}697} & 18{,}176 & 77{,}920 & 206k & 395k \\
\cmidrule(lr){1-8}
\multirow{10}{*}{\shortstack[l]{Post-Selection\\Probability}} & \multirow{2}{*}{Round-Robin} & Uniform & 18{,}027 & \textbf{15{,}236} & 15{,}276 & 15{,}405 & 15{,}378 \\
 &  & Proportional & \textbf{15{,}409} & 15{,}996 & 16{,}281 & 16{,}178 & 15{,}976 \\
 & \multirow{2}{*}{Weight Cluster} & Uniform & 18{,}027 & \textbf{10{,}976} & 13{,}436 & 15{,}552 & -- \\
 &  & Proportional & 15{,}409 & \textbf{14{,}076} & 16{,}900 & 20{,}448 & -- \\
 & \multirow{2}{*}{Router Cluster} & Uniform & 18{,}027 & \textbf{11{,}379} & 11{,}960 & 12{,}506 & 14{,}086 \\
 &  & Proportional & 15{,}409 & \textbf{8{,}850} & 13{,}105 & 16{,}681 & 24{,}100 \\
 & \multirow{2}{*}{Anchor-Based} & Uniform & 11{,}385 & 10{,}453 & \textbf{10{,}425} & 11{,}264 & 11{,}185 \\
 &  & Proportional & \textbf{12{,}269} & 12{,}946 & 13{,}889 & 14{,}828 & 14{,}871 \\
 & \multirow{2}{*}{Output Cluster} & Uniform & 11{,}385 & 10{,}914 & 16{,}675 & \textbf{8{,}509} & 25{,}862 \\
 &  & Proportional & \textbf{12{,}269} & 15{,}557 & 41{,}130 & 24{,}286 & 62{,}593 \\
\cmidrule(lr){1-8}
\multirow{10}{*}{\shortstack[l]{Conditional\\Probability}} & \multirow{2}{*}{Round-Robin} & Uniform & \textbf{4{,}280} & 20{,}580 & 311k & 6.3m & 21.8m \\
 &  & Proportional & 25{,}495 & \textbf{6{,}846} & 77{,}402 & 3.9m & 18.5m \\
 & \multirow{2}{*}{Weight Cluster} & Uniform & \textbf{4{,}280} & 10{,}836 & 37{,}955 & 18{,}543 & -- \\
 &  & Proportional & 25{,}495 & \textbf{11{,}730} & 35{,}458 & 250k & -- \\
 & \multirow{2}{*}{Router Cluster} & Uniform & \textbf{4{,}280} & 9{,}473 & 21{,}755 & 24{,}590 & 18{,}008 \\
 &  & Proportional & 25{,}495 & \textbf{5{,}068} & 18{,}650 & 317k & 1.6m \\
 & \multirow{2}{*}{Anchor-Based} & Uniform & \textbf{6{,}033} & 132k & 54{,}246 & 92{,}860 & 170k \\
 &  & Proportional & 31{,}189 & \textbf{25{,}609} & 32{,}057 & 147k & 348k \\
 & \multirow{2}{*}{Output Cluster} & Uniform & \textbf{6{,}033} & 6{,}085 & 12{,}998 & 55{,}860 & 186k \\
 &  & Proportional & 31{,}189 & \textbf{6{,}897} & 47{,}979 & 660k & 956k \\
\bottomrule
\end{tabular}
\end{table*}

\newpage

\begin{table*}[h]
\caption{WikiText-2 pre-distill perplexity ($\downarrow$) for activation-weighted and D-optimal scoring methods.  Each row shows one scoring $\times$ grouping $\times$ DP scaling combination across all $K$ values.  \textbf{Bold}: best $K$ per row.  Values $\geq$100k in compact notation.}
\label{tab:full_ppl_grid_diverse}
\centering
\footnotesize
\setlength{\tabcolsep}{4pt}
\renewcommand{\arraystretch}{1.15}
\begin{tabular}{ll l rrrrr}
\toprule
\textbf{Scoring} & \textbf{Grouping} & \textbf{DP Scaling} & $K{=}8$ & $K{=}16$ & $K{=}32$ & $K{=}64$ & $K{=}128$ \\
\midrule
\multirow{10}{*}{\shortstack[l]{Activation-Wtd\\Cond.\ Prob.}} & \multirow{2}{*}{Round-Robin} & Uniform & \textbf{6{,}334} & 13{,}080 & 67{,}926 & 417k & 1.1m \\
 &  & Proportional & 31{,}273 & 31{,}129 & \textbf{21{,}519} & 32{,}031 & 129k \\
 & \multirow{2}{*}{Weight Cluster} & Uniform & 6{,}334 & \textbf{5{,}527} & 11{,}223 & 30{,}095 & 16{,}821 \\
 &  & Proportional & 31{,}273 & \textbf{16{,}018} & 23{,}598 & 38{,}933 & 35{,}822 \\
 & \multirow{2}{*}{Router Cluster} & Uniform & 6{,}334 & \textbf{3{,}384} & 7{,}586 & 9{,}161 & 15{,}296 \\
 &  & Proportional & 31{,}273 & 10{,}450 & \textbf{9{,}248} & 19{,}192 & 203k \\
 & \multirow{2}{*}{Anchor-Based} & Uniform & 6{,}334 & \textbf{4{,}829} & 12{,}612 & 17{,}082 & 60{,}169 \\
 &  & Proportional & 31{,}273 & \textbf{7{,}273} & 10{,}824 & 20{,}683 & 65{,}647 \\
 & \multirow{2}{*}{Output Cluster} & Uniform & 6{,}334 & \textbf{2{,}002} & 12{,}079 & 26{,}600 & 151k \\
 &  & Proportional & 31{,}273 & \textbf{30{,}198} & 33{,}533 & 68{,}216 & 320k \\
\cmidrule(lr){1-8}
\multirow{10}{*}{\shortstack[l]{D-Optimal\\+ CP}} & \multirow{2}{*}{Round-Robin} & Uniform & \textbf{3{,}837} & 5{,}137 & 60{,}925 & 1.1m & 3.3m \\
 &  & Proportional & 10{,}928 & \textbf{4{,}897} & 65{,}655 & 1.1m & 3.3m \\
 & \multirow{2}{*}{Weight Cluster} & Uniform & \textbf{3{,}837} & 4{,}970 & 17{,}150 & 36{,}424 & 16{,}186 \\
 &  & Proportional & 10{,}928 & \textbf{4{,}013} & 41{,}808 & 260k & 463k \\
 & \multirow{2}{*}{Router Cluster} & Uniform & \textbf{3{,}837} & 4{,}964 & 12{,}789 & 11{,}815 & 15{,}709 \\
 &  & Proportional & 10{,}928 & \textbf{6{,}050} & 35{,}575 & 147k & 728k \\
 & \multirow{2}{*}{Anchor-Based} & Uniform & 3{,}837 & \textbf{2{,}840} & 5{,}420 & 12{,}823 & 52{,}050 \\
 &  & Proportional & 10{,}928 & \textbf{3{,}022} & 12{,}467 & 69{,}341 & 736k \\
 & \multirow{2}{*}{Output Cluster} & Uniform & \textbf{3{,}837} & 4{,}063 & 107k & 42{,}441 & 113k \\
 &  & Proportional & 10{,}928 & \textbf{3{,}805} & 370k & 333k & 920k \\
\cmidrule(lr){1-8}
\multirow{10}{*}{\shortstack[l]{D-Optimal\\+ ACP}} & \multirow{2}{*}{Round-Robin} & Uniform & \textbf{5{,}134} & 10{,}572 & 66{,}588 & 1.0m & 2.8m \\
 &  & Proportional & 11{,}949 & \textbf{10{,}945} & 72{,}395 & 992k & 2.8m \\
 & \multirow{2}{*}{Weight Cluster} & Uniform & 5{,}134 & \textbf{4{,}779} & 8{,}632 & 28{,}519 & 16{,}340 \\
 &  & Proportional & 11{,}949 & \textbf{4{,}333} & 24{,}579 & 208k & 431k \\
 & \multirow{2}{*}{Router Cluster} & Uniform & \textbf{5{,}134} & 8{,}553 & 8{,}751 & 9{,}410 & 15{,}635 \\
 &  & Proportional & 11{,}949 & \textbf{11{,}352} & 34{,}707 & 109k & 714k \\
 & \multirow{2}{*}{Anchor-Based} & Uniform & 5{,}134 & \textbf{4{,}776} & 15{,}564 & 18{,}842 & 50{,}751 \\
 &  & Proportional & 11{,}949 & \textbf{3{,}939} & 16{,}248 & 78{,}899 & 743k \\
 & \multirow{2}{*}{Output Cluster} & Uniform & \textbf{5{,}134} & 5{,}521 & 22{,}932 & 102k & 97{,}865 \\
 &  & Proportional & 11{,}949 & \textbf{6{,}230} & 313k & 552k & 912k \\
\bottomrule
\end{tabular}
\end{table*}

\newpage
\section{Full distillation results}
\label{app:full_distillation}

Table~\ref{tab:distillation} presents the full per-benchmark distillation results for all 35 scoring$\times$grouping combinations on Qwen3-30B-A3B.

\begin{table*}[h]
\caption{Full distillation results (0.3B tokens) for all 35 configurations on Qwen3-30B-A3B.  \emph{Pre}/\emph{Post} = WikiText-2 PPL before/after distillation.  Color scale: \colorbox[RGB]{166,199,166}{\scriptsize best} $\to$ \colorbox[RGB]{247,240,222}{\scriptsize mid} $\to$ \colorbox[RGB]{244,203,195}{\scriptsize worst}.  \textbf{Bold}: best average per scoring method.}
\label{tab:distillation}
\centering
\footnotesize
\setlength{\tabcolsep}{2.5pt}
\begin{tabular}{lll c cc cccccc}
\toprule
 & & & & \multicolumn{2}{c}{\textbf{PPL} ($\downarrow$)} & \multicolumn{6}{c}{\textbf{Downstream Accuracy (\%)}} \\
\cmidrule(lr){5-6} \cmidrule(lr){7-12}
\textbf{Scoring} & \textbf{Grouping} & \textbf{DP Scaling} & $K$ & \textbf{Pre} & \textbf{Post} & \textbf{Wino} & \textbf{Hella} & \textbf{ARC-E} & \textbf{ARC-C} & \textbf{MMLU} & \textbf{Avg} \\
\midrule
\multirow{5}{*}{\shortstack[l]{Selection\\Frequency}} & Round-robin & Proportional & 16 & 19{,}440 & \cellcolor[RGB]{245,213,202} 29.3 & 53.04 & 32.58 & 49.07 & 25.51 & 27.62 & \cellcolor[RGB]{245,216,205} 37.56 \\
 & Weight cluster & Uniform & 16 & 12{,}896 & \cellcolor[RGB]{246,222,209} 27.9 & 54.38 & 32.65 & 49.49 & 24.40 & 27.16 & \cellcolor[RGB]{245,217,205} \textbf{37.62} \\
 & Router cluster & Uniform & 16 & 11{,}342 & \cellcolor[RGB]{245,218,206} 28.5 & 52.09 & 32.99 & 49.62 & 24.49 & 27.08 & \cellcolor[RGB]{245,213,202} 37.25 \\
 & Anchor-based & Proportional & 8 & 11{,}668 & \cellcolor[RGB]{245,214,203} 29.2 & 52.17 & 31.64 & 47.14 & 23.12 & 27.50 & \cellcolor[RGB]{244,203,195} 36.31 \\
 & Output cluster & Uniform & 64 & 8{,}638 & \cellcolor[RGB]{246,223,210} 27.7 & 53.12 & 33.06 & 49.33 & 24.74 & 27.34 & \cellcolor[RGB]{245,216,204} 37.52 \\
\midrule
\multirow{5}{*}{\shortstack[l]{Pre-Selection\\Probability}} & Round-robin & Proportional & 8 & 13{,}340 & \cellcolor[RGB]{244,203,195} 30.9 & 53.12 & 31.85 & 48.53 & 23.89 & 27.57 & \cellcolor[RGB]{245,210,200} 36.99 \\
 & Weight cluster & Uniform & 16 & 11{,}085 & \cellcolor[RGB]{244,208,199} 30.1 & 53.43 & 32.21 & 48.74 & 23.63 & 27.59 & \cellcolor[RGB]{245,211,201} \textbf{37.12} \\
 & Router cluster & Proportional & 16 & 13{,}299 & \cellcolor[RGB]{245,209,200} 29.9 & 52.41 & 32.19 & 47.77 & 23.98 & 27.22 & \cellcolor[RGB]{244,207,198} 36.71 \\
 & Anchor-based & Uniform & 8 & 9{,}463 & \cellcolor[RGB]{245,217,205} 28.6 & 53.35 & 32.17 & 48.65 & 24.15 & 27.28 & \cellcolor[RGB]{245,211,201} \textbf{37.12} \\
 & Output cluster & Uniform & 8 & 9{,}463 & \cellcolor[RGB]{245,217,205} 28.7 & 52.88 & 32.14 & 47.47 & 23.63 & 27.70 & \cellcolor[RGB]{244,208,198} 36.77 \\
\midrule
\multirow{5}{*}{\shortstack[l]{Post-Selection\\Probability}} & Round-robin & Uniform & 16 & 15{,}236 & \cellcolor[RGB]{245,214,203} 29.2 & 51.78 & 32.51 & 48.61 & 23.98 & 27.28 & \cellcolor[RGB]{244,208,199} 36.83 \\
 & Weight cluster & Uniform & 16 & 10{,}976 & \cellcolor[RGB]{246,223,210} 27.7 & 52.64 & 32.84 & 49.79 & 24.66 & 27.07 & \cellcolor[RGB]{245,214,203} \textbf{37.40} \\
 & Router cluster & Proportional & 16 & 8{,}850 & \cellcolor[RGB]{245,216,205} 28.8 & 51.85 & 32.68 & 49.07 & 25.09 & 27.33 & \cellcolor[RGB]{245,212,202} 37.21 \\
 & Anchor-based & Uniform & 32 & 10{,}425 & \cellcolor[RGB]{245,217,205} 28.6 & 51.62 & 31.95 & 47.60 & 23.89 & 27.48 & \cellcolor[RGB]{244,205,197} 36.51 \\
 & Output cluster & Uniform & 64 & 8{,}509 & \cellcolor[RGB]{245,220,207} 28.2 & 52.49 & 32.45 & 50.04 & 24.40 & 27.30 & \cellcolor[RGB]{245,214,203} 37.34 \\
\midrule
\multirow{5}{*}{\shortstack[l]{Conditional\\Probability}} & Round-robin & Proportional & 16 & 3{,}696 & \cellcolor[RGB]{240,236,217} 24.4 & 53.83 & 36.18 & 56.36 & 27.82 & 27.60 & \cellcolor[RGB]{236,234,214} 40.35 \\
 & Weight cluster & Proportional & 16 & 2{,}148 & \cellcolor[RGB]{238,236,216} 24.3 & 54.46 & 36.02 & 56.44 & 27.73 & 27.50 & \cellcolor[RGB]{234,233,213} \textbf{40.43} \\
 & Router cluster & Proportional & 16 & 2{,}264 & \cellcolor[RGB]{225,229,206} 23.3 & 53.67 & 36.06 & 56.69 & 27.39 & 27.77 & \cellcolor[RGB]{237,235,215} 40.32 \\
 & Anchor-based & Uniform & 8 & 6{,}033 & \cellcolor[RGB]{200,216,190} 21.5 & 53.67 & 35.73 & 55.51 & 26.96 & 27.11 & \cellcolor[RGB]{247,239,222} 39.80 \\
 & Output cluster & Uniform & 8 & 6{,}033 & \cellcolor[RGB]{200,216,190} 21.5 & 53.83 & 35.73 & 55.72 & 27.82 & 26.80 & \cellcolor[RGB]{244,239,220} 39.98 \\
\midrule
\multirow{5}{*}{\shortstack[l]{Activation-Wtd\\Cond.\ Prob.}} & Round-robin & Uniform & 8 & 6{,}334 & \cellcolor[RGB]{167,200,167} 19.1 & 56.35 & 39.78 & 56.02 & 29.69 & 30.73 & \cellcolor[RGB]{186,209,180} \textbf{42.52} \\
 & Weight cluster & Uniform & 16 & 5{,}527 & \cellcolor[RGB]{197,215,188} 21.3 & 55.01 & 38.67 & 53.83 & 29.61 & 28.91 & \cellcolor[RGB]{216,224,201} 41.21 \\
 & Router cluster & Uniform & 16 & 3{,}384 & \cellcolor[RGB]{197,215,188} 21.3 & 54.54 & 36.95 & 53.11 & 25.94 & 28.24 & \cellcolor[RGB]{247,239,221} 39.76 \\
 & Anchor-based & Uniform & 16 & 4{,}829 & \cellcolor[RGB]{203,218,191} 21.7 & 52.01 & 37.30 & 53.41 & 26.88 & 27.64 & \cellcolor[RGB]{247,236,219} 39.45 \\
 & Output cluster & Uniform & 16 & 2{,}002 & \cellcolor[RGB]{186,209,180} 20.5 & 54.78 & 38.11 & 55.05 & 27.47 & 27.11 & \cellcolor[RGB]{232,233,212} 40.50 \\
\midrule
\multirow{5}{*}{\shortstack[l]{D-Optimal\\+ CP}} & Round-robin & Uniform & 8 & 3{,}837 & \cellcolor[RGB]{167,200,167} 19.1 & 54.46 & 41.30 & 57.37 & 30.12 & 31.19 & \cellcolor[RGB]{178,205,174} 42.89 \\
 & Weight cluster & Uniform & 8 & 3{,}837 & \cellcolor[RGB]{166,199,166} 19.0 & 55.56 & 41.13 & 57.53 & 29.61 & 31.06 & \cellcolor[RGB]{176,204,173} \textbf{42.98} \\
 & Router cluster & Uniform & 8 & 3{,}837 & \cellcolor[RGB]{166,199,166} 19.0 & 55.25 & 41.39 & 57.41 & 29.61 & 31.16 & \cellcolor[RGB]{176,204,173} 42.96 \\
 & Anchor-based & Uniform & 16 & 2{,}840 & \cellcolor[RGB]{189,211,182} 20.7 & 55.25 & 38.93 & 55.35 & 28.67 & 27.45 & \cellcolor[RGB]{218,225,202} 41.13 \\
 & Output cluster & Proportional & 16 & 3{,}805 & \cellcolor[RGB]{186,209,180} 20.5 & 56.75 & 38.96 & 55.89 & 29.86 & 28.74 & \cellcolor[RGB]{197,215,188} 42.04 \\
\midrule
\multirow{5}{*}{\shortstack[l]{D-Optimal\\+ ACP}} & Round-robin & Uniform & 8 & 5{,}134 & \cellcolor[RGB]{169,200,168} 19.2 & 56.99 & 41.13 & 57.37 & 29.86 & 31.70 & \cellcolor[RGB]{166,199,166} \textbf{43.41} \\
 & Weight cluster & Proportional & 16 & 4{,}333 & \cellcolor[RGB]{188,210,181} 20.6 & 54.22 & 39.59 & 54.67 & 28.24 & 28.91 & \cellcolor[RGB]{218,225,202} 41.13 \\
 & Router cluster & Uniform & 8 & 5{,}134 & \cellcolor[RGB]{169,200,168} 19.2 & 56.59 & 40.99 & 57.45 & 29.18 & 32.30 & \cellcolor[RGB]{169,200,168} 43.30 \\
 & Anchor-based & Proportional & 16 & 3{,}939 & \cellcolor[RGB]{177,205,174} 19.8 & 56.12 & 40.20 & 55.77 & 28.92 & 29.59 & \cellcolor[RGB]{195,214,186} 42.12 \\
 & Output cluster & Uniform & 8 & 5{,}134 & \cellcolor[RGB]{169,200,168} 19.2 & 56.99 & 41.08 & 57.07 & 29.27 & 31.86 & \cellcolor[RGB]{170,201,169} 43.25 \\
\bottomrule
\end{tabular}
\end{table*}

\clearpage
\section{Dense-to-Dense (D2D) pruning baseline}
\label{app:d2d}

We implement a D2D baseline following the Minitron methodology~\citep{muralidharan2024compact}: structured pruning of a dense teacher (Qwen3-32B, 32B parameters) to a student of comparable size to our MoE-to-dense models, followed by distillation with the same dense teacher using matched hyperparameters and token budget.

\paragraph{Architecture search.}
Following the Minitron approach~\citep{muralidharan2024compact}, we search over five candidate architectures at matched parameter count (${\sim}$3.4B), varying pruning strategy (width-only vs.\ combined width+depth), number of layers, and hidden/FFN dimensions.  All candidates are pruned using activation-based importance scoring~\citep{muralidharan2024compact} calibrated on 1{,}024 WikiText-103 samples and evaluated by WikiText-2 perplexity before distillation (Table~\ref{tab:d2d_configs}).

\begin{table}[h]
\caption{D2D architecture search: five candidate architectures pruned from Qwen3-32B (64 layers, $d{=}5120$, $d_{\text{dense}}{=}25600$).  Width-only pruning that preserves all layers achieves the best pre-distill PPL.  \textbf{Bold}: selected configuration.}
\label{tab:d2d_configs}
\centering
\small
\begin{tabular}{lcccccr}
\toprule
\textbf{Strategy} & \textbf{Layers} & $d$ & $d_{\text{dense}}$ & \textbf{Heads} & \textbf{Params} & \textbf{Pre-distill PPL} \\
\midrule
\rowcolor{palGood!10}
\textbf{Width-only} & \textbf{64} & \textbf{2{,}048} & \textbf{6{,}144} & \textbf{8} & \textbf{3.44B} & \textbf{15{,}300} \\
Width-only & 64 & 1{,}536 & 8{,}192 & 12 & 3.29B & 44{,}605 \\
Combined & 56 & 2{,}048 & 6{,}144 & 16 & 3.44B & 20{,}299 \\
Combined & 48 & 2{,}048 & 6{,}144 & 16 & 3.44B & 44{,}100 \\
Combined & 40 & 2{,}048 & 6{,}144 & 16 & 2.55B & 286{,}698 \\
\bottomrule
\end{tabular}
\end{table}

The best architecture preserves all 64 layers with aggressive width pruning ($d{=}2048$, $d_{\text{dense}}{=}6144$, 3.44B parameters), achieving a pre-distill PPL of 15{,}300.  Removing layers consistently degrades quality: even retaining 56 of 64 layers substantially increases PPL, suggesting that depth preservation is critical for structured pruning at this compression ratio.

\paragraph{Distillation.}
The selected D2D student is distilled with its dense teacher (Qwen3-32B) using identical hyperparameters, data, and token budget as all MoE-to-dense experiments.  With this setup, D2D achieves only 33.28\% average accuracy, barely above the random FFN baseline (32.70\%, $+$0.6~pp) and far below even the weakest MoE-to-dense configuration (SF$\times$AB, 36.31\%, $+$3.0~pp).  This suggests that one-shot dense pruning at this compression ratio provides little structural advantage, while MoE-to-dense conversion preserves expert-level structure as a stronger initialization for distillation.

\paragraph{Iterative D2D pruning.}

Minitron recommends obtaining smaller models through repeated pruning and distillation~\citep{muralidharan2024compact}, so we also replace one-shot pruning with iterative prune-and-distill.  Each schedule prunes the original Qwen3-32B down to ${\sim}$3.5B through repeated prune-and-distill stages (Table~\ref{tab:iterative_d2d}).  The total token budget is held fixed across schedules and matched to all other 0.3B-token experiments.  At each pruning stage, we follow Minitron and probe six candidate architectures over commonly used dimensions, keeping the one with the lowest perplexity.  At every distillation stage, the original Qwen3-32B is the teacher.

\begin{table}[H]
\caption{Iterative D2D schedules at a matched 0.3B-token budget.  Iterative pruning improves on one-shot D2D by up to $+$6.9~pp, but DO-ACP leads the strongest schedule by $+$3.2~pp and on every benchmark.}
\label{tab:iterative_d2d}
\centering
\footnotesize
\setlength{\tabcolsep}{4pt}
\begin{tabular}{l l ccccc c}
\toprule
\textbf{Method} & \textbf{Ladder / token share} & \textbf{Wino} & \textbf{Hella} & \textbf{ARC-E} & \textbf{ARC-C} & \textbf{MMLU} & \textbf{Avg (\%)} \\
\midrule
DO-ACP (ours) & -- & \textbf{57.0} & \textbf{41.1} & \textbf{57.4} & \textbf{29.9} & \textbf{31.7} & \textbf{43.41} \\
Iterative D2D & 8B$\to$5B$\to$3.5B, 17/33/50\% & 52.9 & 36.8 & 55.6 & 28.7 & 27.0 & 40.20 \\
Iterative D2D & 8B$\to$5B$\to$3.5B, 33/33/33\% & 52.6 & 36.6 & 55.0 & 28.0 & 27.0 & 39.84 \\
Iterative D2D & 8B$\to$6B$\to$4.5B$\to$3.5B, 25/25/25/25\% & 54.2 & 36.7 & 53.7 & 28.1 & 26.4 & 39.82 \\
One-shot D2D & -- & 51.7 & 27.7 & 38.9 & 22.4 & 25.7 & 33.28 \\
\bottomrule
\end{tabular}
\end{table}

We also calibrate both methods on a corpus closer to the target distribution, built from the train splits of ARC, HellaSwag, MMLU, and Winogrande rendered as few-shot-style prompts.  This corpus is used for expert scoring in MoE-to-dense and for architecture selection in the three-stage iterative D2D baseline.  It benefits MoE-to-dense more ($+$1.60~pp) than iterative D2D ($+$0.16~pp), widening the gap from 3.57~pp to 5.01~pp (Table~\ref{tab:d2d_calibration}), so calibration data does not explain the D2D deficit.

\begin{table}[H]
\caption{MoE-to-dense and iterative D2D under WikiText and benchmark-train calibration (0.3B tokens).}
\label{tab:d2d_calibration}
\centering
\small
\begin{tabular}{l ccccc c}
\toprule
\textbf{Configuration} & \textbf{Wino} & \textbf{Hella} & \textbf{ARC-E} & \textbf{ARC-C} & \textbf{MMLU} & \textbf{Avg (\%)} \\
\midrule
DO-ACP (WikiText) & 57.0 & 41.1 & 57.4 & 29.9 & 31.7 & 43.41 \\
DO-ACP (benchmark-train) & 56.4 & 43.1 & 59.3 & 31.7 & 34.7 & 45.01 \\
Iterative D2D (WikiText) & 52.6 & 36.6 & 55.0 & 28.0 & 27.0 & 39.84 \\
Iterative D2D (benchmark-train) & 52.6 & 37.4 & 54.8 & 27.9 & 27.3 & 40.00 \\
\bottomrule
\end{tabular}
\end{table}

\paragraph{Scope of the comparison.}
The two teachers are comparable in total parameter count (30.5B vs.\ 32.8B) but differ in depth and width.  Qwen3-30B-A3B has 48 layers of width 2048, while Qwen3-32B has 64 layers of width 5120.  Our student inherits the MoE teacher's geometry and attention layers, which is a structural advantage of MoE parents.  The random FFN baseline (Table~\ref{tab:baselines}) shares this advantage, and with it held constant, expert selection contributes $+$10.7~pp.

\section{Broader benchmark evaluation}
\label{app:extended_eval}

Since fixed experts could lose domain-specific capability, we evaluate the four extended-training checkpoints (Section~\ref{sec:scaling_results}) on code, math, instruction following, long context, multilingual, and agentic tasks (Table~\ref{tab:extended_benchmarks}).  As these are base models, we rely primarily on bits-per-byte (BPB), the prefill-only negative log-likelihood of held-out gold responses normalized by UTF-8 bytes, which remains informative even when accuracy is near chance, together with few-shot and zero-shot accuracy where tasks remain discriminative.  DO-ACP leads on seven of the nine measures, and D2D is marginally lower on the other two (HumanEval and multilingual BPB).  The advantage over SF is also larger than on the core suite, with $+$9.1~pp on RULER, $+$8.8~pp on Belebele, and 5$\times$ on GSM8K, compared with $+$4.6~pp on the core average.

\begin{table}[H]
\caption{Broader benchmarks after ${\sim}$4B-token distillation.  BPB columns are bits-per-byte ($\downarrow$ better); GSM8K, RULER, and Belebele are accuracy in \% ($\uparrow$ better).  \textbf{Bold}: best per column.}
\label{tab:extended_benchmarks}
\centering
\small
\setlength{\tabcolsep}{3pt}
\begin{tabular}{l cc cc c c cc c}
\toprule
 & \multicolumn{2}{c}{\textbf{Code BPB}} & \multicolumn{2}{c}{\textbf{Math}} & \textbf{Instr.} & \textbf{Long-ctx} & \multicolumn{2}{c}{\textbf{Multilingual}} & \textbf{Agentic} \\
\cmidrule(lr){2-3} \cmidrule(lr){4-5} \cmidrule(lr){8-9}
\textbf{Model} & \textbf{HumanEval} & \textbf{MBPP} & \textbf{GSM8K} & \textbf{BPB} & \textbf{BPB} & \textbf{RULER} & \textbf{Belebele} & \textbf{BPB} & \textbf{BPB} \\
\midrule
DO-ACP & 0.703 & \textbf{1.139} & \textbf{12.05} & \textbf{0.856} & \textbf{0.791} & \textbf{89.94} & \textbf{35.01} & 1.323 & \textbf{0.780} \\
SF & 0.837 & 1.441 & 2.43 & 0.980 & 0.809 & 80.86 & 26.19 & 1.334 & 0.859 \\
D2D & \textbf{0.695} & 1.404 & 2.43 & 0.945 & 0.809 & 80.09 & 25.99 & \textbf{1.322} & 0.817 \\
Random FFN & 1.039 & 1.621 & 1.29 & 1.118 & 0.884 & 63.66 & 24.60 & 1.476 & 1.122 \\
\bottomrule
\end{tabular}
\end{table}

\paragraph{Setup for Table~\ref{tab:extended_benchmarks}.}
BPB corpora are HumanEval~\citep{chen2021evaluating} and MBPP~\citep{austin2021program} reference solutions, GSM8K~\citep{cobbe2021training} and MATH~\citep{hendrycks2021measuring} reference solutions, No Robots instruction responses, Glaive function-calling traces, and FLORES~\citep{goyal2022flores} passages in 10 languages, using up to 300 samples each at 2048 tokens.  Math BPB is the mean over GSM8K and MATH.  Accuracy metrics are GSM8K 5-shot exact match, Belebele~\citep{bandarkar2024belebele} 0-shot over the same 10 languages, and RULER~\citep{hsieh2024ruler} at 4k context averaged over 11 self-contained subtasks (8 needle-in-a-haystack variants plus variable tracking, common-word extraction, and frequent-word extraction).

\paragraph{Science benchmarks.}
We also evaluate scientific knowledge on SciQ~\citep{welbl2017crowdsourcing}, OpenBookQA~\citep{mihaylov2018can}, and MMLU-STEM (Table~\ref{tab:science}).  DO-ACP performs best among the students on every benchmark.  MMLU-STEM shows the clearest separation, with DO-ACP $+$10.0~pp over SF and $+$13.3~pp over D2D.

\begin{table}[H]
\caption{Science benchmarks after ${\sim}$4B-token distillation (accuracy, \%).}
\label{tab:science}
\centering
\small
\begin{tabular}{l ccc}
\toprule
\textbf{Model} & \textbf{SciQ 0-shot} & \textbf{OpenBookQA 0-shot} & \textbf{MMLU-STEM 5-shot} \\
\midrule
DO-ACP & \textbf{93.6} & \textbf{40.2} & \textbf{40.1} \\
SF & 93.0 & 39.2 & 30.1 \\
D2D & 92.0 & 37.0 & 26.8 \\
Random FFN & 88.7 & 33.8 & 28.4 \\
\bottomrule
\end{tabular}
\end{table}

\section{Deployment measurements}
\label{app:deployment}

We measure resident memory, peak VRAM, throughput, and latency on a single H200 GPU, using the vLLM engine~\citep{kwon2023efficient} for throughput and latency.  TTFT is the time to first token (prefill) and TPOT is the time per output token (decode).  $b$ denotes batch size and $s$ prompt length.  The dense student has lower weight memory, higher throughput, and faster decode than its MoE parent in every setting measured (Tables~\ref{tab:deployment_memory} and~\ref{tab:deployment_throughput}).

Relative to its MoE parent, the student has 9.14$\times$ fewer parameters and uses 3.0$\times$ less peak VRAM, with 1.46$\times$ higher throughput and 1.75$\times$ faster decode at batch 32$\times$4096.  Compared with the D2D student, which has 3\% more parameters, it uses 12.1~GiB less peak VRAM and has 31\% lower decode latency, which we attribute to its narrower KV cache (4 vs.\ 8 key-value heads).

\begin{table}[H]
\caption{Weight memory and peak VRAM.}
\label{tab:deployment_memory}
\centering
\footnotesize
\begin{tabular}{l c cc ccc}
\toprule
 & & & & \multicolumn{3}{c}{\textbf{Peak VRAM (GiB)}} \\
\cmidrule(lr){5-7}
\textbf{Model} & \textbf{Params} & \textbf{Weights (GiB)} & \textbf{vs.\ ours} & $b1,s128$ & $b32,s2048$ & $b32,s4096$ \\
\midrule
MoE teacher (Qwen3-30B-A3B) & 30.532B & 56.9 & 9.14$\times$ & 56.9 & 66.0 & 75.2 \\
Dense student (ours) & 3.340B & 6.2 & 1.00$\times$ & 6.3 & 15.5 & 24.8 \\
D2D teacher (Qwen3-32B) & 32.762B & 61.0 & 9.81$\times$ & 61.1 & 88.9 & 116.8 \\
D2D student & 3.441B & 6.4 & 1.03$\times$ & 6.5 & 21.7 & 36.9 \\
\bottomrule
\end{tabular}
\end{table}

\begin{table}[H]
\caption{Throughput (tokens/s) by batch size and latency at batch 32, by prompt length.}
\label{tab:deployment_throughput}
\centering
\small
\begin{tabular}{l l rrrr rr}
\toprule
\textbf{Prompt} & \textbf{Model} & $b1$ & $b4$ & $b16$ & $b32$ & \textbf{TTFT (ms)} & \textbf{TPOT (ms)} \\
\midrule
\multirow{4}{*}{128} & MoE teacher & 217 & 686 & 1{,}618 & 2{,}712 & 80 & 11.26 \\
 & Dense student (ours) & 279 & 1{,}078 & 4{,}067 & 7{,}501 & 51 & 3.90 \\
 & D2D teacher & 54 & 210 & 790 & 1{,}416 & 410 & 19.55 \\
 & D2D student & 272 & 1{,}056 & 3{,}942 & 6{,}832 & 49 & 4.33 \\
\midrule
\multirow{4}{*}{2048} & MoE teacher & 202 & 584 & 1{,}226 & 1{,}752 & 1{,}000 & 10.53 \\
 & Dense student (ours) & 261 & 879 & 2{,}155 & 2{,}893 & 742 & 5.31 \\
 & D2D teacher & 50 & 157 & 349 & 433 & 6{,}524 & 23.12 \\
 & D2D student & 252 & 818 & 1{,}914 & 2{,}468 & 769 & 7.01 \\
\midrule
\multirow{4}{*}{4096} & MoE teacher & 190 & 494 & 955 & 1{,}118 & 2{,}160 & 11.83 \\
 & Dense student (ours) & 242 & 713 & 1{,}372 & 1{,}630 & 1{,}653 & 6.77 \\
 & D2D teacher & 46 & 122 & 212 & 241 & 13{,}607 & 26.92 \\
 & D2D student & 233 & 652 & 1{,}193 & 1{,}377 & 1{,}723 & 9.86 \\
\bottomrule
\end{tabular}
\end{table}

\clearpage
\section{Error taxonomy and examples}
\label{app:error_taxonomy}

Table~\ref{tab:error_taxonomy} defines the six error categories used in the qualitative analysis (Section~\ref{sec:qualitative}).  Surface-level failures (incoherent, repetitive loop, other) are classified by rule-based heuristics applied in priority order; semantic errors (knowledge error, reasoning error) are classified by LLM-as-a-judge using Claude Opus 4.6.  Table~\ref{tab:error_examples} shows representative examples from each category.

\begin{table}[h]
\caption{Error taxonomy for qualitative MMLU analysis.  Categories are applied in priority order (top to bottom): a response matching an earlier category is not checked against later ones.}
\label{tab:error_taxonomy}
\centering
\footnotesize
\begin{tabular}{lp{10cm}}
\toprule
\textbf{Category} & \textbf{Definition} \\
\midrule
Incoherent & Nonsensical output that fails to convey meaning: circular redefinitions, meta-commentary without substance, or degenerate numeric sequences.  Detected by filler-word ratio ${>}35\%$ and unique-word ratio ${<}15\%$, or digit ratio ${>}30\%$. \\
Repetitive Loop & The same phrases or sentences cycle repeatedly without making progress toward an answer.  Detected by 3-gram repetition count $\geq 5$. \\
Knowledge Error & The reasoning structure is coherent and on-topic, but the response contains factual errors (e.g., hallucinated definitions, incorrect attributions). \\
Reasoning Error & The response attempts logical or mathematical reasoning but arrives at an incorrect conclusion through flawed logic. \\
Other & Topic drift (coherent but off-topic, keyword overlap ${<}15\%$), truncation without a final answer, or out-of-range answer selection. \\
\bottomrule
\end{tabular}
\end{table}

\begin{table}[h]
\caption{Representative error examples from the qualitative MMLU analysis.}
\label{tab:error_examples}
\centering
\footnotesize
\setlength{\tabcolsep}{3pt}
\begin{tabular}{llp{9cm}}
\toprule
\textbf{Category} & \textbf{Model} & \textbf{Example (excerpt)} \\
\midrule
Incoherent & Random FFN & ``The duty is a principle, but the principle is a norm. The duty is a moral obligation, but the norm is a standard of conduct\ldots'' (circular redefinition, no content) \\
\midrule
Incoherent & SF & ``We need to find the correct answer. First, the first step is to check the list of options\ldots the next question is: what is the correct fact?'' (meta-commentary about answering without engaging the question) \\
\midrule
Repetitive Loop & D2D & ``Jaden's score is a multiple of 7. The number of points Jaden scored is less than 45. So, Jaden's score is a multiple of 7\ldots'' (same two premises repeated with no computation) \\
\midrule
Knowledge Error & D2D & ``Miaphystism is the doctrine that God is present in the Eucharist\ldots m\={\i}aphysos is a word that can be translated as disease\ldots'' (hallucinated definition and fabricated etymology) \\
\midrule
Reasoning Error & DO-ACP & ``1/2 exponent is the smallest\ldots 1/2 is the fastest.'' Self-contradiction within two sentences; growth rate ordering reversed. \\
\midrule
Other & Random FFN & Four answer choices paraphrased at equal length; response truncated mid-sentence without selecting A--D. \\
\bottomrule
\end{tabular}
\end{table}

\clearpage
\section{Cross-model full results}
\label{app:cross_model_tables}

Tables~\ref{tab:cross_model} and~\ref{tab:gptoss} provide per-benchmark results for all cross-model configurations.

\begin{table}[h]
\caption{DeepSeek-V2-Lite cross-model validation (16B $\to$ dense, 0.3B-token distillation).  Rows sorted by scoring method; both $K{=}6$ (pure pruning) and $K{=}12$ (merging) shown.  \textbf{Bold}: best per column.}
\label{tab:cross_model}
\centering
\footnotesize
\begin{tabular}{ll ccccc c}
\toprule
\textbf{Scoring} & $K$ & \textbf{Wino} & \textbf{Hella} & \textbf{ARC-E} & \textbf{ARC-C} & \textbf{MMLU} & \textbf{Avg (\%)} \\
\midrule
SF     & 6  & 54.9 & 38.9 & 52.4 & 27.1 & 26.9 & 40.04 \\
SF     & 12 & 56.9 & 41.1 & \textbf{54.8} & \textbf{28.4} & 24.6 & 41.16 \\
\midrule
CP     & 6  & 53.0 & 36.9 & 49.4 & 25.7 & 25.3 & 38.07 \\
CP     & 12 & 55.6 & 40.2 & 53.5 & 26.6 & 26.8 & 40.53 \\
\midrule
ACP    & 6  & 56.8 & 38.6 & 51.0 & 27.5 & 28.1 & 40.37 \\
ACP    & 12 & 57.1 & 40.9 & 52.9 & 27.4 & 26.4 & 40.93 \\
\midrule
DO-ACP & 6  & \textbf{60.3} & 41.0 & 53.7 & 28.2 & \textbf{28.7} & \textbf{42.39} \\
DO-ACP & 12 & 59.0 & \textbf{41.5} & 51.7 & 26.2 & 26.9 & 41.07 \\
\midrule
\multicolumn{2}{l}{Random FFN + teacher attn} & 50.6 & 25.6 & 30.6 & 20.9 & 23.6 & 30.25 \\
\multicolumn{2}{l}{Random initialization} & 50.1 & 25.4 & 28.8 & 24.1 & 22.9 & 30.27 \\
\midrule
\rowcolor{gray!10}
\multicolumn{2}{l}{Teacher (DeepSeek-V2-Lite)} & 76.2 & 80.5 & 84.4 & 56.3 & 58.0 & 71.09 \\
\bottomrule
\end{tabular}
\end{table}

\begin{table}[h]
\caption{GPT-OSS-20B cross-model validation (21B $\to$ dense, 0.3B-token distillation).  Both $K{=}4$ (pure pruning) and $K{=}8$ (merging) shown.  \textbf{Bold}: best per column.  $^\dagger$Post-trained reasoning model evaluated in completion mode; native-format performance is expected to be substantially higher.}
\label{tab:gptoss}
\centering
\footnotesize
\begin{tabular}{ll ccccc c}
\toprule
\textbf{Scoring} & $K$ & \textbf{Wino} & \textbf{Hella} & \textbf{ARC-E} & \textbf{ARC-C} & \textbf{MMLU} & \textbf{Avg (\%)} \\
\midrule
SF     & 4 & 51.6 & 29.4 & 34.0 & 22.4 & 23.3 & 32.15 \\
SF     & 8 & 51.6 & 29.3 & 33.3 & 21.7 & 22.8 & 31.72 \\
\midrule
CP     & 4 & 50.9 & 30.4 & \textbf{36.4} & 23.0 & 23.6 & 32.86 \\
CP     & 8 & 49.7 & 29.3 & 32.2 & 23.2 & 23.1 & 31.49 \\
\midrule
ACP    & 4 & 53.0 & 31.9 & 35.6 & 23.0 & 23.3 & 33.36 \\
ACP    & 8 & \textbf{53.1} & 30.5 & 33.6 & 23.3 & \textbf{23.7} & 32.82 \\
\midrule
DO-ACP & 4 & 53.0 & \textbf{32.1} & 36.7 & 23.2 & \textbf{23.7} & \textbf{33.71} \\
DO-ACP & 8 & 51.3 & 29.9 & 33.5 & 22.5 & 23.3 & 32.11 \\
\midrule
\multicolumn{2}{l}{Random FFN + teacher attn} & 50.2 & 27.5 & 28.5 & 23.0 & 23.2 & 30.46 \\
\multicolumn{2}{l}{Random initialization} & 50.0 & 26.0 & 25.8 & \textbf{25.3} & 23.0 & 30.02 \\
\midrule
\rowcolor{gray!10}
\multicolumn{2}{l}{Teacher (GPT-OSS-20B)$^\dagger$} & 59.3 & 39.9 & 80.9 & 53.7 & 49.6 & 56.67 \\
\bottomrule
\end{tabular}
\end{table}

\newpage
\section{Model architecture details}
\label{app:architecture}

Table~\ref{tab:architecture} summarizes the teacher and student architectures for all three models used in our experiments.

\begin{table}[h]
\caption{Architecture of each MoE teacher and the corresponding dense student produced by our pipeline.  $^\dagger$DeepSeek-V2-Lite has 2 shared experts per MoE layer; layer~0 is a standard dense FFN (zero-padded to match $d_{\text{dense}}$).}
\label{tab:architecture}
\centering
\small
\begin{tabular}{lccc}
\toprule
\textbf{Property} & \textbf{Qwen3-30B-A3B} & \textbf{DeepSeek-V2-Lite} & \textbf{GPT-OSS-20B} \\
\midrule
\multicolumn{4}{l}{\textit{Teacher}} \\
\midrule
Total parameters       & 30.5B      & 15.7B      & 20.9B \\
Active parameters      & 3.3B       & 2.4B       & 3.6B \\
Hidden dimension ($d$) & 2{,}048    & 2{,}048    & 2{,}880 \\
Routed experts / layer ($E$) & 128  & 64$^\dagger$ & 32 \\
Active routed experts ($k$) & 8    & 6          & 4 \\
Shared experts / layer & --         & 2$^\dagger$ & -- \\
\midrule
\multicolumn{4}{l}{\textit{Dense Student}} \\
\midrule
Total / active parameters & 3.3B    & 2.4B       & 3.6B \\
Groups                 & 8          & 8 (2 shared + 6 routed) & 4 \\
FFN intermediate dim ($d_{\text{dense}}$) & 6{,}144 ($= 8{\times}768$) & 11{,}264 ($= 8{\times}1{,}408$) & 11{,}520 ($= 4{\times}2{,}880$) \\
\bottomrule
\end{tabular}
\end{table}

Model-specific adjustments for DP scaling, routing renormalization, shared experts, and layer~0 handling are described in Section~\ref{sec:cross_model}.

\end{document}